%% file: main.tex
\documentclass{article}

\PassOptionsToPackage{numbers, compress}{natbib}
 \usepackage[main,final]{neurips_2026}

\usepackage[utf8]{inputenc} 
\usepackage[T1]{fontenc}    
\usepackage{hyperref}       
\usepackage{url}            
\usepackage{booktabs}       
\usepackage{amsfonts}       
\usepackage{nicefrac}       
\usepackage{microtype}      
\usepackage{xcolor}         

\usepackage{dsfont}
\usepackage{amsmath}
\usepackage{amssymb}
\usepackage{mathtools}
\usepackage{amsthm}
\usepackage{enumitem}
\usepackage{pifont}
\newcommand*\diff{\mathop{}\!\mathrm{d}}
\newcommand{\cmark}{\ding{51}}%
\newcommand{\xmark}{\ding{55}}%

\usepackage[normalem]{ulem}
\useunder{\uline}{\ul}{}

\usepackage{keytheorems}
\newkeytheorem{theorem}[name=Theorem, numberwithin=section]
\newkeytheorem{proposition}[name=Proposition, numberlike=theorem]

\usepackage{xspace}

\theoremstyle{plain}

\theoremstyle{definition}

\theoremstyle{remark}

\usepackage[textsize=tiny]{todonotes}
\usepackage{algorithm}
\usepackage[noend]{algorithmic}
\usepackage{multirow}
\usepackage[table,dvipsnames]{xcolor}
\usepackage{wrapfig}
\usepackage{caption}

\usepackage{xspace}
\newcommand{\ourmodel}{GPFlow\xspace}

\title{Variable-Length Generative Protein Design via \\ Generalized Poisson Flow}

\author{%
\makebox[\linewidth][c]{
  Chaoran Cheng$^{*}$, \,
  Zhanghan Ni$^{*}$,  \,
  Yanru Qu$^{*}$,  \,
  Yuxin Chen,  \,
  Ruihan Guo,  \,
  Jiajun Fan, \,
  Ge Liu
}\\
  Siebel School of Computing and Data Science\\ 
  University of Illinois Urbana-Champaign\\
  \texttt{\{chaoran7, zn10, yanruqu2, nealchen, ruihang6, jiajunf3, geliu\}@illinois.edu}
}

\begin{document}

\maketitle
{
\renewcommand{\thefootnote}{\fnsymbol{footnote}}
\footnotetext[1]{Equal contribution.}
\footnotetext{Correspondence to \texttt{chaoran7@illinois.edu}.}
}

\setcounter{tocdepth}{2}
\addtocontents{toc}{\protect\setcounter{tocdepth}{-10}}

\begin{figure}[H]
    \vspace{-2em}
    \centering
    \makebox[\linewidth][c]{\includegraphics[width=1.15\linewidth]{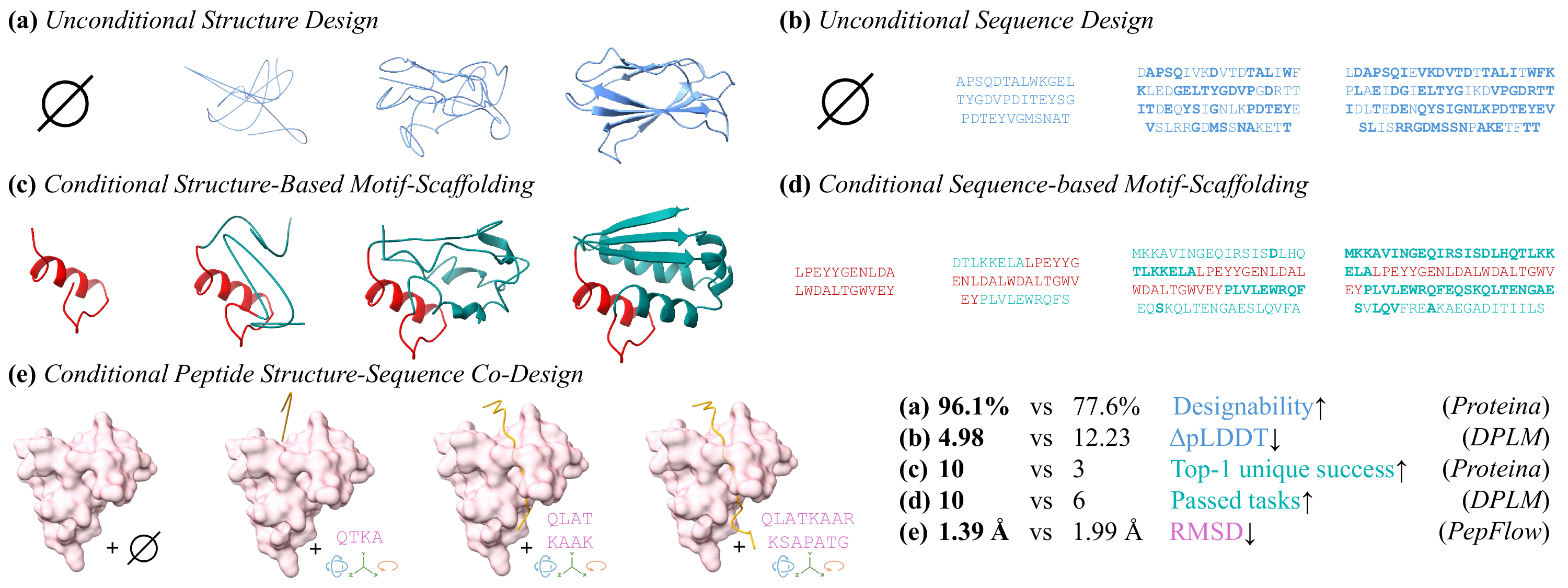}}
\end{figure}

\input{secs/0_abstract}

\input{secs/1_intro}

\input{secs/2_related}

\input{secs/3_method}

\input{secs/pi_background}

\input{secs/4_experiment}

\input{secs/5_conclusion}

\clearpage
\newpage
\section*{Acknowledgments}
This research is supported in part by the Molecule Maker Lab Institute, an AI Research Institutes program supported by NSF under Award No. 2505932, and the DOE Center for Advanced Bioenergy and Bioproducts Innovation (U.S. Department of Energy, Office of Science, Biological and Environmental Research Program under Award Number DESC0018420). Any opinions, findings, and conclusions or recommendations expressed in this publication are those of the author(s) and do not necessarily reflect the views of the U.S. Department of Energy.

\bibliography{ref}
\bibliographystyle{plainnat}

\clearpage
\newpage
\appendix
\onecolumn
\centerline{\Large\bf Supplementary Material}

\addtocontents{toc}{\protect\setcounter{tocdepth}{2}}
\vspace{1em}
\tableofcontents
\bigskip

\input{suppl/A_proof}
\input{suppl/B_algorithm}
\input{suppl/C_exp}

\input{suppl/D_ablation}
\input{suppl/E_result}

\input{secs/6_impact}



\end{document}

%% file: secs/0_abstract.tex
\begin{abstract}
The ability to generate variable-length proteins is crucial in protein design, where the optimal length is often unknown and tightly coupled to designability. Current diffusion- and flow-based generative models typically require the protein length to be specified before sampling, limiting their flexibility in exploring the feasible design space. To address this limitation, we introduce \textbf{Generalized Poisson Flow (\ourmodel)}, a variable-length generative framework that learns the rate function of an inhomogeneous generalized Poisson process by minimizing its negative log-likelihood. We establish population-level guarantees for recovering the joint multimodal distribution and derive an upper bound on the KL divergence between the data and generated distributions. We comprehensively evaluate \ourmodel across structure and sequence design, motif scaffolding, and peptide co-design, spanning Euclidean, categorical, and Riemannian modalities to fully validate its variable-length generation quality. In unconditional design, \ourmodel improves structural designability and achieves the best distributional fitness for sequence design compared to fixed-length baselines, while perfectly recovering the length distribution. In conditional motif scaffolding, \ourmodel ranks first on 10 of 16 structure-based design tasks with significantly more unique successes and also achieves more passed tasks in sequence-based design. In peptide co-design, \ourmodel remains competitive even without access to a native-length oracle.
Our code is publicly available at \url{https://github.com/ccr-cheng/generalized-poisson-flow}.
\end{abstract}

%% file: secs/1_intro.tex
\section{Introduction}


Diffusion models~\cite{ho2020denoising,song2019generative} and flow matching~\cite{lipman2022flow} have become widely used frameworks for protein generative modeling and have produced strong results in protein design~\cite{yim2023fast,geffner2025proteina,lin2024out}. However, their standard formulations operate on fixed-dimensional state spaces.
Such a constraint is restrictive for tasks like motif scaffolding, where the optimal length is often unknown \emph{a priori} and is tightly coupled to designability.
Sampling sweep over a pre-specified length range increases inference cost and may miss feasible designs when the selected lengths are incompatible with the conditional constraints.

To bridge this gap, we introduce \textbf{Generalized Poisson Flow (\ourmodel)}, a variable-length generative framework in which length evolves under an inhomogeneous generalized Poisson process. \ourmodel couples this length process to a within-length generator, thereby accommodating continuous, discrete, Riemannian, and mixed length-dependent modalities within the same construction.
Building on a marginalization theorem for the conditional process, we derive a tractable objective from the negative log-likelihood (NLL) of the stochastic process and further extend the construction to multimodal generation with guarantees on the joint distribution recovery. 
Furthermore, we derive a generator-based KL upper bound whose terms align with the additive losses, thereby complementing the population-level recovery guarantee at the optimal marginal generator. 
While the final rate objective aligns with existing work, our formulation provides a common probability-flow construction for continuous, discrete, Riemannian, and mixed modalities, with deeper theoretical connections. 

We consider five diverse protein design scenarios: (a) \textit{unconditional structure design} (Section~\ref{sec:struc_design}), (b) \textit{unconditional sequence design} (Section~\ref{sec:seq_design}), (c) \textit{structure-based motif scaffolding} (Section~\ref{sec:struc_motif}), (d) \textit{sequence-based motif scaffolding} (Appendix~\ref{suppl:seq_motif}), and (e) \textit{peptide structure-sequence co-design} (Section~\ref{sec:pep_design}). 
This task landscape evaluates how the same variable-length construction adapts across modalities and generative goals. By retaining the corresponding fixed-length base architectures, we obtain direct comparisons with their variable-length extensions.
Across these tasks, our experiments provide evidence that dynamic-length modeling can improve corresponding fixed-length systems in the evaluated settings while removing the need to specify a target length. 
For unconditional generation, our small structure model significantly improves designability over the corresponding Proteina base model, while the 642M sequence model more closely matches the UniRef50 statistics and maintains higher diversity than the DPLM model. On the motif-scaffolding benchmark, the structure model ranks first on 10 of 16 tasks and yields more than a 10-fold increase in unique successes on several challenging targets, while the sequence model also achieves more passed tasks. In peptide co-design, where translation, rotation, residue types, and torsion angles are learned jointly with insertion, \ourmodel remains competitive and outperforms the PepFlow base model even without access to a native-length oracle.
Our contributions are summarized as follows:
\begin{itemize}[noitemsep,leftmargin=*]
    \item \textbf{A unified variable-dimensional probability flow.} We formulate the variable-length generative modeling in \ourmodel with inhomogeneous generalized Poisson probability paths and couple them with continuous, discrete, Riemannian, or mixed within-length dynamics. A single marginalization principle covers these modality types for multimodal generation.
    \item \textbf{Simulation-free likelihood learning with theoretical guarantees.} We derive the rate objective from the generalized Poisson trajectory likelihood while analytically marginalizing its auxiliary insertion times, so training is fully simulation-free. We further prove population-level recovery of the joint multimodal distribution and establish a generator-based KL upper bound.
    \item \textbf{Variable-length generation across five protein-design settings.} We instantiate the framework for protein structure design, sequence design, motif scaffolding, and peptide structure-sequence co-design. Across all tasks, direct comparisons with the fixed-length base model demonstrate \ourmodel's superior performance and flexibility. We achieve superior structural designability, better sequence fitness, more unique successes and passed tasks in motif scaffolding, and stable peptide designs without a native-length oracle.
\end{itemize}

%% file: secs/2_related.tex
\section{Related Work}

\paragraph{Protein Design Models}
Building on diffusion and flow matching, recent generative models for protein design have achieved strong performance. FrameDiff~\cite{yim2023se} adopts Riemannian diffusion to learn the positions and orientations of each amino acid, and FrameFlow~\cite{yim2023fast} further extends it with Riemannian flow matching. Proteina~\cite{geffner2025proteina} is a CA-only flow-based structure generative model that achieves remarkable scaling results. RFDiffusion~\cite{rfdiffusion,butcher2025novo} is a family of diffusion-based models that incorporates the structure-prediction capabilities of RoseTTAFold~\cite{baek2021accurate} into generative modeling. 
Generally, these structure generative models also support motif scaffolding, in which fragments of functionally important protein structures, known as \emph{motifs}, are provided to the model as conditions. 
There are also dedicated models for motif scaffolding, e.g., Genie2~\cite{lin2024out} is a diffusion-based generative model with a specialized multi-motif framework that designs co-occurring motifs. These models either adopt an inpainting-style sampling that enforces hard constraints by fixing the motif residues and structures, or inject motif information as an explicit but soft conditioning signal.
Additional design models for sequences and peptides are available in Appendix~\ref{suppl:related}.

\paragraph{Variable-Length Generative Models}
Standard diffusion or flow matching models require a predetermined sampling length, limiting their applicability to variable-length data.
Two existing works incorporate data length through related stochastic-process constructions. TDDM~\cite{campbell2023trans} models dimension changes with a learnable jump kernel coupled to continuous diffusion dynamics, but does not treat a categorical length-dependent modality. EditFlow~\cite{havasi2025edit} instead derives insertion rates from discrete flow matching~\cite{gat2024discrete}; it applies to categorical modalities but not to continuous coordinates. Importantly, neither of these models has been applied to the protein design domain.
More recently, SCISOR~\cite{baron_diffusion_2025} approaches variable-length protein sequence generation via a deletion-based process, where random residues are inserted during forward noising and removed by the learned reverse process.
Concurrent Branching Flow~\cite{billera2025branching} constructs variable-length probability paths via deterministic branching of mean positions and applies it to latent protein generation and small-molecule co-design. Additional related work can be found in Appendix~\ref{suppl:related}.

\ourmodel differs primarily in scope and probabilistic interpretation. It uses one unified variable-dimensional probability-flow formulation for continuous, discrete, Riemannian, and mixed modalities. Our rate objective follows from the generalized Poisson trajectory likelihood, in which auxiliary insertion times are analytically marginalized to yield a simulation-free objective. In their overlapping regimes, \ourmodel recovers the insertion-only, modality-specific instantiations of TDDM and EditFlow and yields the same rate objective. Our contribution is therefore the unified construction and likelihood interpretation with the new KL bound on the training dynamics. While the original TDDM and EditFlow papers do not evaluate protein design tasks, our experiments instantiate the framework across five comprehensive settings for unconditional and conditional protein design.

%% file: secs/3_method.tex
\section{Generalized Poisson Flow}

To construct a probability path $\{p_t\}_{t\in[0,1]}$ for a stochastic state $Y_t$ using flow matching (or diffusion), one needs (a) $p_1\approx q$, where $q$ is the data distribution, (b) simulation-free sampling from $p_t$ for feasible training, and (c) a tractable loss for learning the marginal flow. We first derive theoretical results for the length-only state $Y_t=X_t\in\mathbb{N}$ and then extend them to the multimodal state $Y_t=(X_t,S_t)\in\mathsf Y$, where $S_t$ may comprise arbitrary discrete or continuous length-dependent modalities or their mixes. We demonstrate how to construct feasible probability paths using the generalized Poisson process.

\subsection{Length-Only Generalized Poisson Flow}\label{sec:length_only}
A \emph{generalized Poisson process} $\{X_t\}_{0\le t\le T}$ is a continuous-time Markov process defined by a \emph{rate function} $\lambda_t:[0,T]\times \mathbb{N}\to\mathbb{R}^{\ge 0}$, such that $\Pr(X_{t+h}=X_t+1)=\lambda_t(X_t)h+o(h)$ and $\Pr(X_{t+h}=X_t)=1-\lambda_t(X_t)h+o(h)$. Consistent with the flow matching convention, we will assume the generalized Poisson process to be defined on the time interval $[0,1]$ with $X_0=0$. We now show one feasible construction of the probability path $p_t(X)$ using the generalized Poisson process with the following theorem:

\begin{theorem}[store=forwardrate]\label{thm:forward_rate}
Let $\kappa:[0,1]\to[0,1]$ be a nondecreasing, absolutely continuous scheduler with $\kappa_0=0$ and $\kappa_1=1$, and define the \emph{completion time} $\tau_\kappa:=\inf\{t:\kappa_t=1\}$. For any fixed $L\ge 1$, assume $\kappa_t<1$ for $t<\tau_\kappa$ and define the insertion \emph{hazard} and \emph{conditional rate} by
\begin{equation}
    h_t:=\begin{cases}
    \dfrac{\dot\kappa_t}{1-\kappa_t}, & t<\tau_\kappa,\\
    0, & t\geq\tau_\kappa,
    \end{cases}
    \qquad
    \lambda_t(X_t|L):=(L-X_t)h_t.
    \label{eqn:cond_rate}
\end{equation}
Then, the resulting process satisfies $X_t\sim B(L,\kappa_t)$ for $t<\tau_\kappa$ and is absorbed at $X_t=L$ for $t\geq\tau_\kappa$.
\end{theorem}
See the proof in Appendix~\ref{suppl:proof_rate}.
Theorem~\ref{thm:forward_rate} offers a simulation-free approach for sampling $X_t$ and guarantees $X_1=L$ almost surely. 
Similar to standard flow matching for continuous modalities, we then establish the following marginalization theorem:
\begin{theorem}[store=marginal]\label{thm:marginal}
For any scheduler $\kappa_t$ defined above, let $q$ be a target length distribution, define $p_t(k|y):=\Pr(X_t=k|X_1=y)$, and let $p_t(k):=\sum_y p_t(k|y)q(y)$. For every $k$ with $p_t(k)>0$, the posterior probability is $p(X_1=y|X_t=k):={p_t(k|y)q(y)}/{p_t(k)}$.
Then the \emph{marginal rate}
\begin{equation}
    \lambda_t^*(k):=\mathbb{E}_{X_1\sim p(\cdot|X_t=k)}
    [\lambda_t(k|X_1)]
    \label{eqn:marginal}
\end{equation}
generates the marginal path $p_t$ and hence $p_1=q$.
\end{theorem}

See Appendix~\ref{suppl:proof_marginal} for the proof, which relies primarily on the Kolmogorov forward equation as the continuity equation.
Theorem~\ref{thm:marginal} guarantees that any length distribution $q$ can be modeled using a generalized Poisson process and also ensures the data constraint $p_1=q$ for constructing a probability flow $p_t$ over the discrete length space. We therefore refer to our generative framework as \emph{Generalized Poisson Flow}.
To effectively learn the flow, in Proposition~\ref{prop:pp_nll}, we prove that the negative log-likelihood of a realization of a generalized Poisson process can be calculated as $\mathcal{L}_\text{NLL}= \int_0^1\lambda_t\diff t-\sum_{k=1}^L\log \lambda_{t_k}$, where the two terms can be viewed as the NLL for the survival and event terms, respectively. 
The insertion times are auxiliary variables used to construct the probability path, and by the integration formula for a stochastic point process~\citep{bremaud1981point}, the expected sum over insertion times can be marginalized into an integral against the ground-truth marginal rate $\lambda^*_t$.
Combining the marginalization result in Theorem~\ref{thm:marginal} and the conditional rate in Eq.~\ref{eqn:cond_rate}, we have:
\begin{align}
    \mathbb{E}_{\{t_k\}_{k=1}^L,L}\left[-\sum_{k=1}^L\log \lambda_{t_k}\right]=\mathbb{E}_t[-\lambda^*_t\log\lambda_t]
    =\mathbb{E}_{t,X_1\sim q(X)}\left[-h_t(X_1-X_t)\log\lambda_t\right].
\end{align}
Therefore, we arrive at a tractable, simulation-free NLL loss that requires only an endpoint $X_1$, a sampled time $t$, and the corresponding intermediate state $X_t$ obtained from Theorem~\ref{thm:forward_rate}:
\begin{equation}
    \mathcal{L}_\text{GP}=\mathbb{E}_{t,X_1\sim q(X)}\left[\lambda_t-h_t(X_1-X_t)\log\lambda_t\right].\label{eqn:nll}
\end{equation}
 We note that, unlike continuous flow matching, in which the ground-truth marginal vector field is intractable, Theorem~\ref{thm:forward_rate} provides a closed-form conditional rate target. 

\subsection{Multimodal Generalized Poisson Flow}\label{sec:order_agnostic}
To truly unlock variable-length protein structure and sequence design, we now extend the length-only formulation to incorporate other length-dependent modalities such as continuous coordinates or discrete types.
Let $\mathsf S$ be a component space and define the variable-length state space $\mathsf Y:=\bigsqcup_{k\geq0}\{k\}\times\mathsf S^k$. The component space may be continuous (with a Lebesgue or manifold-volume reference measure), discrete (with counting measure), or a product of such spaces. We write $Y_t=(X_t,S_t)\in\mathsf Y$, where $S_t=(S_t^1,\dots,S_t^{X_t})$. An \emph{insertion} of a component can be described by a rate $\lambda_t$ and a sampling distribution $\rho_t$ that selects the new component value (and, in the ordered case, its insertion position). Between insertions, a within-length generator $\mathcal B_t$ refines the existing components. For continuous modalities, $\mathcal B_t f=v_t\cdot\nabla f$ (with an optional diffusion term); for discrete modalities, $\mathcal B_t$ is the corresponding continuous-time Markov-chain generator.

The key observation is that the conditional generators associated with individual clean targets can be marginalized under the posterior of the target given the current state. The following theorem states this construction; a measure-theoretic formulation and proof are provided in Appendix~\ref{suppl:proof_multimodal}.
\begin{theorem}[store=multimodal]\label{thm:multimodal}
Let $Y_1=(X_1,S_1)\sim q$ index conditional probability paths $p_t(\cdot|Y_1)$ on $\mathsf Y$, with $p_1(\cdot|Y_1)=\delta_{Y_1}$. At state $y\in\mathsf Y$, suppose each conditional process has within-length generator $\mathcal B_t^{Y_1}$, insertion rate $\lambda_t(y|Y_1)$, and sampling distribution $\rho_t(\diff m|y,Y_1)$ for the insertion variable $m$. For any $y$ with $p_t(y)>0$, the posterior of the clean target is $p(\diff y_1|Y_t=y):={p_t(y|y_1)q(\diff y_1)}/{p_t(y)}$, where $p_t=\int p_t(\cdot|y_1)q(\diff y_1)$.
Define the marginals by expectations over this posterior:
\begin{gather}
    \mathcal B_t^*f(y)
    :=\mathbb E_{Y_1\sim p(\cdot|Y_t=y)}[\mathcal B_t^{Y_1}f(y)],\quad
    \lambda_t^*(y)
    :=\mathbb E_{Y_1\sim p(\cdot|Y_t=y)}[\lambda_t(y|Y_1)],\\
    \rho_t^*(\diff m|y)
    :=\frac{\mathbb E_{Y_1\sim p(\cdot|Y_t=y)}
    [\lambda_t(y|Y_1)\rho_t(\diff m|y,Y_1)]}
    {\mathbb E_{Y_1\sim p(\cdot|Y_t=y)}[\lambda_t(y|Y_1)]},
\end{gather}
where $\rho_t^*$ can be arbitrary when the marginal rate is zero. Then the process with $\mathcal B_t^*$ and insertion kernel $Q_t^*=\lambda_t^*\rho_t^*$ generates the marginal $p_t$ and hence $p_1=q$.
\end{theorem}
The proof follows the same posterior-marginalization principle as Theorem~\ref{thm:marginal}. In the continuous case, $\mathcal B_t^*$ reduces to the usual marginal vector field; in the discrete case, the same posterior expectation marginalizes the within-length transition rates. 
We learn $\lambda_t^*$ using the Poisson NLL in Eq.~\ref{eqn:nll}. The within-length generator is learned with the appropriate continuous or discrete flow-matching loss $\mathcal{L}_\text{FM}$. As the sampling distribution $\rho_t(\cdot|y,Y_1)$ is a mixture over the $X_1-X_t$ upcoming insertions, the general reconstruction NLL is:
\begin{equation}
    \mathcal{L}_\text{rec}=\mathbb{E}_{Y_1,Y_t}\left[-h_t\sum_{k=1}^{X_1-X_t}\log\rho_t(Y_t,s_k)\right],\label{eqn:loss_rec}
\end{equation}
where the marginal rate term $(X_1-X_t)h_t$ is combined with the marginal length denominator of $1/(X_1-X_t)$ to give the final coefficient. Generally, the final loss for multimodal \ourmodel is
\begin{equation}
    \mathcal{L}=\mathcal{L}_\text{GP}+w_\text{rec}\mathcal{L}_\text{rec}+w_\text{FM}\mathcal{L}_\text{FM},\label{eqn:loss}
\end{equation}
where $w_\text{rec},w_\text{FM}$ are hyperparameters. We describe the parameterization choices for learning the flow and sampling distributions in Appendix~\ref{suppl:flow_param} and \ref{suppl:sample_param}.
Sampling from a multimodal \ourmodel generally follows the procedure discussed above, in which flow sampling for the existing length-dependent modalities is performed using the standard approach with the learned $\mathcal B_t$. An event is then sampled according to the learned rate $\lambda_t$, and if successful, it will draw the new insertion value from the learned sampling distribution $s_\text{new}\sim\rho_t$. We outline the sampling procedure in Algorithm~\ref{alg:sample_set}.

In practice, because proteins have a natural ordering according to the residue index, we need to extend our formulation to an order-preserving multimodal \ourmodel. Intuitively, this can be easily obtained by modeling $X_t+1$ separate rates and sampling distributions to represent the different insertion positions while preserving the order. A more rigorous formulation together with a modified characterization of the sampling distribution is available in Appendix~\ref{suppl:order_preserve}. 

\begin{algorithm}[ht]
\small
\caption{Sampling from Multimodal \ourmodel (Euler)}\label{alg:sample_set}
\begin{algorithmic}[1]
\STATE $Y_0=(X_0=0,S_0=\emptyset)$.
\FOR{$i\gets 1,2,\dots,N$}
    \STATE $t\gets t_i,\Delta t\gets t_{i+1}-t_i$, where $t_{N+1}:=1$.
    \STATE Predict the rate $\lambda_{\theta}(Y_t,t)$, the vector field $v_{\theta}(Y_t,t)$, and the sampling distribution $\rho_{\theta}(Y_t,t)$.
    \STATE Advance the generative flow for each $S_t$ with $v_{\theta}$ to obtain $\hat S_{t+\Delta t}$.
    \STATE Simulate the Poisson process by sampling an event on $X_{t+\Delta t}$ with probability $\lambda_\theta\Delta t$.
    \STATE If an event happens, insert the new feature value $s_\text{new}\sim\rho_{\theta}$ into $\hat S_{t+\Delta t}$ to obtain $S_{t+\Delta t}$.
    \STATE Update $Y_{t+\Delta t}=(X_{t+\Delta t},S_{t+\Delta t})$.
\ENDFOR
\STATE \textbf{Return:} $Y_1=(X_1,S_1)$.
\end{algorithmic}
\end{algorithm}

Adapting to conditional generation is generally straightforward for \ourmodel and does not affect the theoretical results. In the multimodal setup, the condition parts $Y_0=(X_c,S_c)$ are sampled from the clean conditioning data and remain unchanged along the conditional probability path. During sampling, instead of starting from an empty list, we initialize from the provided conditions and their length-dependent modality, mimicking the standard ``inpainting'' setup. Depending on the task, we may or may not allow insertions between the conditions. For example, in the motif scaffolding setup, insertion is not allowed inside a contiguous motif part, but is allowed between different motifs and at the beginning and end.

%% file: secs/pi_background.tex
\section{Theoretical Guarantees on Training Dynamics}\label{sec:background}

In this section, we derive a generator-based upper bound on the KL divergence between the ground-truth and generated distributions and relate its terms to the exact likelihood components of our loss objective.
We consider general Markov processes as a superset of \ourmodel. We first introduce the \emph{generator} as a useful tool for describing the time evolution of their probability densities, and then derive a bound on the KL divergence between the distributions generated by two Markov processes.
According to~\citet{holderrieth2024generator}, under some regularity assumptions, a Markov process $\{X_t\}_{t\ge0}$ can be characterized with its generator
$\mathcal{A}_t f(X_t) = \lim_{h \to 0^+} (\mathbb{E}[f(X_{t+h}) | X_t] - f(X_t))/h$, which can be determined by the velocity field $u_t$, the diffusion coefficient $\Sigma_t$, and the jump measure $Q_t$ as:
\begin{equation}
    \mathcal{A}_t f(x) = u_t(x) \cdot \nabla f(x) + \frac{1}{2} \text{Tr}\left( \Sigma_t(x) \nabla^2 f(x) \right) 
    + \int \left[ f(y) - f(x) \right] Q_t(\mathrm dy | x).
\end{equation}
If all $X_t$'s have a probability density function $p_t$, then $\mathbb{E}_{x\sim p_t} [\mathcal{A}_t f(x)] = \int \dot p_t(x) f(x)\,\mathrm dx$ connects the generator with the time evolution of the density $p_t$.
In common diffusion models, we consider a Markov process where $\Sigma_t=\sigma_t^2 I$ for some scalar function $\sigma_t$.

\subsection{KL Divergence Bound for Markov Processes}

We now derive a generator-based bound on the KL divergence between the distributions generated by two Markov processes, which generalizes the variable-length result in~\citet{campbell2023trans}.
See Appendix~\ref{suppl:kl_bound} for a detailed proof.

\begin{theorem}[store=klbound]\label{thm:kl_bound}
Let $\{X_t\}_{t\ge0}$ and $\{\tilde X_t\}_{t\ge0}$ be two Markov processes with generators $\mathcal{A}_t$ and $\tilde{\mathcal{A}}_t$, respectively. Suppose $\mathcal{A}_t$ is determined by $(u_t, \sigma_t^2, Q_t)$ and $\tilde{\mathcal{A}}_t$ is determined by $(\tilde u_t, \sigma_t^2, \tilde Q_t)$, so they coincide in the diffusion coefficient.
Assume that both processes have the same initial distribution $p_0$ and that sufficient regularity conditions hold. Then, the KL divergence between $p_T$ and $\tilde p_T$ can be bounded by the loss terms $\mathcal L_t(x)$ as
\begin{equation}
    \mathcal D_\text{KL}(p_T \| \tilde p_T) \leq \int_0^T \mathbb{E}_{x \sim p_t} \left[ \mathcal L_t(x) \right] \mathrm dt,\quad \mathcal{L}_t(x) =
    \frac{\|u_t(x)-\tilde u_t(x)\|_2^2}{2\sigma_t^2} + \mathcal D_\text{KL}(Q_t\|\tilde Q_t),
\end{equation}
where the first drift-diffusion term is omitted when no continuous component is present, and
$\mathcal D_\text{KL}(Q_t\|\tilde Q_t)$ denotes the generalized KL divergence between two positive measures:

\begin{align}
    \mathcal D_\text{KL}(Q\|\tilde Q)=\int \tilde Q(\mathrm dy|x)
    + \int\left( \log\frac{\mathrm dQ(y|x)}{\mathrm d\tilde Q(y|x)}-1 \right)Q(\mathrm dy|x).
\end{align}
\end{theorem}


\subsection{Generator View of \ourmodel}

We now discuss how to interpret \ourmodel as a special case of a continuous-time generator on the joint state space $\mathsf Y$. The marginal within-length generator $\mathcal B_t^*$ that refines existing components constitutes the first drift-diffusion term in the continuous case or the generalized KL term of the transition kernel for the finite Markov chain in the discrete case~\cite{gat2024discrete}. 
On the other hand, the marginal jump measure across dimensions factorizes as $Q_t^*=\lambda_t^*\rho_t^*$: at state $y$, an event occurs at rate $\lambda_t^*(y)$, after which the new component (and insertion position, when applicable) is sampled from $\rho_t^*(\diff m|y)$. The learned counterpart is the jump kernel $Q_{\theta}=\lambda_\theta\rho_\theta$, factorized into rate $\lambda_\theta(y,t)$ and sampling distribution $\rho_\theta(\diff m|y,t)$. See Appendix~\ref{suppl:generator} for the precise state-space construction and the explicit generator definition.
One crucial observation is that the gradient of our NLL-based rate loss objective coincides with the gradient of the expected generalized KL divergence:
\begin{equation}
    \nabla_\theta\mathcal{L}_\text{GP}
    =\nabla_\theta\mathbb E_{t,Y_t\sim p_t}
    \left[\mathcal{D}_\text{KL}\bigl(\lambda_t^*(Y_t)\|\lambda_\theta(Y_t,t)\bigr)\right].
\end{equation}
In Appendix~\ref{suppl:generator}, we describe how the jump contribution decomposes exactly into a generalized KL for the rate and a rate-weighted KL for the sampling distribution.
Therefore, Theorem~\ref{thm:kl_bound} is applicable and implies that the additive \ourmodel objective in Eq.~\ref{eqn:loss} minimizes the corresponding upper-bound terms, thereby providing a finer-grained distributional guarantee on the training dynamics.

Additionally, our theoretical framework offers a natural explanation of why the generalized KL divergence is chosen over other Bregman divergences, such as MSE. Importantly, existing works' choice of the generalized KL divergence as the rate loss is purely empirical. In contrast, our choice is derived directly from the exact point process NLL, thereby connecting variable-length generative loss to the underlying, more principled stochastic process NLL minimization.

%% file: secs/4_experiment.tex
\section{Experimental Results}
We instantiate our \ourmodel across five diverse protein design tasks across continuous, discrete, and Riemannian modalities. Additional details on datasets, models, and pipelines can be found in Appendix~\ref{suppl:exp}. Additional ablation studies and experimental results are in Appendix~\ref{suppl:ablation} and \ref{suppl:result}.

\subsection{Unconditional Protein Structure Generation}\label{sec:struc_design}
We first instantiate \ourmodel for unconditional structure generation, in which the CA coordinates serve as the continuous length-dependent modality. We follow the setup in Proteina~\cite{geffner2025proteina}, a CA-only protein generative model. We build our order-preserving \ourmodel on the 60M Proteina model, adding additional heads to predict rates and reconstruct sampling distributions for the continuous coordinates, bringing the total to 65M trainable parameters. 
The datasets also follow the Proteina setup, including PDB~\cite{berman2000protein} and AFDB~\cite{jumper2021highly}. For PDB, we retain only single-chain entries with lengths between 50 and 256, yielding 27k data samples in total. For AFDB, we follow Proteina's clustering instructions to obtain 713k data samples.
We train two \ourmodel variants on the two datasets from scratch. For the Proteina baseline, we use its official 60M checkpoint trained on AFDB for inference and retrain a model on PDB using the same hyperparameters.

\begin{wraptable}{r}{.54\textwidth}
\vspace{-1.em}
\input{tabs/uncond}
\vspace{-.5em}
\end{wraptable}
For the fixed-length models, we follow the standard Proteina benchmark to sample 100 structures at each length in $\{50,100,150,200,250\}$. For \ourmodel, we generate 512 structures without a length input; their lengths are generated from the learned rate. 
We report \textbf{Designability} as the proportion of \emph{designable} structures averaged across all generated samples. A designable structure is defined as having a best self-consistent RMSD (scRMSD) less than 2 Å to the 8 ProteinMPNN~\cite{dauparas2022robust} redesigned and ESMFold~\cite{lin2023evolutionary} refolded candidates.
Additional diversity-based metrics and secondary-structure ratios are also reported. \textbf{Diversity} is defined as the proportion of the designable clusters in all designable samples. \textbf{Novelty} is defined as the maximum TM-score~\cite{zhang2004tmscore} via exhaustive search against the PDB or AFDB databases, averaged across all samples. Therefore, a higher diversity score indicates more designable clusters, and a lower novelty score indicates more novel generations.

\begin{wrapfigure}{r}{.54\textwidth}
\vspace{-1.em}
\centering
\includegraphics[width=\linewidth]{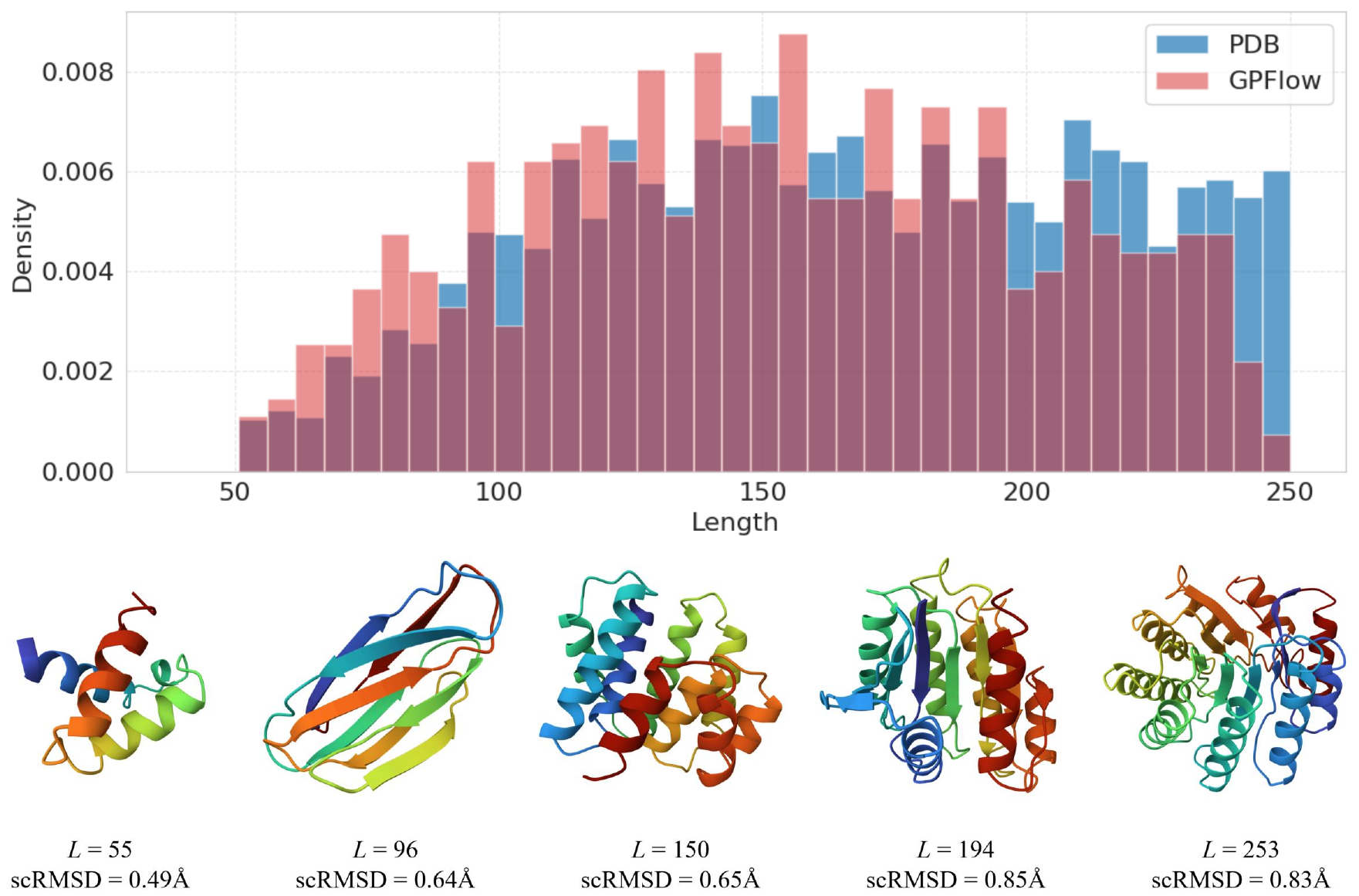}
\vspace{-1.5em}
\caption{Length distributions and representative examples of \ourmodel-generated structures.}
\label{fig:length_dist}
\vspace{-1em}
\end{wrapfigure}
The results are summarized in Table~\ref{tab:uncond}, where \ourmodel achieves a 3.9\%--18.5\% improvement in designability over the corresponding Proteina base model. Figure~\ref{fig:length_dist} serves as an intrinsic calibration check for the learned length marginal: \ourmodel's generated distribution follows the empirical training profile with a mean length of 153.7, whereas the fixed-length benchmark is 150. Appendix~\ref{suppl:result} further reports designability across \ourmodel's generated length bins, providing more concrete evidence that our superior designability is not due to the unmatched length distribution but to its better capture of the joint distribution.
In terms of novelty, we note that \ourmodel achieves a lower TM-score than the corresponding Proteina model, indicating greater novelty, although both models fall behind other baselines. The generally lower diversity of the Proteina-family models, including \ourmodel, is due to the inductive bias of low-temperature sampling, further ablated in Appendix~\ref{suppl:tradeoff}.
The motif scaffolding benchmark provides a more concrete evaluation of the diversity of \ourmodel through the number of unique successes after clustering.


\subsection{Unconditional Protein Sequence Generation}\label{sec:seq_design}

\begin{figure}[htbp]
    \centering
    \includegraphics[width=\linewidth]{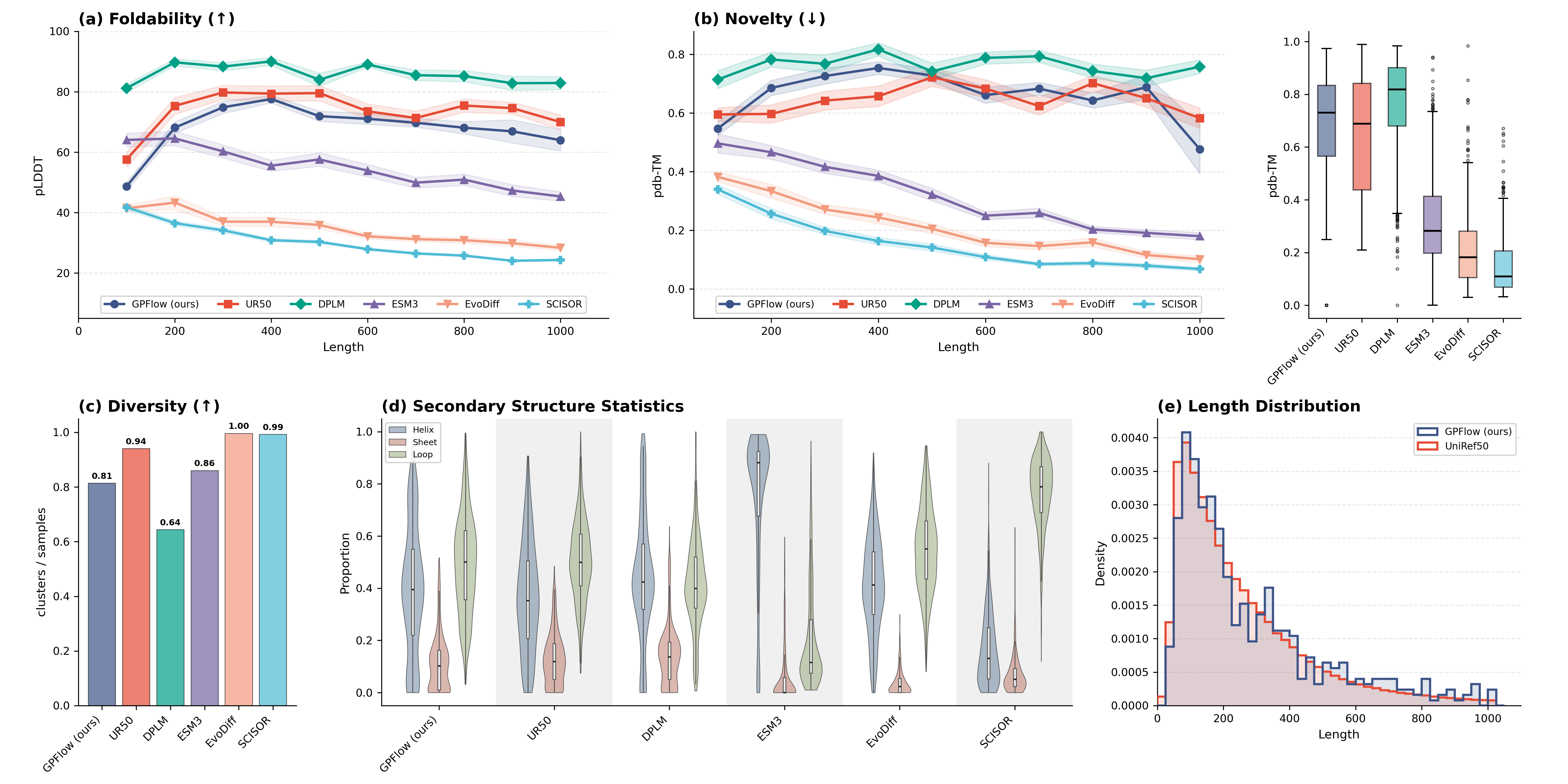}
    \vspace{-1.5em}
    \caption{Performance benchmark for unconditional protein sequence design. \textbf{(a)}~Foldability (pLDDT) across sequence length bins. \textbf{(b)}~Novelty measured by TM-score to the nearest PDB structure. \textbf{(c)}~Diversity by FoldSeek clustering. \textbf{(d)}~Secondary structure distributions. \textbf{(e)}~Length distribution of \ourmodel samples generated without length conditioning. Error bands denote $\pm$1 SEM.}
    \label{fig:seq}
    \vspace{-1em}
\end{figure}

We instantiate \ourmodel to generate variable-length protein sequences, incorporating the amino acid sequence as the discrete length-dependent modality. 
We train \ourmodel on 41M UniRef50~\citep{suzek_uniref_2007} sequences of lengths 10-1024 with a 642M-parameter DiT~\citep{peebles_scalable_2023}. To improve long-sequence generation, we follow \citet{havasi2025edit} to adopt a localized probability path as detailed in Appendix~\ref{suppl:localpath}.
For inference-time correction, we additionally learn a deletion process that can be coupled with the insertion process during sampling. 
Three different paradigms of protein sequence generative models are compared. For fixed-length discrete diffusion over categorical amino-acid tokens, we evaluate EvoDiff-OADM~\citep{alamdari_protein_2023} and DPLM~\citep{wang_diffusion_2024}. For masked language models with iterative decoding, we evaluate ESM3~\citep{hayes_simulating_2025} restricted to its sequence track. For variable-length generation, we evaluate SCISOR~\citep{baron_diffusion_2025}, a deletion-based discrete model, as a special case of EditFlow~\cite{havasi2025edit}. While other baselines are trained on UniRef50, ESM3 is trained on a substantially larger corpus that subsumes UniRef50 and is included for reference. 
To evaluate the generated samples, we fold the generated sequences with ESMFold~\citep{lin2023evolutionary} and quantify: 
(a) \textbf{Foldability} via mean pLDDT;
(b) \textbf{Novelty} via \texttt{pdb-TM}, the maximum TM-score to any PDB structure found by FoldSeek search ($\text{TM}<0.5$);
(c) \textbf{Diversity} via FoldSeek clustering at $\text{TM}\ge 0.5$, reported as the proportion of unique clusters among all generations;
(d) \textbf{Secondary structure} composition via DSSP (helix/strand/coil fractions)~\citep{kabsch1983dssp};
and 
(e) \textbf{Length distribution}, compared to UniRef50.

Figure~\ref{fig:seq}(a-d) evaluates structural statistics when assessed by ESMFold. Across all lengths, \ourmodel is the closest match to UniRef50 in mean pLDDT, faithfully reproducing the structural characteristics rather than optimizing toward uniformly high-confidence folds.
DPLM achieves an abnormally high pLDDT compared to the training data, with a mean pLDDT difference of 12.23, whereas \ourmodel achieves 4.98. DPLM has the lowest structural novelty and diversity among the models, illustrating a clear trade-off and off-data behavior, whereas \ourmodel maintains 17\% higher diversity while keeping novelty closest to UniRef50.
In contrast, ESM3, EvoDiff, and SCISOR exhibit consistently lower pLDDT together with lower \texttt{pdb-TM}.
Finally, \ourmodel best matches UniRef50 in secondary-structure composition, including a near-identical mean structural balance.
DPLM deviates from the UniRef50 profile in loop content, whereas ESM3, EvoDiff, and SCISOR exhibit large departures from the UniRef50 profile.

Figure~\ref{fig:seq}(e) illustrates the variable-length formulation: \ourmodel learns and samples from an emergent length prior that closely matches the training length. Fixed-length baselines define only $p_\theta(x|L)$ and therefore have no marginal length distribution without an additional sampler over $L$.
Selected generated sequences folded by ESMFold are presented in Figure~\ref{fig:seq_sample} with their lengths and pLDDTs.

\begin{figure}[htbp]
    \centering
    \includegraphics[width=.85\linewidth]{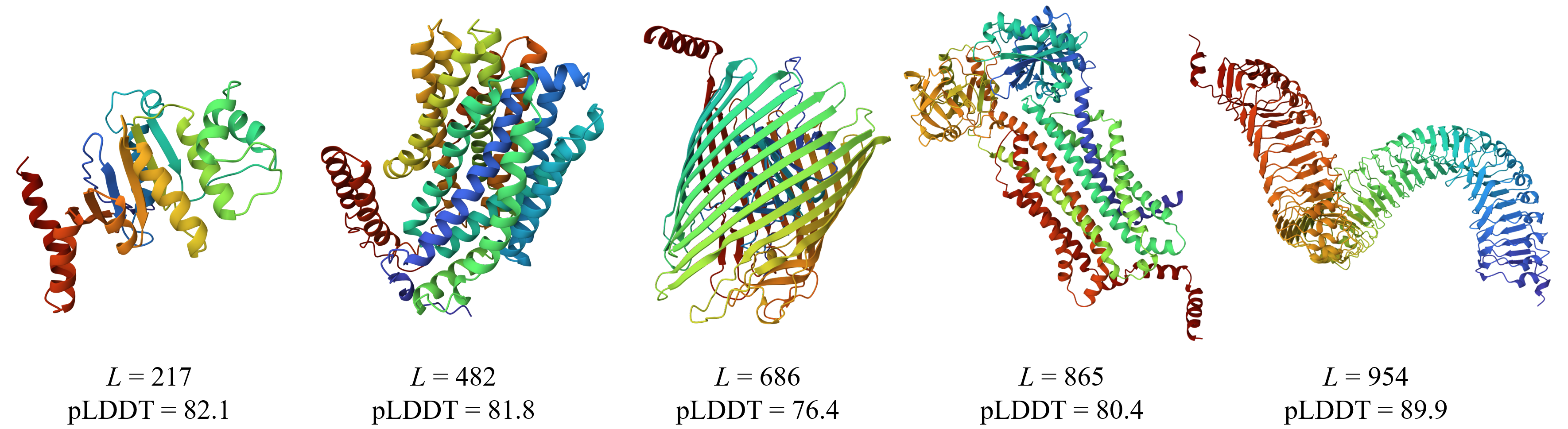}
    \vspace{-.5em}
    \caption{Representative unconditional sequence generations folded by ESMFold.}
    \label{fig:seq_sample}
\end{figure}

\subsection{Motif Scaffolding}\label{sec:struc_motif}
\input{tabs/motif}

For conditional generation, we instantiate \ourmodel for motif scaffolding, with the motif structure and/or sequence as the conditioning variable. Starting from isolated motif segments, our model gradually grows a protein scaffold through inserting and updating linker residues between motif segments, without relying on predefined contig templates. We consider both structure-based and sequence-based motif scaffolding tasks. The latter is deferred to Appendix~\ref{suppl:seq_motif}.

For structure-based motif scaffolding, we build our \ourmodel upon the Proteina variant, which injects motif structure and residue information as conditioning signals and fixes the motif portion during both training and sampling.
We train it from scratch on the same AFDB dataset and adopt the motif training augmentation used in Proteina.
The evaluation criteria follow~\citet{rfdiffusion}, which uses a combination of scRMSD, motifRMSD, pLDDT, and pAE thresholds detailed in Appendix~\ref{suppl:eval}. All successes are clustered via FoldSeek to count unique successes rather than raw successes. 

\begin{wrapfigure}{r}{.48\linewidth}
\vspace{-1.em}
    \centering
    \includegraphics[width=\linewidth]{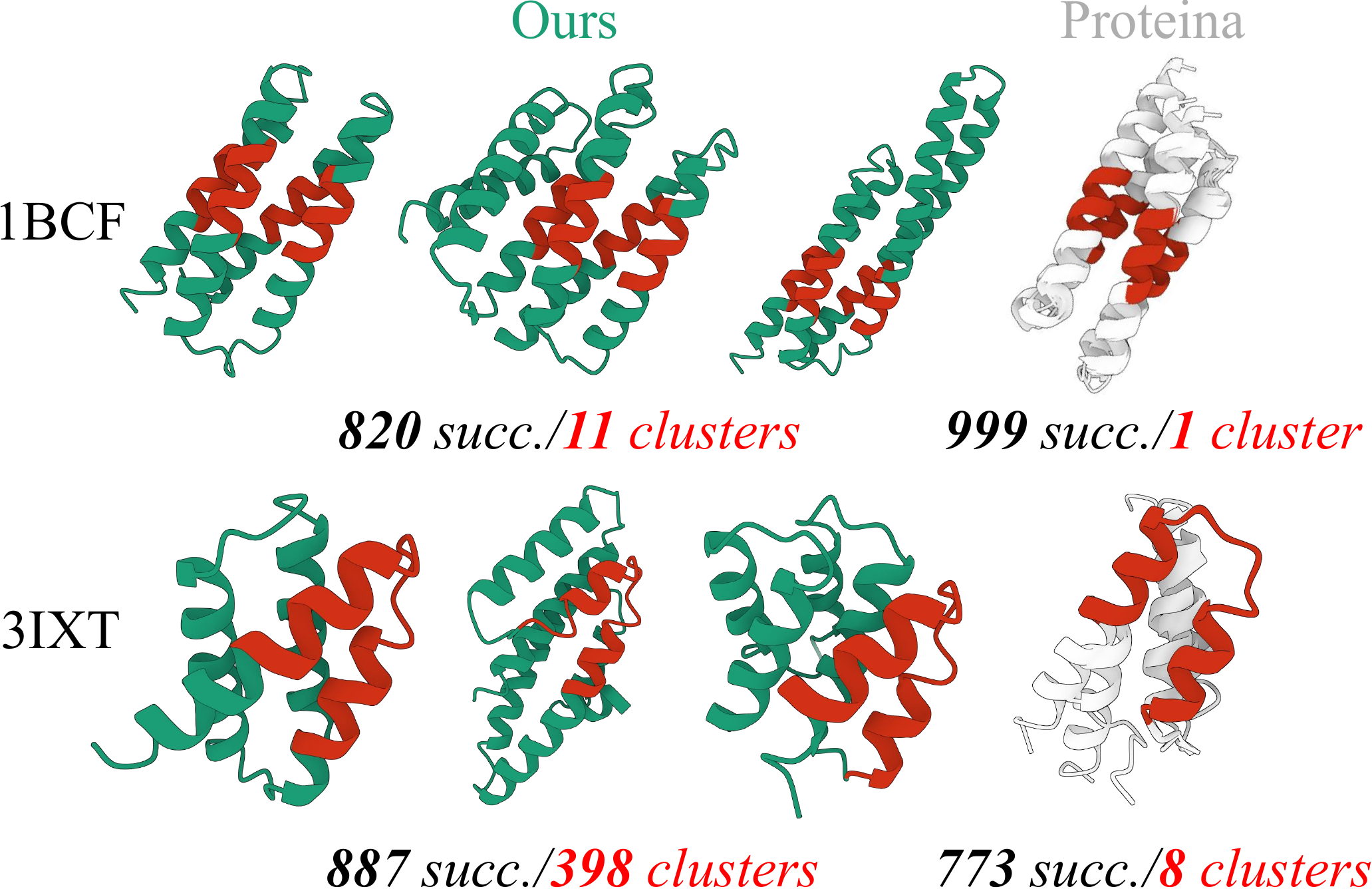}
    \vspace{-1.em}
    \caption{\ourmodel generates more diverse clusters for the same motif (highlighted in red).}
    \label{fig:motif_example}
 \vspace{-.5em}
\end{wrapfigure}
Table~\ref{tab:motif} lists the unique success counts across 16 motif tasks from the RFDiffusion benchmark~\citep{rfdiffusion}, where different contig templates corresponding to the same motif target are grouped into a single task for \ourmodel. 
Overall, \ourmodel ranks first on 10 out of 16 motif tasks.
For easier tasks like 1BCF and 3IXT, \ourmodel yields considerably more diverse generations, with variation in both secondary structures and scaffold lengths, whereas the base Proteina generations collapse into only a few clusters. (Figure~\ref{fig:motif_example}). For more challenging tasks like 5WN9 and 5YUI, the raw success counts of \ourmodel also dominate.
Noticeably, \ourmodel can also generate successful scaffolds longer than 200 residues, whereas Proteina generations are restricted by the evaluated contig templates.
A finer-grained analysis of the generation lengths in Appendix~\ref{suppl:motif_length} further confirms the diversity of our generations across different length bins.
Our evaluations show that the learned length-varying process can produce successful scaffolds over a wider range than the predefined contig choices used by the fixed-length baselines.

\subsection{Peptide Co-Design}\label{sec:pep_design}
Finally, we extend \ourmodel to peptide sequence-structure co-design. Given a target receptor protein, the model jointly generates the all-atom structure and amino-acid sequence of a binding peptide, with the peptide length not predetermined. We use the dataset and benchmark from~\citet{li2024full}, in which the receptor structure and pocket positions are provided as conditioning context.
\ourmodel is adapted from the PepFlow~\cite{li2024full} architecture, in which per-residue rotation, translation, residue type, and side-chain torsion angles are learned jointly to model the all-atom peptide structure. Rotation and torsion angles are modeled using Riemannian flow matching, with an additional extension that learns inserted elements using the corresponding Riemannian norm. The residue types are modeled using a simplex-based continuous representation that embeds soft one-hot logits in Euclidean geometry.

\begin{wrapfigure}{r}{.54\textwidth}
\vspace{-1.em}
\centering
\includegraphics[width=\linewidth]{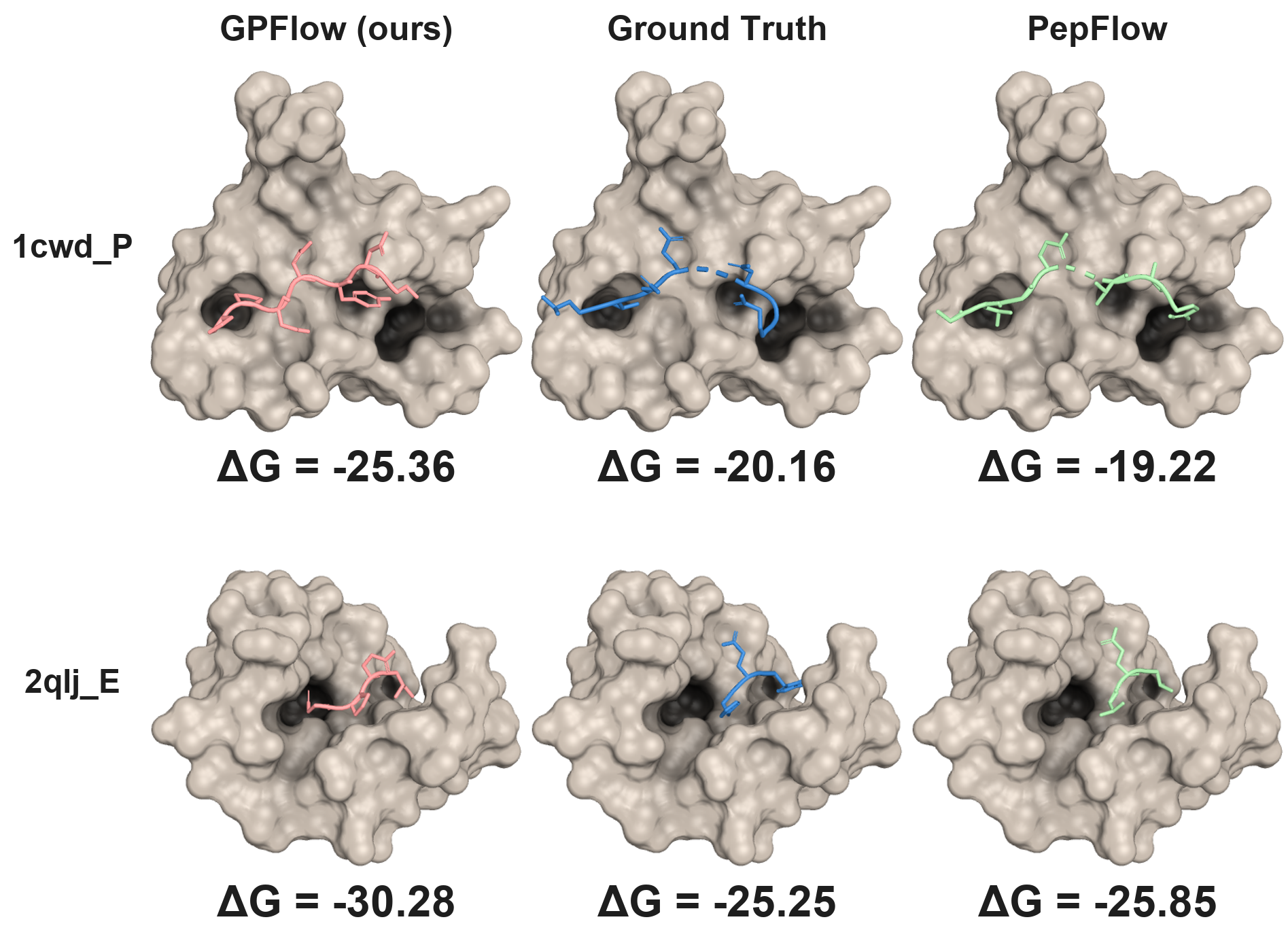}
\caption{Samples of \ourmodel-generated peptides compared to ground truth and PepFlow~\cite{li2024full}.}
\label{fig:pepdemo}
\vspace{-1.em}
\end{wrapfigure}
We compare against the three co-design baselines: RFdiffusion~\citep{rfdiffusion} composed with ProteinMPNN~\citep{dauparas2022robust} for sequence design, ProteinGenerator~\citep{lisanza2025multistate} that jointly samples backbone and sequence, and the base PepFlow model. We report eight generative metrics: amino-acid recovery (\textbf{AAR}), \textbf{RMSD}, secondary-structure ratio (\textbf{SSR}), binding-site recovery (\textbf{BSR}), \textbf{Affinity}, \textbf{Stability}, \textbf{Designability} and \textbf{Diversity}. The complete evaluation protocol and the length-matched/unmatched breakdown are provided in Appendix~\ref{suppl:peptide_metric}. As shown in Table~\ref{tab:codesign}, \ourmodel attains the best performance on 4 out of the 8 metrics and also surpasses the base PepFlow in terms of Designability and Diversity. 
Notably, \ourmodel is the only model that samples peptide length from a learned rate function rather than conditioning on the native length (Appendix~\ref{suppl:pep_length}). This matters for \textit{de novo} peptide design, where fixed-length baselines rely on a native-length oracle unavailable in practice.
Figure~\ref{fig:pepdemo} highlights two qualitative behaviors of \ourmodel with a better energy landscape. In 1cwd\_P, the processed ground truth contains a gap induced by the removal of a non-standard residue, and PepFlow inherits this gap through its native-indexed generation. In contrast, \ourmodel generates a continuous peptide across this region. In 2qlj\_E, the generated insertion extends more deeply into the binding site, suggesting greater sensitivity to the local pocket geometry.

\input{tabs/peptide}

%% file: tabs/uncond.tex
{
\centering
\caption{Unconditional protein structure design benchmark. The best result for each metric is in \textbf{bold}.}\label{tab:uncond}
\vspace{-.8em}
\resizebox{\linewidth}{!}{
\begin{tabular}{@{}lccccc@{}}
\toprule
\multirow{2}{*}{Model} & \multirow{2}{*}{\begin{tabular}[c]{@{}c@{}}Design-\\ ability $\uparrow$\end{tabular}} & \multirow{2}{*}{Diversity $\uparrow$} & \multicolumn{2}{c}{Novelty $\downarrow$} & \multirow{2}{*}{\begin{tabular}[c]{@{}c@{}}Struct\%\\ $\alpha$/$\beta$\end{tabular}} \\ \cmidrule(lr){4-5}
 &  &  & PDB & AFDB &  \\ \midrule
FrameDiff [\citenum{yim2023se}] & 65.4 & 0.39 & 0.73 & { 0.75} & 64.9/11.2 \\
FrameFlow [\citenum{yim2023fast}] & 88.6 & { 0.53} & \textbf{0.69} & \textbf{0.73} & 55.7/18.4 \\
RFDiffusion [\citenum{rfdiffusion}] & 94.4 & 0.46 & { 0.71} & 0.77 & 64.3/17.2 \\
Proteina\textsuperscript{*} (PDB) & 77.6 & 0.25 & 0.78 & 0.81 & 48.7/15.2 \\
Proteina\textsuperscript{\textdagger} (AFDB) & 90.6 & \textbf{0.57} & 0.77 & 0.84 & 68.8/6.3 \\ \midrule
\ourmodel (PDB) & \textbf{96.1} & 0.26 & 0.76 & 0.80 & 61.2/13.8 \\
\ourmodel (AFDB) & { 94.5} & 0.43 & 0.75 & 0.82 & 72.4/5.4 \\ \bottomrule
\end{tabular}
}
}
\scriptsize{\textsuperscript{*}A 60M Proteina is retrained on PDB.}\\
\scriptsize{\textsuperscript{\textdagger}Official Proteina 60M checkpoint for inference.}

%% file: tabs/motif.tex
\begin{table}[ht]
\centering
\caption{Number of unique successes (1000 generations for each task) on the RFDiffusion benchmark~\citep{rfdiffusion}. Superscripts denote the number of motif segments. Motif tasks with multiple lengths are indicated in parentheses. Best results are shown in \textbf{bold}.}\label{tab:motif}
\small
\rowcolors{2}{gray!15}{white}
\resizebox{\linewidth}{!}{
\begin{tabular}{l|ccccc} \toprule
Task Name (contig ver.) & \ourmodel (Ours)                  & Proteina [\citenum{geffner2025proteina}]                & Genie2 [\citenum{lin2024out}]         & RFDiffusion [\citenum{rfdiffusion}]    & FrameFlow [\citenum{yim2023fast}]             \\ \midrule
Unique success (avg. / total) & \textbf{117.4 / 1878} & 94.9 / 1519             & 64.8 / 1037    & 41.5 / 664     & 51.1 / 817             \\ \midrule
1BCF$^4$                  & \textbf{11}           & 1                       & 1              & 1              & 1                      \\
1PRW$^2$                  & \textbf{42}            & 1                       & 1              & 1              & 1                      \\
1QJG$^3$                  & 1          & 3                       & 5              & 1              & \textbf{18}                     \\
1YCR$^1$                  & \textbf{413}          & 249                     & 134            & 7              & 149                    \\
2KL8$^2$                  & \textbf{3}           & 1                       & 1              & 1              & 1                      \\
3IXT$^1$                  & \textbf{398}          & 8                       & 14             & 3              & 8                      \\
4JHW$^2$                  & 0           & 0                       & 0              & 0              & 0                      \\
4ZYP$^1$                  & \textbf{144}          & 11                      & 3              & 6              & 4                      \\
5IUS$^2$                  & \textbf{3}           & 1                       & 1              & 1              & 0                      \\
5TPN$^1$                  & \textbf{147}           & 4                       & 8              & 5              & 6                      \\
5TRV$^1$ (long / med / short)  & 136          & \textbf{179} / 22 / 1            & 97 / 23 / 3    & 23 / 10 / 1    & 77 / 21 / 1            \\
5WN9$^1$                  & \textbf{19}           & 2                       & 1              & 0              & 3                      \\
5YUI$^3$                  & \textbf{9}          & 5                       & 3              & 1              & 1                      \\
6E6R$^1$ (long / med / short)  & 268                   & \textbf{713} / 417 / 56 & 415 / 272 / 26 & 381 / 151 / 23 & 110 / 99 / 25          \\
6EXZ$^1$ (long / med / short)  & 263                   & 290 / 43 / 3            & 326 / 54 / 2     & 167 / 25 / 1   & \textbf{403} / 110 / 3 \\
7MRX$^1$ (128 / 85 / 60)       & 21          & 51 / 31 / 2             & 27 / 23 / 5    & \textbf{66} / 13 / 1    & 35 / 22 / 1           \\ \bottomrule
\end{tabular}
\vspace{-1em}
}
\end{table}

%% file: tabs/peptide.tex
\begin{table}[ht]
\centering
\caption{Conditional sequence-structure co-design on peptide design. Best results are shown in \textbf{bold}. \textsuperscript{\textdagger} are composite over the GT-length matched/unmatched samples; see Appendix~\ref{suppl:peptide_metric} for details.}\label{tab:codesign}
\resizebox{\linewidth}{!}{
\begin{tabular}{@{}lcccccccccc@{}}
\toprule
\multirow{2}{*}{Model} & \multicolumn{4}{c}{Geometry} & \multicolumn{2}{c}{Energy} & \multicolumn{2}{c}{Design} & \multicolumn{2}{c}{Length} \\
\cmidrule(lr){2-5} \cmidrule(lr){6-7} \cmidrule(lr){8-9} \cmidrule(lr){10-11}
 & AAR \% $\uparrow$ & RMSD \AA{} $\downarrow$ & SSR \% $\uparrow$ & BSR \% $\uparrow$
      & Affinity \% $\uparrow$ & Stability \% $\uparrow$ & Designability \% $\uparrow$
      & Diversity $\uparrow$ & Mean & Requires GT? \\ \midrule
ProteinGenerator [\citenum{lisanza2025multistate}]   & 28.18 & 4.36 & 58.89 & 15.60 & 11.25 & \textbf{40.00}
                   & 72.66 & 
            \textbf{0.476} & 10.53 & \cmark \\
RFdiffusion+MPNN [\citenum{rfdiffusion}]   & 26.69 & 4.17 & 62.70 & 18.09 &  7.27 & 17.36
                   & \textbf{84.72} & 0.408 & 10.53 & \cmark \\
PepFlow [\citenum{li2024full}]          & 46.35 & 1.99 & 78.98 & \textbf{89.40} & 11.74 &  6.46 & 37.27 & 0.390 & 10.53 & \cmark  \\
\midrule
\ourmodel   & \textbf{53.71}\textsuperscript{\textdagger} & \textbf{1.39}\textsuperscript{\textdagger} & \textbf{91.08}\textsuperscript{\textdagger}
                   & \text{86.04} & \textbf{21.03} &  6.27 & 48.19
                   & \text{0.443} & 9.86 & \xmark  \\ 
\bottomrule                  
\end{tabular}
}
\end{table}

%% file: secs/5_conclusion.tex
\section{Conclusion \& Limitations}\label{sec:limit}
In this work, we propose \ourmodel, a unified framework for variable-length generative modeling applicable to diverse protein design scenarios. \ourmodel uses a generalized Poisson process coupled to within-length dynamics and covers multimodal settings with continuous, discrete, and mixed modalities. Simulation-free rate learning based on exact negative log-likelihood minimization is proven via posterior marginalization, with an additional generator-based KL bound on the training dynamics. 
Our comprehensive experimental setup for \ourmodel provides concrete evidence of its flexibility in unconditional and conditional generation without requiring a prescribed target length and demonstrates its superior performance in designability and distributional fitness across all tasks.

One limitation of \ourmodel is its sensitivity to scheduler choices, especially in multimodal setups where different modalities may interact. Additional sampling techniques such as $\tau$-leaping (Appendix~\ref{suppl:tau_leap}) and localized paths (Appendix~\ref{suppl:localpath}) may also introduce subtle trade-offs. 
We are actively exploring empirical training and sampling techniques to develop a more comprehensive unified variable-length generative framework for biological domains.

%% file: suppl/A_proof.tex
\section{Theoretical Details}\label{suppl:proofs}
In this section, we provide proofs for the major theorems established in the main context. We also discuss the technical details of interpreting \ourmodel as a generator and its connections to existing models such as EditFlow.

\subsection{Proof for Theorem~\ref{thm:forward_rate} (Conditional Rate)}\label{suppl:proof_rate}
\getkeytheorem{forwardrate}

\begin{proof}
Consider $Y_t=L-X_t$ as a pure-death process with death rate $\mu_t(Y_t)=Y_t h_t$. Let $p_t(k):=\Pr(Y_t=k)$. The Kolmogorov forward equation reads:
\begin{align}
    \dot p_t(k)&=\mu_t(k+1) p_t(k+1) - \mu_t(k) p_t(k)
    = h_t \left[ (k+1) p_t(k+1) - k p_t(k) \right].
\end{align}
Consider the probability generating function (PGF) $G(z,t):=\mathbb{E}[z^{Y_t}]=\sum_{k=0}^Lp_t(k)z^k$. We have
\begin{align}
    \frac{\partial G}{\partial t} &= \sum_{k=0}^L \dot p_t(k) z^k = h_t \sum_{k=0}^L \left[ (k+1) p_t(k+1) - k p_t(k) \right] z^k\\
    &= h_t \left( \frac{\partial G}{\partial z} - z \frac{\partial G}{\partial z} \right) = h_t(1-z) \frac{\partial G}{\partial z}.
\end{align}
One can verify that the function 
\begin{equation}
    G(z, t) = \left( \kappa_t + z(1-\kappa_t) \right)^L
\end{equation} 
satisfies the above PDE with the initial condition $Y_0=L$. As $G(z, t)$ is the standard PGF for the binomial distribution $B(L, 1-\kappa_t)$, we have $Y_t\sim B(L, 1-\kappa_t)$ and $X_t\sim B(L,\kappa_t)$ for $t<\tau_\kappa$. As $t\to\tau_\kappa^-$, $\kappa_t\to1^-$ and $\mathbb E[Y_t]=L(1-\kappa_t)\to0$.

Moreover, $Y_t=L-X_t$ is nonnegative and nonincreasing along every pure-birth path, so its pathwise limit exists; the vanishing expectation implies that this limit is zero almost surely. Hence $X_t\to L$ almost surely. Defining $L$ as absorbing and setting the rate to zero for $t\geq\tau_\kappa$ yields the stated result without evaluating the hazard after saturation.  
\end{proof}

The general definition of the scheduler above allows us to use additional scheduler families without compromising theoretical validity. In particular, the early-insertion schedule $\kappa_t=\min\{1,t/\tau\}$ has $h_t=1/(\tau-t)$ for $t<\tau$; at $t=\tau$, the process has reached $L$ almost surely, after which the state is absorbing, and the rate is defined to be zero. Thus, the expression $\dot\kappa_t/(1-\kappa_t)$ is never evaluated in the saturated regime. In the losses, the target insertion contribution is also set to zero for $t\geq\tau_\kappa$.
Such a scheduler was already used in prior work, such as TDDM~\cite{campbell2023trans}, but it was never rigorously formulated.

\subsection{Proof for Theorem~\ref{thm:marginal} (Marginal Rate)}\label{suppl:proof_marginal}
\getkeytheorem{marginal}

\begin{proof}
The Kolmogorov forward equation reads:
\begin{equation}
    \dot p_t(k|y) = \lambda_t(k-1|y)p_t(k-1|y)-\lambda_t(k|y)p_t(k|y).
\end{equation}
For the marginal $p_t(k) := \sum_y q(y) p_t(k|y)$, differentiating with respect to time, we obtain:
\begin{equation}
    \dot p_t(k) = \sum_y q(y) \dot p_t(k|y).
\end{equation}
Substituting the Kolmogorov equation, we get
\begin{equation}
    \dot p_t(k) = \underbrace{\sum_y q(y)\lambda_t(k-1|y) p_t(k-1|y)}_{\text{event}}
    -\underbrace{\sum_y q(y) \lambda_t(k|y) p_t(k|y)}_{\text{survival}}.
    \label{eqn:kolmogorov}
\end{equation}
For every $k$ with $p_t(k)>0$, Bayes' rule gives the posterior probability stated in Theorem~\ref{thm:marginal}:
\begin{equation}
    p(X_1=y|X_t=k)=\frac{p_t(k|y)q(y)}{p_t(k)}.
\end{equation}
Therefore, the marginal rate conditioned on the explicit current state $X_t=k$ is
\begin{equation}
    \lambda_t^*(k)
    :=\mathbb{E}_{X_1\sim p(\cdot|X_t=k)}[\lambda_t(k|X_1)]
    =\sum_y \frac{p_t(k|y)q(y)}{p_t(k)}\lambda_t(k|y).
\end{equation}
Therefore,
\begin{equation}
    \lambda^*_t(k) p_t(k) = \sum_y q(y) p_t(k|y) \lambda_t(k|y),
\end{equation}
which is the survival term in Eq.~\ref{eqn:kolmogorov}. Similarly, we have
\begin{equation}
    \lambda^*_t(k-1) p_t(k-1) = \sum_y q(y) p_t(k-1|y) \lambda_t(k-1|y),
\end{equation}
which is the event term in Eq.~\ref{eqn:kolmogorov}. Combining both results:
\begin{equation}
    \dot p_t(k) = \lambda^*_t(k-1) p_t(k-1) - \lambda^*_t(k) p_t(k),\label{eqn:length_continuity}
\end{equation}
which is precisely the Kolmogorov forward equation for a generalized Poisson process driven by the rate $\lambda^*_t(X_t)$.
Since $p_1(k|y)=\mathds{1}\{k=y\}$, the terminal marginal is $p_1(k)=q(k)$.
\end{proof}

\subsection{Proposition~\ref{prop:pp_nll} (Poisson NLL)}\label{suppl:proof_nll}
\begin{proposition}[store=ppnll]\label{prop:pp_nll}
For a target generalized Poisson process with rate function $\lambda^*_t$, consider a realization on $[0,1]$ with event times $\{t_k\}_{k=1}^{L}$, where $0<t_k\le 1$. Its negative log-likelihood (NLL) under an estimated rate $\lambda_t$ is:
\begin{equation}
    \mathcal{L}_\text{GP}= \int_0^1\lambda_t\diff t-\sum_{k=1}^L\log \lambda_{t_k}.\label{eqn:nll_naive}
\end{equation}
\end{proposition}
\begin{proof}
At each time step $t$, the probability of an event is $\lambda_t\diff t$, whereas the probability of survival is $1-\lambda_t\diff t$. Therefore, we have
\begin{align}
    \mathcal{L}_\text{GP}&=-\log\left(\prod_{t\in{\text{no event}}} (1-\lambda_t\mathrm{d}t)\cdot\prod_{k=1}^L\lambda_{t_k}\mathrm{d}t\right)\\
    &=-\log\left(\prod_{t=0}^1 (1-\lambda_t\mathrm{d}t)\cdot\prod_{k=1}^L\lambda_{t_k}\right)\\
    &=\int_0^1\lambda_t\mathrm{d}t-\sum_{k=1}^L\log \lambda_{t_k},
\end{align}
which concludes the proof.
\end{proof}

The identity used in Section~\ref{sec:length_only} to rewrite the event term of the Poisson NLL follows from the integration (smoothing) formula for the stochastic intensity of a point process~\citep{bremaud1981point}. Let $N_t$ be a simple point process with event times $\{t_k\}$, adapted to a filtration $\{\mathcal{F}_t\}$, that admits an $\mathcal{F}_t$-stochastic intensity $\lambda^*_t$, i.e., $N_t-\int_0^t\lambda^*_s\diff s$ is an $\mathcal{F}_t$-martingale. Then, for every nonnegative $\mathcal{F}_t$-predictable process $H_t$:
\begin{equation}
    \mathbb{E}\left[\sum_k H_{t_k}\right]=\mathbb{E}\left[\int_0^1 H_t\mathrm{d}N_t\right]=\mathbb{E}\left[\int_0^1 H_t\lambda^*_t\diff t\right],
\end{equation}
a property that in fact characterizes the stochastic intensity $\lambda^*_t$; the deterministic-integrand case recovers Campbell's theorem~\citep{daley2003introduction}. By Theorem~\ref{thm:marginal}, the marginal rate $\lambda^*_t$ is such a stochastic intensity (the process conditioned on $X_1=L$ is an inhomogeneous pure-birth process with predictable compensator~\citep{bremaud1975extension}). Choosing the predictable integrand $H_t=-\log\lambda_t(X_{t^-})$ with the learned rate $\lambda_t$ then yields:
\begin{equation}
    \mathbb{E}_{\{t_k\},L}\left[-\sum_k\log\lambda_{t_k}\right]=\mathbb{E}\left[-\int_0^1\lambda^*_t\log\lambda_t\diff t\right]=\mathbb{E}_t[-\lambda^*_t\log\lambda_t],
\end{equation}
which is the identity in the main text. In this way, the NLL loss objective is fully simulation-free.

\subsection{Proof for Theorem~\ref{thm:multimodal} (Multimodal \ourmodel)}\label{suppl:proof_multimodal}

We first describe a measure-theoretic characterization of the variable-length state space $\mathsf Y$ that ensures that defining a probability path is mathematically well-posed.
Let $(\mathsf S,\mathcal S,\nu)$ be a component space with a $\sigma$-finite reference measure. We use Lebesgue measure for Euclidean components, counting measure for categorical components, Riemannian volume for manifold-valued components, and the corresponding product measure for mixed components. Define
\begin{equation}
    \mathsf Y:=\bigsqcup_{k\geq0}\{k\}\times\mathsf S^k,
    \qquad
    \mu:=\sum_{k\geq0}\delta_k\otimes\nu^{\otimes k}.
\end{equation}
Thus, a probability law on variable-length objects can be represented by a density or mass function with respect to the single reference measure $\mu$.

For $y=(k,s)\in\mathsf Y$, let $\mathsf M(y)$ be the insertion-variable space and let $I(y,m)\in\{k+1\}\times\mathsf S^{k+1}$ be the measurable insertion map. In the order-agnostic case, $m$ contains only the new component value; in the ordered case, it also contains an insertion slot.
For a clean target $Y_1=y_1$, let $\rho_t(\diff m|y,y_1)$ be a probability kernel on $\mathsf M(y)$. The associated conditional jump measure is
\begin{equation}
    Q_t(A|y,y_1)
    :=\lambda_t(y|y_1)
      \int_{\mathsf M(y)}\mathds{1}\{I(y,m)\in A\}
      \rho_t(\diff m|y,y_1).
    \label{eqn:insertion_jump_cond}
\end{equation}
This equation makes the coupling explicit: an insertion value is drawn from $\rho_t(\cdot|y,y_1)$ only when an event governed by $\lambda_t(y|y_1)$ occurs. The total mass of $Q_t(\cdot|y,y_1)$ is exactly $\lambda_t(y|y_1)$.

\getkeytheorem{multimodal}
\begin{proof}
This proof follows the outline in transdimensional jump diffusion~\citep{campbell2023trans}, while extending it to allow the inserted component to be continuous, discrete, Riemannian, or mixed. 
Let $\mathcal B_t^{y_1}$ denote the conditional within-length generator. For every bounded test function $f$ in the domain of the generator, the conditional path satisfies the weak forward equation:
\begin{equation}
\frac{\diff}{\diff t}\int_{\mathsf Y}f(y)p_t(\diff y|y_1)
=\int_{\mathsf Y}\left[\mathcal B_t^{y_1} f(y)
+\int_{\mathsf Y}(f(y')-f(y))Q_t(\diff y'|y,y_1)\right]p_t(\diff y|y_1).
\label{eqn:conditional_weak_forward}
\end{equation}
For brevity in the proof, denote the posterior law written explicitly in Theorem~\ref{thm:multimodal} by
\begin{equation}
    \pi_t(\diff y_1|y):=p(\diff y_1|Y_t=y).
\end{equation}
Equivalently, it is characterized by the disintegration identity $q(\diff y_1)p_t(\diff y|y_1)=p_t(\diff y)\pi_t(\diff y_1|y)$.
Thus, the three posterior expectations in the theorem are:
\begin{align}
    \mathcal B_t^*f(y)
    &=\int\mathcal B_t^{y_1} f(y)\pi_t(\diff y_1|y),\\
    \lambda_t^*(y)
    &=\int\lambda_t(y|y_1)\pi_t(\diff y_1|y),\\
    \rho_t^*(\diff m|y)
    &=\frac{\int\lambda_t(y|y_1)\rho_t(\diff m|y,y_1)\pi_t(\diff y_1|y)}
    {\int\lambda_t(y|y_1)\pi_t(\diff y_1|y)}.
\end{align}
Integrating Eq.~\ref{eqn:conditional_weak_forward} over $q(\diff y_1)$ and applying the disintegration identity then yields the marginal jump measure
\begin{equation}
\begin{aligned}
    Q_t^*(A|y)
    :=\int Q_t(A|y,y_1)\pi_t(\diff y_1|y)
    =\lambda_t^*(y)\int_{\mathsf M(y)}\mathds{1}\{I(y,m)\in A\}\rho_t^*(\diff m|y),
\end{aligned}
\label{eqn:insertion_jump_marginal}
\end{equation}
where $\lambda_t^*$ and $\rho_t^*$ are precisely those in Theorem~\ref{thm:multimodal}. Therefore,
\begin{equation}
\frac{\diff}{\diff t}\int f(y)p_t(\diff y)
=\int\left[\mathcal B_t^*f(y)+\int(f(y')-f(y))Q_t^*(\diff y'|y)\right]p_t(\diff y),
\end{equation}
which is the weak forward equation of the claimed marginal process. Since $p_1(\cdot|y_1)=\delta_{y_1}$ for $q$-almost every $y_1$, its endpoint is $p_1=q$.
\end{proof}

For intuition, the corresponding forward equation can be written, as an identity of measures, as
\begin{equation}
    \partial_t p_t=\mathcal B_t^{*\dagger}p_t-\lambda_t^*p_t
    +\int_{\mathsf Y}p_t(\diff \bar y)Q_t^*(\cdot|\bar y).
    \label{eqn:multimodal_continuity}
\end{equation}
The second term is outgoing probability mass, whereas the third term is inflow from predecessor states; only the outgoing insertion measure $Q_t^*=\lambda_t^*\rho_t^*$ belongs to the generator at the current state. When $\mathsf S$ is continuous, $\mathcal B_t^{*\dagger}p_t=-\nabla\cdot(v_t^*p_t)$ (or its diffusion/Riemannian analog), and integrals over inserted values are ordinary integrals. When $\mathsf S$ is discrete, the same display is a master equation, all integrals over $\mathsf S$ become sums, and $\mathcal B_t^*$ is a within-length CTMC generator. Hence, the proof applies to both cases without treating a categorical component as if it had a Euclidean density.

\subsection{Theorem~\ref{thm:order} (Order-Preserving \ourmodel)}\label{suppl:proof_order}
In the order-preserving specialization, $\mathsf M(y)=\{0,\dots,k\}\times\mathsf S$ and $m=(i,a)$. The general insertion map becomes $I_i(y,a):=I(y,m)=I(y,(i,a))$, where $I_i$ places component $a$ in slot $i$. Thus, the slot is part of the insertion variable whose rate-weighted distribution is marginalized; the insertion map itself remains deterministic. A detailed construction is provided in Appendix~\ref{suppl:order_preserve}.

\begin{theorem}[store=order]\label{thm:order}
At a state $Y_t=y=(k,s)$, let the insertion variable be $m=(i,a)$, where $i\in\{0,\dots,k\}$ is an insertion slot and $a\in\mathsf S$ is the new component. Using the posterior law $p(\diff y_1|Y_t=y)$ from Theorem~\ref{thm:multimodal}, define for each slot
\begin{align}
    (\lambda_t^i)^*(y)
    &:=\mathbb E_{Y_1\sim p(\cdot|Y_t=y)}[\lambda_t^i(y|Y_1)],\\
    (\rho_t^i)^*(\diff a|y)
    &:=\frac{\mathbb E_{Y_1\sim p(\cdot|Y_t=y)}
    [\lambda_t^i(y|Y_1)\rho_t^i(\diff a|y,Y_1)]}
    {\mathbb E_{Y_1\sim p(\cdot|Y_t=y)}[\lambda_t^i(y|Y_1)]}.
\end{align}
Together with the marginal within-length generator $\mathcal B_t^*$ from Theorem~\ref{thm:multimodal}, the insertion jump measure
\begin{equation}
    Q_t^*(\diff y'|y)=\sum_{i=0}^k(\lambda_t^i)^*(y)
    \int_{\mathsf S}\delta_{I_i(y,a)}(\diff y')(\rho_t^i)^*(\diff a|y)
\end{equation}
generates the ordered target distribution $Y_1=(X_1,S_1)\sim q$, where $S_1=(S_1^1,\dots,S_1^{X_1})$. As before, $(\rho_t^i)^*$ may be chosen arbitrarily if $(\lambda_t^i)^*(y)=0$.
\end{theorem}

\begin{proof}
The order-preserving process is the specialization of Theorem~\ref{thm:multimodal} in which the insertion variable is $m=(i,a)$ and $I(y,(i,a))=I_i(y,a)$. Its outgoing insertion jump measure is therefore:
\begin{equation}
    Q_t^*(\diff y'|y)=\sum_{i=0}^{k}(\lambda_t^i)^*(y)
    \int_{\mathsf S}\delta_{I_i(y,a)}(\diff y')(\rho_t^i)^*(\diff a|y).
\end{equation}
For a clean target of length $L$ and current length $k$, the thinning path partitions the $L-k$ remaining target components into insertion bins $\Omega_0,\dots,\Omega_k$ (see Appendix~\ref{suppl:order_preserve} for construction). The conditional slot rate is $|\Omega_i|h_t$, hence the total conditional rate is:
\begin{equation}
    \sum_{i=0}^k\lambda_t^i(k|L)=(L-k)h_t,
\end{equation}
which recovers the conditional rate in Theorem~\ref{thm:forward_rate}. Posterior marginalization of each slot-specific insertion measure gives exactly the rates and insertion kernels in Theorem~\ref{thm:order}. Therefore, the terminal distribution is the ordered target distribution $q$.
\end{proof}
The general weak-forward argument from Appendix~\ref{suppl:proof_multimodal} then yields
\begin{equation}
    \partial_t p_t=\mathcal B_t^{*\dagger}p_t
    -p_t\sum_{i=0}^k(\lambda_t^i)^*
    +\int_{\mathsf Y}p_t(\diff\bar y)Q_t^*(\cdot|\bar y),
    \label{eqn:order_continuity}
\end{equation}
and the insertion maps $I_i$ preserve the clean-target ordering by construction. 

\input{suppl/proof_kl_bound}

\subsection{Generator for \ourmodel}\label{suppl:generator}
\input{tabs/generator}

The continuity equations in Eq.~\ref{eqn:length_continuity}, \ref{eqn:multimodal_continuity}, and \ref{eqn:order_continuity} describe the evolution of the three \ourmodel variants. Table~\ref{tab:generator} lists the corresponding \emph{outgoing} generators. It is important to distinguish the jump measure from its appearance in the forward equation. At a current state $y$, the generator contains one outgoing insertion measure:
\begin{equation}
    Q_t(\diff y'|y)=\lambda_t(y)
    \int\delta_{I(y,m)}(\diff y')\rho_t(\diff m|y).
    \label{eqn:gpflow_jump_measure}
\end{equation}
Its total mass is $\lambda_t(y)$. Consequently, the adjoint forward equation contains an outgoing term $-\lambda_t(y)p_t(y)$ and an incoming term obtained by integrating $p_t(\bar y)Q_t(\diff y|\bar y)$ over predecessor states $\bar y$. The predecessor inflow is not a second jump measure at $y$. For length-only \ourmodel, Eq.~\ref{eqn:gpflow_jump_measure} reduces to $Q_t(\diff k'|k)=\lambda_t(k)\delta_{k+1}(\diff k')$.

This also clarifies the relation to a general generator-matching formulation. A general jump measure may move arbitrary parts of an existing continuous state. In \ourmodel, its support is restricted to length-increasing successors $I(y,m)$: the jump adds a new component sampled from $\rho_t$, while existing components evolve between jumps through the within-length generator $\mathcal B_t$. Thus, $\lambda_t$ and $\rho_t$ are always coupled as the rate and normalized insertion kernel of the same event.

Let $Q_t=\lambda_t\rho_t$ and $\tilde Q_t=\tilde\lambda_t\tilde\rho_t$, suppressing the deterministic insertion map for clarity. Because $\rho_t$ and $\tilde\rho_t$ are probability kernels, the generalized KL divergence between the positive jump measures decomposes exactly as:
\begin{align}
    \mathcal D_\text{KL}(Q_t\|\tilde Q_t)
    &=\tilde\lambda_t-\lambda_t
      +\lambda_t\log\frac{\lambda_t}{\tilde\lambda_t}
      +\lambda_t\int\rho_t(\diff m)\log\frac{\diff\rho_t}{\diff\tilde\rho_t}(m)\\
    &=\mathcal D_\text{KL}(\lambda_t\|\tilde\lambda_t)
      +\lambda_t\mathcal D_\text{KL}(\rho_t\|\tilde\rho_t).
    \label{eqn:insertion_kl_decomposition}
\end{align}
There is therefore no double counting: the Poisson rate NLL supplies the first term, and the insertion-kernel NLL supplies the second. With exact NLL parameterizations, their sum is the jump-measure term in Theorem~\ref{thm:kl_bound}. The mean-regression approximation used for continuous insertion values is a practical surrogate.

\subsection{Connection to EditFlow}\label{suppl:editflow}
In this section, we demonstrate the connection between \ourmodel, when instantiated only on the discrete modality, and existing work such as EditFlow~\cite{havasi2025edit}. We first note that \ourmodel can adopt the ``deletion'' operation by extending the generalized Poisson process to an inhomogeneous birth-death process. For simplicity, we now consider EditFlow with only insertion and substitution. Recall that the EditFlow is a variable-length discrete flow matching model whose loss is defined as:
\begin{equation}
    \mathcal{L}_\text{edit}=\mathbb{E}_{t,S_t,X_t}\left[\sum_{k=1}^{X_t} u_t(\cdot|S_t)-\sum_{k=1}^{n_\text{edit}} \frac{\dot\kappa_t}{1-\kappa_t}\log u_t(S_1|S_t)\right],\label{eqn:edit_loss}
\end{equation}
where $n_\text{edit}\ge X_1-X_t$ is the number of edit operations, including insertions ($n_\text{ins}=X_1-X_t$) and substitutions ($n_\text{sub}=n_\text{edit}-n_\text{ins}$). EditFlow predicts the token-wise mixed-type ``vector field'' $u_\theta(S_t,t)$, which is parameterized as:
\begin{align}
    u^\text{ins}_\theta(S_t,t)&:=\lambda^\text{ins}_\theta(S_t,t)Q_\theta^\text{ins}(s|S_t,t), \\
    u^\text{sub}_\theta(S_t,t)&:=\lambda^\text{sub}_\theta(S_t,t)Q_\theta^\text{sub}(s,s'|S_t,t),
\end{align}
where $\lambda_\theta$ is the parameterized insertion/substitution rate, $Q_\theta^\text{ins}$ is the Markov chain transition matrix representation of the probability of inserting a new token $s$, and $Q_\theta^\text{sub}$ is the probability of substituting the old token $s'$ with a new token $s$. We can first decompose the loss into the insertion and substitution parts as:
\begin{equation}
    \mathcal{L}_\text{edit}=\mathcal{L}_\text{ins}+\mathcal{L}_\text{sub}
    =\;
    \begin{aligned}[t]
        &\mathbb{E}_{t,S_t,X_t}\left[\sum_{k=1}^{X_t} u^\text{ins}_t-\sum_{k=1}^{n_\text{ins}} \frac{\dot\kappa_t}{1-\kappa_t}\log u^\text{ins}_t(s_1)\right]\\
        &+\mathbb{E}_{t,S_t,X_t}\left[\sum_{k=1}^{X_t} u^\text{sub}_t-\sum_{k=1}^{n_\text{sub}} \frac{\dot\kappa_t}{1-\kappa_t}\log u^\text{sub}_t(s_1,s')\right].
    \end{aligned}
\end{equation}
We first note that for both $Q_\theta$, the distribution property requires $\sum_s Q(s|\cdot)=1$. In this way, the first term in Eq.~\ref{eqn:edit_loss} reduces to $\sum\lambda_\theta$, the total rate described in the order-preserving \ourmodel in Appendix~\ref{suppl:order_preserve}. Using $\log u_\theta=\log\lambda_\theta+\log Q_\theta$, we have:
\begin{equation}
    \mathcal{L}=\underbrace{\sum_{k=1}^{X_t}\lambda-\sum_{k=1}^n\frac{\dot\kappa_t}{1-\kappa_t}\log\lambda}_{\text{GP}} 
    -\underbrace{\sum_{k=1}^n\frac{\dot\kappa_t}{1-\kappa_t}\log Q}_{\text{reconstruction}}.\label{eqn:edit_expand}
\end{equation}
For the insertion loss, it immediately becomes clear that the first two terms together coincide with $\mathcal{L}_\text{GP}$ in Eq.~\ref{eqn:nll} as $n_\text{ins}=X_1-X_t$, and the last terms coincide with the cross-entropy reconstruction loss $\mathcal{L}_\text{rec}$ in Eq.~\ref{eqn:loss_dfm} using the formulation in~\citet{gat2024discrete}.
For the substitution loss, since substitutions do not alter sequence lengths, the rate component can be effectively coalesced into a single logit component $\hat Q_\theta$. Specifically, the survival of a substitution event is equivalent to a substitution with $s=s'$. In this way, we may assume the event always happens while compensating with a non-zero $\hat Q_\theta(s,s):=\lambda_\theta^\text{sub}>0$ to allow staying in the same state. Therefore, throwing away the unnecessary rate part, the substitution loss becomes:
\begin{equation}
    \mathcal{L}_\text{sub}=-\sum_{k=1}^n\frac{\dot\kappa_t}{1-\kappa_t}\log \hat Q,
\end{equation}
which is the standard loss for training a discrete flow matching model in~\citet{gat2024discrete}, serving as $\mathcal{L}_\text{FM}$ to refine the discrete modality in our \ourmodel formulation. Combining the losses together, we get $\mathcal{L}_\text{edit}=\mathcal{L}_\text{ins}+\mathcal{L}_\text{sub}=\mathcal{L}_\text{GP}+\mathcal{L}_\text{rec}+\mathcal{L}_\text{FM}$, which is but a special case of our \ourmodel loss in Eq.~\ref{eqn:loss} on a single discrete modality.

%% file: suppl/proof_kl_bound.tex
\subsection{Proof for Theorem~\ref{thm:kl_bound} (KL Bound)}\label{suppl:kl_bound}
\getkeytheorem{klbound}

\begin{proof}
We first discuss one sufficient set of regularity conditions. These conditions ensure that the Radon-Nikodym derivative and every expression below are well-defined. 

Both processes start from the same distribution $p_0$. For every $t\in[0,T]$, $p_t$ and $\tilde p_t$ are strictly positive densities with respect to a common $\sigma$-finite reference measure; they are differentiable in time and twice differentiable in every continuous coordinate. The densities, their derivatives, and all generator terms are integrable so that differentiation may be exchanged with integration, the generator identities may be applied, and the continuous-coordinate integration-by-parts boundary terms vanish. On every continuous component, the two processes have the same measurable diffusion coefficient $\sigma_t^2I$ with $\sigma_t>0$, and the time integral of $\mathbb E_{p_t}[\|u_t-\tilde u_t\|^2/\sigma_t^2]$ is finite. The jump measures are finite kernels, $Q_t(\cdot|x)$ is absolutely continuous with respect to $\tilde Q_t(\cdot|x)$ for $\diff t\,p_t(\diff x)$-almost every $(t,x)$, and the generalized KL term is measurable and integrable over time and $p_t$. If no continuous component is present, the regularity conditions are unnecessary.

The proof uses the generator identity together with completing the square for the continuous component and the nonnegativity of $z\log z-z+1$ for the jump component. For readability, we suppress the time subscript on $Q_t$ and $\tilde Q_t$ below.
Consider the time derivative of the KL divergence:
\begin{align}
    \frac{\mathrm d}{\mathrm dt} \mathcal D_{\rm KL}(p_t \| \tilde p_t) &
    = \frac{\mathrm d}{\mathrm dt} \int p_t(x) \log \frac{p_t(x)}{\tilde p_t(x)} \,\mathrm dx\\
    &= \int \dot p_t(x) \log \frac{p_t(x)}{\tilde p_t(x)} \,\mathrm dx + \underset{=0}{\underbrace{\int p_t(x) \frac{\dot p_t(x)}{p_t(x)} \,\mathrm dx}} - \int \frac{\mathrm d\tilde p_t}{\mathrm dt}(x) \frac{p_t(x)}{\tilde p_t(x)} \,\mathrm dx
    \\
    & = \mathbb{E}_{x\sim p_t} \left[ \mathcal{A}_t \left( \log \frac{p_t}{\tilde p_t} \right)(x) \right] - \mathbb{E}_{x\sim \tilde p_t} \left[ \tilde{\mathcal{A}}_t \left( \frac{p_t}{\tilde p_t} \right)(x) \right]\\
    &= \mathbb{E}_{x\sim p_t} \left[ \mathcal{A}_t \left( \log \frac{p_t}{\tilde p_t} \right)(x) - \frac{\tilde p_t(x)}{p_t(x)} \tilde{\mathcal{A}}_t \left( \frac{p_t}{\tilde p_t} \right)(x) \right].
\end{align}

Below, for simplicity, we denote $f = p_t/\tilde p_t$. Then
\begin{align}
    &\frac{\mathrm d}{\mathrm dt} \mathcal D_{\rm KL}(p_t \| \tilde p_t) \\ 
    &= \frac{\mathrm d}{\mathrm dt} \int p_t(x) \log \frac{p_t(x)}{\tilde p_t(x)} \,\mathrm dx 
    = \mathbb{E}_{x\sim p_t} \left[ \mathcal{A}_t (\log f)(x) - \frac{1}{f(x)} \tilde{\mathcal{A}}_t f(x) \right] \\
    &= \begin{aligned}[t]
        \mathbb{E}_{x\sim p_t} \Bigg[ & u_t(x)\cdot\nabla\log f(x)+\frac{1}{2}\sigma_t^2 \nabla \cdot \nabla\log f(x)+\int [\log f(y)-\log f(x)]Q(\mathrm dy|x) \\
        & -\frac{1}{f(x)}\left( \tilde u_t(x)\cdot\nabla f(x)+\frac{1}{2}\sigma_t^2 \nabla \cdot \nabla f(x)+\int [f(y)-f(x)]\tilde Q(\mathrm dy|x) \right)\Bigg].
    \end{aligned}\label{eqn:kl_derivative_1}
\end{align}
Substitute $\nabla f(x)/f(x)=\nabla\log f(x)$ and $\nabla \cdot \nabla f(x)/f(x)=\nabla \cdot \nabla\log f(x) + \|\nabla \log f(x)\|^2$ in Eq.~\ref{eqn:kl_derivative_1} to get:
\begin{align}
    &\frac{\mathrm d}{\mathrm dt} \mathcal D_{\rm KL}(p_t \| \tilde p_t)\\
    & = \begin{aligned}[t]
        \mathbb{E}_{x\sim p_t} \Bigg[ & u_t(x)\cdot\nabla\log f(x)+\frac{1}{2}\sigma_t^2 \nabla \cdot \nabla\log f(x)+\int [\log f(y)-\log f(x)]Q(\mathrm dy|x) \\
        & - \tilde u_t(x)\cdot\frac{\nabla f(x)}{f(x)} - \frac{1}{2}\sigma_t^2 \frac{\nabla \cdot \nabla f(x)}{f(x)} - \int \left( \frac{f(y)}{f(x)}-1 \right)\tilde Q(\mathrm dy|x) \Bigg]
    \end{aligned} \\
    & = \begin{aligned}[t]
        \mathbb{E}_{x\sim p_t} \Bigg[ & (u_t(x)-\tilde u_t(x))\cdot\nabla\log f(x)-\frac{1}{2}\sigma_t^2\|\nabla \log f(x)\|^2 \\
        & +\int \left( \log\frac{f(y)}{f(x)} \right)Q(\mathrm dy|x) + \int \left( 1-\frac{f(y)}{f(x)} \right)\tilde Q(\mathrm dy|x) \Bigg].
    \end{aligned}
\end{align}

Note that
\begin{align}
    & (u_t(x)-\tilde u_t(x))\cdot\nabla\log f(x)-\frac{\sigma_t^2}{2}\|\nabla \log f(x)\|^2 \\
    & = \frac{\|u_t(x)-\tilde u_t(x)\|^2}{2\sigma_t^2}-\frac{\sigma_t^2}{2} \left\| \nabla \log f(x) - \frac{u_t(x)-\tilde u_t(x)}{\sigma_t^2} \right\|^2 \le \frac{\|u_t(x)-\tilde u_t(x)\|^2}{2\sigma_t^2},
\end{align}
For the jump terms, define $a(y):=f(y)/f(x)$ and $r(y):=\mathrm dQ(y|x)/\mathrm d\tilde Q(y|x)$. The pointwise inequality $r\log a+1-a\leq r\log r-r+1$ holds because the difference between the right- and left-hand sides is $a[(r/a)\log(r/a)-r/a+1]\geq0$. Consequently,
\begin{align}
    &\int \log\frac{f(y)}{f(x)}Q(\mathrm dy|x)
    +\int\left(1-\frac{f(y)}{f(x)}\right)\tilde Q(\mathrm dy|x)\\
    &\qquad\leq\int\left(r\log r-r+1\right)\tilde Q(\mathrm dy|x)
    =\mathcal D_\text{KL}(Q(\cdot|x)\|\tilde Q(\cdot|x)).
\end{align}
Combining the continuous and jump bounds, we conclude:
\begin{align}
    \frac{\mathrm d}{\mathrm dt} \mathcal D_{\rm KL}(p_t \| \tilde p_t) & \le \mathbb{E}_{x\sim p_t} \left[ \frac{\|u_t(x)-\tilde u_t(x)\|^2}{2\sigma_t^2} + \int \left( \log \frac{\mathrm dQ(y|x)}{\mathrm d\tilde Q(y|x)}-1 \right) Q(\mathrm dy|x) + \int \tilde Q(\mathrm dy|x) \right],
\end{align}
which completes the proof by integration. If no continuous component is present, the completion-of-the-square step and its resulting drift-diffusion term are absent, leaving only the jump-measure term.
\end{proof}

%% file: tabs/generator.tex
\begin{table}[ht]
\centering
\caption{Generators for each variant of \ourmodel.}\label{tab:generator}
\begin{tabular}{@{}lccc@{}}
\toprule
Model Variant & Length-Only & Order-Agnostic & Order-Preserving \\ \midrule
Continuity Equation & Eq.~\ref{eqn:length_continuity} & Eq.~\ref{eqn:multimodal_continuity} & Eq.~\ref{eqn:order_continuity} \\
Within-length $\mathcal B$ & $0$ & $v\cdot\nabla$ (or discrete CTMC) & $v\cdot\nabla$ (or discrete CTMC) \\
Outgoing jump $Q$ & $\lambda(k)\delta_{k+1}$ & $\lambda(y)\rho(\diff m|y)$ & $\sum_{i=0}^k\lambda^i(y)\rho^i(\diff a|y)$ \\ \bottomrule
\end{tabular}
\end{table}

%% file: suppl/B_algorithm.tex
\section{Algorithmic Details}\label{suppl:algorithm}

\begin{figure}[ht]
    \centering
    \includegraphics[width=\linewidth]{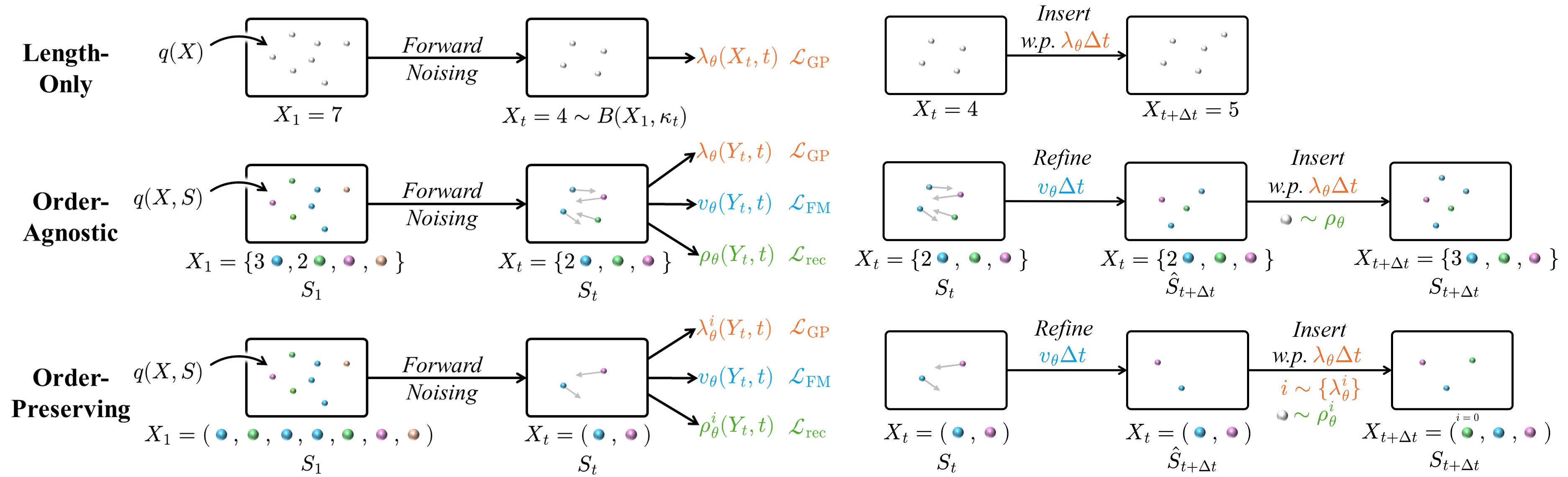}
    \vspace{-0.5em}
    \caption{Training procedure (\textbf{left}) and one sampling step (\textbf{right}) for three variants of \ourmodel: \textbf{Length-Only} (Section~\ref{sec:length_only}), \textbf{Multimodal Order-Agnostic} (Section~\ref{sec:order_agnostic}), and \textbf{Multimodal Order-Preserving} (Appendix~\ref{suppl:order_preserve}). 
    }
    \label{fig:train_sample}
\end{figure}

In this section, we provide additional details on the algorithms, model design, and general techniques we used to improve our \ourmodel. Explicit formulas for the losses are also provided.

\subsection{Order-Preserving Generalized Poisson Flow}\label{suppl:order_preserve}
For multimodal \ourmodel, a natural setting arises when the ordering of a sequence matters in practice, as in protein design, where amino acids have a natural ordering. In the main context, we treat all Poisson events as an unordered set and model the counting process, which is suitable for order-agnostic generative tasks such as molecule generation. We now discuss how to extend \ourmodel to ordered sequential data in a natural way.

For an ordered clean state $Y_1=(X_1,S_1)$ with $X_1=L$ and $S_1=(S^1_1,\dots,S_1^L)$, consider a current state $Y_t=y=(k,s)$ whose $k$ components correspond to retained target indices $r_1<\cdots<r_k$. Set $r_0=0$ and $r_{k+1}=L+1$. The $L-k$ unretained target indices are partitioned among the $k+1$ insertion slots by:
\begin{equation}
    \Omega_i:=\{\ell\in\{1,\dots,L\}\setminus\{r_1,\dots,r_k\}:r_i<\ell<r_{i+1}\},
    \quad 0\le i\le k.
\end{equation}
Thus, $\Omega_i$ contains exactly the upcoming components that belong between the $i$-th and $(i+1)$-st retained components, with $i=0$ and $i=k$ representing the beginning and end.
For each unretained index $\ell$, let $a_\ell\in\mathsf S$ denote its component value along the conditional path at time $t$.

Consistent with Appendix~\ref{suppl:proof_multimodal}, define the insertion-variable space at $y$ as:
\begin{equation}
    \mathsf M(y):=\{0,\dots,k\}\times\mathsf S,
    \quad m=(i,a),
\end{equation}
where $i$ is the insertion slot and $a$ is the new component value. For $s=(s^1,\dots,s^k)$, the general insertion map specializes to:
\begin{equation}
    I(y,(i,a))=:I_i(y,a)
    :=\left(k+1,(s^1,\dots,s^i,a,s^{i+1},\dots,s^k)\right),
    \label{eqn:ordered_insertion_map}
\end{equation}
with the natural empty-prefix or empty-suffix convention at $i=0$ or $i=k$.

For the conditional thinning path indexed by $Y_1$, slot $i$ has rate
\begin{equation}
    \lambda_t^i(y|Y_1)=|\Omega_i|h_t,
    \quad
    \lambda_t(y|Y_1)=\sum_{i=0}^k\lambda_t^i(y|Y_1)=(L-k)h_t,
\end{equation}
and $\rho_t^i(\diff a|y,Y_1)$ is the normalized sampling distribution over $\{a_\ell:\ell\in\Omega_i\}$. When $\lambda_t(y|Y_1)>0$, these slot-specific quantities define the joint insertion-variable distribution:
\begin{equation}
    \rho_t(i,\diff a|y,Y_1)
    :=\frac{\lambda_t^i(y|Y_1)}{\lambda_t(y|Y_1)}
    \rho_t^i(\diff a|y,Y_1).
\end{equation}
Pushing this distribution through Eq.~\ref{eqn:ordered_insertion_map} gives the conditional outgoing jump measure:
\begin{equation}
    Q_t(A|y,Y_1)
    =\sum_{i=0}^k\lambda_t^i(y|Y_1)
    \int_{\mathsf S}\mathds{1}\{I(y,(i,a))\in A\}\rho_t^i(\diff a|y,Y_1),
\end{equation}
which is the order-preserving specialization of Eq.~\ref{eqn:insertion_jump_cond}. Theorem~\ref{thm:order} posterior-marginalizes these slot-specific rates and distributions. 

During learned sampling, the same factorization first simulates an event with total rate $\lambda_\theta(y,t)=\sum_{i=0}^k\lambda_\theta^i(y,t)$, then chooses slot $i$ with probability $\lambda_\theta^i(y,t)/\lambda_\theta(y,t)$, samples $a\sim\rho_\theta^i(\cdot|y,t)$, and updates the state through $I(y,(i,a))=I_i(y,a)$. This process generates the target order-preserving distribution, as guaranteed by Theorem~\ref{thm:order}.
In this way, the generalized Poisson NLL loss in Eq.~\ref{eqn:nll} is adapted as the summation over $k+1$ bins:
\begin{equation}
    \mathcal{L}_\text{GP}=\mathbb{E}_{t,Y_1\sim q,Y_t\sim p_t(\cdot|Y_1)}\left[\sum_{i=0}^{X_t}\left(\lambda^i_t-h_t|\Omega_i|\log\lambda^i_t\right)\right],\label{eqn:adaptednll}
\end{equation}
and the reconstruction loss of the sampling distributions in Eq.~\ref{eqn:loss_rec} is similarly adapted as:
\begin{equation}
    \mathcal{L}_\text{rec}=\mathbb{E}_{t,Y_1\sim q,Y_t\sim p_t(\cdot|Y_1)}\left[-h_t\sum_{i=0}^{X_t}\sum_{\ell\in\Omega_i}\log\rho_t^i(a_\ell|Y_t)\right].\label{eqn:loss_rec_order}
\end{equation}
We outline the training and sampling procedures for the order-preserving \ourmodel in Algorithm~\ref{alg:train} and \ref{alg:sample}, respectively, and adopt it in our experimental setup for protein design. A pictorial comparison of the training and sampling stages of our \ourmodel variants is shown in Figure~\ref{fig:train_sample} across different scenarios.
We also note that the previous discussions for both order-agnostic and order-preserving \ourmodel can be easily generalized to an arbitrary number of length-dependent modalities $S=\{S^{(m)}\}_{1\le m\le M}=\{S^{(m),i}\}_{1\le m\le M,1\le i\le X}$. For simplicity and clarity of notation, we will omit the modality superscripts.

\begin{algorithm}[H]
\caption{Training Order-Preserving \ourmodel}\label{alg:train}
\begin{algorithmic}[1]
\WHILE{not converged}
    \STATE Sample $t\sim[0,1];Y_1=(X_1,S_1)\sim q(X,S)$.
    \STATE Sample $S_t\sim p_t(S)$ for each $S_t^i,1\le i\le X_1$.
    \STATE For each position $\ell$, delete $S_t^\ell$ with probability $1-\kappa_t$ to obtain $\hat{S}_t$ and $X_t=|\hat{S}_t|,Y_t=(X_t,\hat{S}_t)$.
    \STATE Predict the rates $\lambda^i_{\theta}(Y_t,t)$, the length-dependent modality vector field $v_{\theta}(Y_t,t)$, and the sampling distributions $\rho^i_{\theta}(Y_t,t),0\le i\le X_t$.
    \STATE Optimize the loss $\mathcal{L}$ in Eq.~\ref{eqn:loss}. 
\ENDWHILE
\end{algorithmic}
\end{algorithm}

\begin{algorithm}[H]
\caption{Sampling from Order-Preserving \ourmodel (Euler)}\label{alg:sample}
\begin{algorithmic}[1]
\STATE $Y_0=(X_0=0,S_0=\emptyset)$.
\FOR{$n\gets 1,2,\dots,N$}
    \STATE $t\gets t_n,\Delta t\gets t_{n+1}-t_n$, where $t_{N+1}:=1$.
    \STATE Predict the rates $\lambda^i_{\theta}(Y_t,t)$, the vector field $v_{\theta}(Y_t,t)$, and the sampling distributions $\rho_{\theta}^i(Y_t,t)$ for $0\le i \le X_t$.
    \STATE Advance the generative flow for each $S_t$ with $v_{\theta}$ to obtain $\hat S_{t+\Delta t}$.
    \STATE Set $\hat Y_{t+\Delta t}:=(X_t,\hat S_{t+\Delta t})$ and $\lambda_\theta:=\sum_{i=0}^{X_t}\lambda^i_\theta(Y_t,t)$.
    \STATE Sample an insertion event with probability $\lambda_\theta\Delta t$.
    \STATE If an event happens, choose slot $i$ with probability $\lambda_\theta^i/\lambda_\theta$, sample $a\sim\rho_\theta^i(\cdot|Y_t,t)$, and set $Y_{t+\Delta t}=I_i(\hat Y_{t+\Delta t},a)$; otherwise set $Y_{t+\Delta t}=\hat Y_{t+\Delta t}$.
\ENDFOR
\STATE \textbf{Return:} $Y_1=(X_1,S_1)$.
\end{algorithmic}
\end{algorithm}

\subsection{Training and Sampling Pseudocode}
We provide additional pseudo-code for training and sampling length-only (Algorithm~\ref{alg:train_length}, \ref{alg:sample_length}), order-agnostic (Algorithm~\ref{alg:train_set}), and conditional order-preserving (Algorithm~\ref{alg:train_cond}, \ref{alg:sample_cond}) \ourmodel. The three unconditional variants are summarized in Figure~\ref{fig:train_sample}, which highlights the differences in the training objectives and sampling predictions.

\begin{algorithm}[H]
\caption{Training Multimodal \ourmodel}\label{alg:train_set}
\begin{algorithmic}[1]
\WHILE{not converged}
    \STATE Sample $t\sim[0,1];Y_1=(X_1,S_1)\sim q(X,S)$.
    \STATE Sample $S_t\sim p_t(S)$ for each $S_t^i,1\le i\le X_1$.
    \STATE For each element $S_t^i\in S_t$, delete it with probability $1-\kappa_t$ to obtain $\hat{S}_t$ and $X_t=|\hat{S}_t|,Y_t=(X_t,\hat{S}_t)$.
    \STATE Predict the rate $\lambda_{\theta}(Y_t,t)$, the length-dependent modality vector field $v_{\theta}(Y_t,t)$, and the sampling distribution $\rho_{\theta}(Y_t,t)$.
    \STATE Optimize the loss $\mathcal{L}$ in Eq.~\ref{eqn:loss}. 
\ENDWHILE
\end{algorithmic}
\end{algorithm}

\begin{algorithm}[H]
\caption{Training Length-Only \ourmodel}\label{alg:train_length}
\begin{algorithmic}[1]
\WHILE{not converged}
    \STATE Sample $t\sim[0,1];X_1\sim q(X)$.
    \STATE Sample $X_t\sim B(X_1,\kappa_t)$.
    \STATE Predict the rate $\lambda_{\theta}(X_t,t)$ and optimize the loss $\mathcal{L}_\text{PP}$ in Eq.~\ref{eqn:nll}. 
\ENDWHILE
\end{algorithmic}
\end{algorithm}

\begin{algorithm}[H]
\caption{Sampling from Length-Only \ourmodel (Euler)}\label{alg:sample_length}
\begin{algorithmic}[1]
\STATE $X_0=0$.
\FOR{$i\gets 1,2,\dots,N$}
    \STATE $t\gets t_i,\Delta t\gets t_{i+1}-t_i$, where $t_{N+1}:=1$.
    \STATE Predict the rate $\lambda_{\theta}(X_t,t)$.
    \STATE Sample an event with a probability of $\lambda_\theta\Delta t$.
    \STATE If an event happens, $X_{t+\Delta t}=X_t+1$; else $X_{t+\Delta t}=X_t$.
\ENDFOR
\STATE \textbf{Return:} $X_1$.
\end{algorithmic}
\end{algorithm}

\begin{algorithm}[H]
\caption{Training Conditional Order-Preserving \ourmodel}\label{alg:train_cond}
\begin{algorithmic}[1]
\WHILE{not converged}
    \STATE Sample $t\sim[0,1];Y_1=(X_1,S_1)\sim q(X,S)$.
    \STATE Sample condition $Y_0=(X_c,S_c)\sim q_c(X,S|X_1,S_1)$ with clean values.
    \STATE Sample $S_t\sim p_t(S)$ for each non-conditioning position $i$, $S_t^i,1\le i\le X_1$.
    \STATE For each non-conditioning position $i$, delete $S_t^i$ with probability $1-\kappa_t$ to obtain $\hat{S}_t$ and $X_t=|\hat{S}_t|,Y_t=(X_t,\hat{S}_t)$.
    \STATE Predict the rates $\lambda^i_{\theta}(Y_t,t)$, the length-dependent modality vector field $v_{\theta}(Y_t,t)$, and the sampling distributions $\rho^i_{\theta}(Y_t,t),0\le i\le X_t$.
    \STATE Optimize the loss $\mathcal{L}$ in Eq.~\ref{eqn:loss}. 
\ENDWHILE
\end{algorithmic}
\end{algorithm}

\begin{algorithm}[H]
\caption{Sampling from Conditional Order-Preserving \ourmodel (Euler)}\label{alg:sample_cond}
\begin{algorithmic}[1]
\STATE \textbf{Input:} $Y_0=(X_c,S_c)$.
\FOR{$i\gets 1,2,\dots,N$}
    \STATE $t\gets t_i,\Delta t\gets t_{i+1}-t_i$, where $t_{N+1}:=1$.
    \STATE Predict the rates $\lambda^i_{\theta}(Y_t,t)$, the vector field $v_{\theta}(Y_t,t)$, and the sampling distributions $\rho_{\theta}^i(Y_t,t)$ for $0\le i \le X_t$.
    \STATE Advance the generative flow for each $S_t$ with $v_{\theta}$ to obtain $\hat S_{t+\Delta t}$ for each non-conditioning position $i$.
    \STATE Zero out the rates where insertions are not allowed (e.g., inside a contiguous condition sequence).
    \STATE Simulate the Poisson process by sampling an event on $X_{t+\Delta t}$ with a probability of $\lambda_\theta\Delta t$ where $\lambda_\theta=\sum_i\lambda^i_\theta$.
    \STATE If an event happens, choose the insertion position $i$ with probabilities proportional to $\{\lambda_i\}$ and insert the new feature value $s_i\sim\rho^i_{\theta}$ into $\hat S_{t+\Delta t}$ to obtain $S_{t+\Delta t}$.
    \STATE Update $Y_{t+\Delta t}=(X_{t+\Delta t},S_{t+\Delta t})$.
\ENDFOR
\STATE \textbf{Return:} $Y_1=(X_1,S_1)$.
\end{algorithmic}
\end{algorithm}

In all sampling algorithms, the sequence of discretization time steps $0\le t_1<t_2<\cdots<t_N<t_{N+1}=1$ is defined by an inference-time scheduler, and more advanced solvers (e.g., the second-order Heun solver) can be readily applied.

\subsection{Flow Parameterization}\label{suppl:flow_param}
\paragraph{Continuous Euclidean Modality}
Existing works have explored fixed-length diffusion or flow-based generative models for both continuous and discrete modalities, whose loss can be adapted as our $\mathcal{L}_\text{FM}$ for refining the existing length-dependent elements. For example, consider a flow matching model with the optimal-transport path~\cite{lipman2022flow} $s_t=(1-t)s_0+ts_1$, where $s_0\sim\mathcal{N}(0,\sigma_0^2)$. The flow matching loss is
\begin{equation}
    \mathcal{L}_\text{FM}=\mathbb{E}_{t,s_0\sim p_0,s_1\sim q}[\|v_\theta(s_t,t)-(s_1-s_0)\|^2],\label{eqn:loss_ot_fm}
\end{equation}
where $v_\theta(s_t,t)$ is the vector field predictor. We follow Proteina in using this loss to refine the continuous CA coordinates in our structure design model. 
For sampling from such a continuous vector field, a simple Euler update reads:
\begin{equation}
    s\gets s+v_\theta\Delta t.
\end{equation}
We may also employ more complex sampling algorithms that incorporate higher-order information or stochasticity (e.g., SDE sampling). For example, one sampling step in Proteina~\cite{geffner2025proteina} reads:
\begin{equation}
    s\gets s+(v_\theta+g_t s_\theta)\Delta t+\sqrt{2g_t\gamma\Delta t}\,\varepsilon.\label{eqn:sample_sde}
\end{equation}
where $s_\theta(s_t,t):=(tv_\theta(s_t,t)-s_t)/(1-t)$ is the score function, $g_t$ is the score scheduler, $0<\gamma<1$ is the noise scale, and $\varepsilon\sim \mathcal{N}(0,1)$. We also adopt this SDE sampling for the coordinates.

\paragraph{Continuous Riemannian Modality}
For Riemannian modalities such as rotations (which lie in the 3D rotation group $\mathrm{SO}(3)$) and torsion angles (which lie in the flat torus $\mathbb{T}$) in the peptide co-design task, the standard Riemannian flow matching loss~\cite{chen2024flow} is applied:
\begin{equation}
    \mathcal{L}_\text{FM}=\mathbb{E}_{t,s_0\sim p_0,s_1\sim q}[\|v_\theta(s_t,t)-u_t(s_t|s_1)\|_g^2],\label{eqn:loss_rfm}
\end{equation}
where geodesic interpolation is used as the noising process $s_t:=\exp_{s_0}(\kappa_t\log_{s_0}s_1)$ with a differentiable scheduler $\kappa_t$. The $\exp$ denotes the exponential map, $\log$ the logarithm map, and $\|\cdot\|_g$ the Riemannian norm. The conditional vector field $u_t$ can be calculated in closed form as
\begin{equation}
    u_t(s_t|s_1)=\frac{\dot\kappa_t}{1-\kappa_t}\log_{s_t}s_1.
\end{equation}

Sampling from Riemannian flow matching follows the Euler discretization as:
\begin{equation}
    s\gets \exp_s(v_\theta\Delta t).
\end{equation}

The mathematical expression for the Riemannian operations on $\mathrm{SO}(3)$ and $\mathbb{T}$ is known in closed form. We refer interested readers to, e.g., Appendix A of \citet{li2024full} for full details.

\paragraph{Discrete Modality}
For continuous-state discrete flow matching models~\cite{cheng2025alpha} that embed the discrete token into a continuous modality, the previous discussion of continuous-modality losses applies directly~\cite{hoogeboom2021argmax,cheng2024categorical,dunn2024mixed}. 
For example, the simplex-based \ourmodel variant for peptide design adopts the soft logits $s_1=K(2\delta_{c_1}-1)$, where $c_1$ is the corresponding amino acid type and $K=5$. $s_0$ is sampled from $\mathcal{N}(0, KI)$, and Euclidean geometry is adopted. Following~\citet{li2024full}, the cross-entropy loss is used instead of the vector field matching loss:
\begin{equation}
    \mathcal{L}_\text{FM}=\mathbb{E}_{t,s_0\sim p_0,s_1\sim q}[\mathrm{CE}(f_\theta(s_t,t),c_1)],
\end{equation}
where $f_\theta(s_t,t)$ is the logit predictor (probability denoiser).

For the discrete modality that relies on a continuous-time Markov chain, the intermediate noisy $s_t$ is a discrete token rather than continuous logits. The discrete flow matching loss~\cite{gat2024discrete,sahoo2024simple} reads:
\begin{equation}
    \mathcal{L}_\text{FM}=\mathbb{E}_{t,s_0\sim p_0,s_1\sim q}\left[\sum_{s_t\ne s_1}\mathrm{CE}(f_\theta(s_t,t),s_1)\right],
\end{equation}
where the noisy $s_t$ is sampled from some pre-defined probability path on the noisy logits $x_t$:
\begin{equation}
    s_t:=
\begin{cases}
    \sim \mathrm{Cat}(x_t), & \text{w.p. } 1-t,\\
    s_1, & \text{w.p. } t.
\end{cases}
\end{equation}
A simple yet effective practice follows mask diffusion~\cite{sahoo2024simple}: each noisy token starts with a special mask token $m$ with $s_0\equiv\delta_m$, and the FM loss is computed over all mask tokens. During sampling, a single mask diffusion update reads:
\begin{equation}
    s\gets
\begin{cases}
    \sim \mathrm{Cat}(f_\theta(s_t,t)), & \text{if } s_t=m \text{ and w.p. } \Delta t/(1-t),\\
    s_t, & \text{else}.
\end{cases}
\end{equation}
Notably, we may consider a special case in which the insertion element type is fixed to be the mask token, leaving no reconstruction loss $\mathcal{L}_\text{rec}=0$ for the discrete modality. In this case, the model completely relies on the non-zero mask diffusion loss $\mathcal{L}_\text{FM}$ to determine the final discrete type.

\subsection{Sampling Distribution Parameterization}\label{suppl:sample_param}
The marginal sampling distribution in Theorem~\ref{thm:multimodal} and \ref{thm:order} is a mixture distribution of the forward noisy length-dependent modality from which sampling is feasible (when conditioned on the clean target). Therefore, we can rely on another generative model for modeling such a distribution, whose general NLL loss form is described in Eq.~\ref{eqn:loss_rec} and \ref{eqn:loss_rec_order}.

\paragraph{Continuous Euclidean Modality}
We first consider the continuous modality case where directly regressing on the sampling marginal $\rho_t$ is often intractable. We may rely on another flow matching head or any generative model to approximate $\rho_t$ parametrically, using the current state $S_t$ and the noisy element to be inserted as inputs.
While such a nested flow approach is theoretically preferred, it requires simulating an additional full flow sampling for each Poisson event, thereby substantially slowing the sampling process. 
In practice, we found that using the MSE loss to regress the mean behavior of $S_t$ already performs well enough with the existing later denoising steps. The simplified loss can be written as:
\begin{equation}
    \mathcal{L}_\text{rec}=\mathbb{E}_{Y_1,Y_t}\left[h_t\sum_{k=1}^{X_1-X_t}(f_\theta(Y_t)-S_{t}^k)^2\right].\label{eqn:loss_rec_mean}
\end{equation}

\paragraph{Continuous Riemannian Modality}
The simplified reconstruction loss can be easily extended to Riemannian generative modeling (e.g., rotations and torsion angles in peptide design), which takes:
\begin{equation}
    \mathcal{L}_\text{rec}=\mathbb{E}_{Y_1,Y_t}\left[h_t\sum_{k=1}^{X_1-X_t}d_g^2(f_\theta(Y_t),S_{t}^k)\right],\label{eqn:loss_rec_rfm}
\end{equation}
where $d_g^2(x,y)=\|\log_x y\|_g^2$ is the squared geodesic distance.

\paragraph{Discrete Modality}
For Markov-chain-based discrete flow matching models, \citet{gat2024discrete} has established a connection between the discrete vector field and the probability flow over categorical data, where the sampling marginal can be obtained in closed form. Following this, the KL-divergence (cross-entropy loss) is used to train the discrete flow, and the reconstruction loss can be calculated as a summation of weighted cross-entropy losses as:
\begin{equation}
\mathcal{L}_\text{rec}=\mathbb{E}_{Y_1,Y_t}\left[h_t\sum_{k=1}^{X_1-X_t}\mathrm{CE}(f_\theta(Y_t),S_{t}^k)\right]\label{eqn:loss_dfm},
\end{equation}
where $f_\theta$ is the logit predictor. Thus, unlike the continuous case, the discrete case is always simulation-free, as the learned logits already encode the marginalized discrete sampling distribution. 

In contrast to the flow-oriented loss discussed in the previous subsection, we may consider another special case in which the element type is determined upon insertion, yielding no refinement loss with $\mathcal{L}_\text{FM}=0$. In this approach, the sampling distribution serves a dual role in determining the discrete types without further refinement. This parameterization and sampling scheme are adopted for our sequence design model. Nonetheless, one may also allow token refinement relying on the discrete diffusion~\cite{austin2021structured,sahoo2024simple} or discrete flow matching~\cite{gat2024discrete} loss $\mathcal{L}_\text{FM}$ to further refine the discrete modality.


For all time-dependent modalities, the training noise schedulers and inference time sampling schedulers for the flow and reconstruction parts can be either coupled with the rate scheduler or separate. We typically use different schedulers for the rate and length-dependent modality to allow greater flexibility during sampling. In this way, time steps across modalities are not shared during training, and separate time embeddings are fed into the model.

\subsection{\texorpdfstring{$\tau$}{τ}-Leaping} \label{suppl:tau_leap}
One practical limitation of existing rate-based variable-length generative models is their slower sampling speed compared to standard fixed-length flow sampling, because we must simulate the inhomogeneous Poisson process at each time step, which is more sensitive. For example, EditFlow~\cite{havasi2025edit} required 5{,}000-10{,}000 sampling steps to achieve reasonable performance, which is significantly larger than the average sequence length and the typical number of steps used in fixed-length diffusion or flow matching models.
To speed up the sampling process, we consider simulating multiple Poisson events within a small time step $\tau$ by leveraging the \emph{$\tau$-leaping} method~\cite{gillespie2001approximate}. Specifically, instead of sampling one event with a success probability of $\lambda \tau$, the increment is now sampled from a Poisson distribution with a rate of $\lambda \tau$:
\begin{equation}
    X_s-X_t\sim\mathcal{P}\left(\lambda(X_t,t)\tau\right), \quad\tau=s-t.\label{eqn:poisson_leaping}
\end{equation}
Note that we have $\Pr(X_s-X_t=0)=\exp{(-\lambda\tau)}$, which is close to $1-\lambda\tau$ when $\lambda\tau$ is small, recovering the standard simulation procedure. Nonetheless, when $\lambda\tau$ is large, the $\tau$-leaping method can significantly reduce the number of sampling steps while still maintaining a good approximation to the original process. As an example, the ground-truth marginal rate on the PDB dataset at $t=0$ is the average sequence length, which is approximately 166. With a standard 400-step sampling setup in Proteina, the initial rate is approximately 0.4, at which point the approximation will fail.

In addition, we also allow for a more general form of $\tau$-leaping, where simultaneous insertions can happen at multiple places independently according to their own rates and sampling distributions~\cite{li2007analysis}. Such a multi-place insertion better approximates the original process than the Bernoulli approximation. 
Nonetheless, we note that $\tau$-leaping relies on the assumption that a single event does not significantly affect predictions, which does not necessarily hold even for discrete diffusion with a fixed time step~\cite{ren2025fast}. In our scenario, where length is also variable, such an approximation error may be more pronounced. Despite this, we empirically find that both forms of $\tau$-leaping improve performance and efficiency in the 400-step generation setup.



\subsection{Additional Related Work}
\label{suppl:related}
Here, we provide additional related work on variable-length generative frameworks that primarily focus on text or image generation, as well as other autoregressive frameworks that naturally support variable-length protein generation. Although these works don't directly adapt to general protein structure design or co-design, they still offer useful methodological insights.

\paragraph{Protein Design Models.}
Beyond the structure generative models discussed in the main text, related approaches address protein sequence design, multimodal protein generation, and peptide design.
ProteinMPNN~\cite{dauparas2022robust} and ESM-IF1~\cite{hsu2022learning} perform inverse folding by generating amino acid sequences conditioned on backbone structures.
ESM3~\cite{hayes_simulating_2025} jointly models protein sequences, structures, and functional annotations through a token-based framework, enabling protein generation with multimodal conditioning.
Generative modeling has also been applied to the design of peptides targeting specific proteins.
PepMimic~\cite{pepmimic} designs peptide binders by mimicking the binding interfaces of existing protein binders.
PepHAR~\cite{pephar} introduces a hotspot-driven autoregressive framework for peptide design through multi-fragment extension and iterative refinement.
PocketXMol~\cite{pocketxmol} provides a unified atom-level generative framework based on atomic interactions, supporting diverse molecular design tasks, including peptide generation.

\paragraph{Variable-length Generative Frameworks.}
We further discuss alternative tokenization and generation schemes for proteins, together with variable-length generative frameworks developed in other domains.
PAR~\cite{qu2026protein} adopts a multi-scale framework that generates protein backbones through coarse-to-fine next-scale prediction, combining an autoregressive transformer with a flow-based decoder.
Although it compresses protein structures into fixed-length coarse token representations, generation still requires a specified protein length.
APT~\cite{dilip2026adaptive} uses an adaptive tokenizer in which successive tokens encode increasingly fine-grained details of a global structural representation.
The final protein length is inferred from the first few coarse-grained global tokens during generation.
Autoregressive protein language models~\cite{wang2025comprehensive,qiu2024instructplm} generate amino acid sequences token by token, naturally supporting variable-length sequence generation.
Autoregressive modeling has also been extended to structure tokens to incorporate structural information into protein generation~\cite{gaujac2024learning}.
In other domains, OneFlow~\cite{nguyen2025oneflow} combines insertion-based Edit Flows for discrete text tokens with flow matching for continuous image latents, enabling variable-length, interleaved text--image generation.

%% file: suppl/C_exp.tex
\section{Experiment Details}\label{suppl:exp}
In this section, we detail the datasets, task-specific architectures, and evaluation pipelines. 
The two structure-based generation tasks are built on the Proteina base model (with different feature variants) and use the same datasets; the two sequence-design tasks share a separate backbone family, and peptide design is described in its own subsection. We therefore organize this section by dataset/model architecture.

Proteins are naturally ordered according to their amino acid sequence. We always use the order-preserving \ourmodel for generative modeling in all tasks. Table~\ref{tab:config} summarizes the per-task modeling choices, including the modality geometry, the flow and reconstruction losses, the insertion scheduler, and the base model architecture.
\input{tabs/config}

\subsection{Protein Structure Generation}
\subsubsection{Data Specification}\label{suppl:data}
The datasets we use to train the model include PDB~\cite{berman2000protein}, the experimentally validated protein structure database, and AFDB~\cite{jumper2021highly}, the AlphaFold2-predicted structure database that contains more structures.
For the PDB dataset, we use the Proteina preprocessing script to filter entries that have a single chain, contain no non-canonical amino acids, and have lengths between 50 and 256. The data are then partitioned using a 50\% sequence-similarity threshold to prevent information leakage within the same cluster, yielding 6822 distinct clusters, for a total of 27k training structures. During training, we follow Proteina's approach to uniformly sample one cluster and then a member of that cluster.

For AFDB, we follow Proteina to filter and cluster the original 21M AFDB structures using sequence-based MMseqs2~\cite{steinegger2017mmseqs2} and structure-based FoldSeek~\cite{van2024fast}, resulting in 713k structures. Unlike PDB, we use the standard uniform sampling strategy. As AFDB contains much longer proteins, each structure is dynamically cropped to a maximum sequence length of 512 to avoid memory issues during training.

\subsubsection{Model Specification}\label{suppl:model}
\paragraph{Architecture.}
We build each structure model on top of its corresponding 60M Proteina~\cite{geffner2025proteina} architecture and add heads to the final hidden representation to learn per-position rates and sampling distributions, thereby increasing the total number of trainable parameters to 65 M. In both unconditional generation and motif scaffolding, the only length-dependent variable is the continuous CA coordinates.

For motif scaffolding, we adapt Proteina's dedicated motif architecture variant, which includes additional information on the motif mask, coordinates, and pairwise distances. This architecture is not checkpoint-compatible with the unconditional Proteina model, so neither method transfers an unconditional checkpoint into the motif experiment. All residue indices are renumbered to avoid exposing their original positions. We additionally inject sequence embeddings for the motif residues (if available), with other non-motif residues set to unknown (X). We adopt a hard constraint in an inpainting fashion: the motif coordinates and types are never corrupted or modified and incur no loss during either training or sampling.

\paragraph{Training.}
During training, we generally follow the original setup and hyperparameters in Proteina, where the time step is sampled from $t\sim0.98\mathrm{Beta}(1.9,1.0)+0.02U[0,1]$. For the length modality, we adopt the early-completion principle used by TDDM~\citep{campbell2023trans} and set $\kappa_t:=\min\{1,t/0.6\}$. The length time $t\sim U[0,1]$ is sampled independently from the continuous time and is fed to the model as an additional input. As defined in Theorem~\ref{thm:forward_rate}, the process is absorbed once $\kappa_t=1$: for $t\geq0.6$, all target elements are present, the insertion rate and insertion-loss contribution are set to zero, and only the within-length refinement process remains active. The OT-path-based flow matching loss in Eq.~\ref{eqn:loss_ot_fm}, the NLL loss for the Poisson process in Eq.~\ref{eqn:nll}, and the mean reconstruction loss in Eq.~\ref{eqn:loss_rec_mean} are used with weights $w_\text{FM}=w_\text{rec}=1$. For motif scaffolding, we use the conditional generation setup, in which the motif CA coordinates remain unchanged and are never deleted during the forward noising process. The other losses are identical.

For unconditional generation, two separate 65M models are trained: one with a total batch size of 128 and 300k iterations on PDB, and another with a total batch size of 32 and 1.6M iterations on AFDB.
For motif scaffolding, both the 65M \ourmodel model and the corresponding Proteina baseline use the AFDB dataset and the same motif augmentation. Our model is trained for 1M iterations with a total batch size of 192.
We follow~\citet{lin2024out} to randomly sample 0-4 contiguous parts (contigs) as motifs during training. All models are trained with the Adam optimizer with a learning rate of $10^{-4}$. An exponential moving average (EMA) of the model parameters is maintained with a decay of 0.999 and used for all evaluations.
All models were trained on 8 NVIDIA A100 GPUs, taking approximately 2 days on PDB and 5 days on AFDB.

\paragraph{Sampling.}
During sampling, 400 sampling steps are used following the Proteina setup. For the continuous modality, we also follow Proteina to utilize SDE sampling described by Eq.~\ref{eqn:sample_sde} with $g_t=1/t$, $\gamma=0.35$, and an exponential inference scheduler. For the length modality, we use $\kappa_t:=\min\{1,t/0.3\}$ with a constant rate scaling factor of 0.7. Consequently, insertions are completed by $t=0.3$, and the remaining trajectory refines the fixed set of generated coordinates.

For motif scaffolding, motif contigs are provided as the model's initial input. Insertions are allowed between motif contigs and at the beginning and the end, but not within the contiguous motif. The motif parts are centered and remain unchanged during the refinement for the continuous coordinates. We track the lengths of each contig (motifs and scaffolds) to determine the new motif positions for the next input in the iterative sampling process.
The same sampling hyperparameters as in the unconditional setup are used, except for a rate scaling factor of 0.6 to compensate for the longer length distribution in AFDB.

For unconditional design, we use the $\tau$-leaping technique described in Appendix~\ref{suppl:tau_leap} to enable high-quality sampling with the 400-step sampling setup. Simultaneous insertion at multiple places is allowed, and multiple insertions at the same place are also allowed, with the count sampled from the corresponding Poisson distribution. In contrast to Proteina, classifier-free guidance (CFG) and self-conditioning are not used. For motif scaffolding, self-conditioning is enabled with $\tau$-leaping disabled.

\subsubsection{Evaluation Pipeline}\label{suppl:eval}
In our evaluations of structure quality and diversity, the following metrics are used:

\textbf{Designability.}
For each protein structure, we first use ProteinMPNN~\cite{dauparas2022robust} to design 8 candidate sequences, which are then folded by ESMFold~\cite{lin2023evolutionary} to give the refolded structures. The self-consistency root-mean-square deviation (scRMSD), defined as the RMSD between the generated and refolded structures, is used as the primary indicator of structural fidelity. A \emph{designable} protein structure in unconditional protein generation is defined as having a minimum scRMSD $<$ 2Å among the 8 refolded candidates. The scRMSD is always calculated on the CA coordinates due to Proteina's CA-only architecture. \emph{Designability} is defined as the proportion of designable protein structures in all generations. The designability score assesses the quality of structural generation, since natural protein structures and scaffolds should be designable.

\textbf{Scaffold Designability.}
For motif scaffolding, the following additional metrics are calculated:
\begin{itemize}[leftmargin=*]
    \item \textbf{motifRMSD}, which is the RMSD between the structure of the refolded motif region and the ground-truth motif structure specified as the condition. We extract the motif residues in each refolded structure to calculate motifRMSD. motifRMSD is calculated over the four backbone atoms (N, C, CA, O) instead of CA only.
    \item \textbf{pLDDT} (predicted local distance difference test score)~\citep{mariani2013lddt}, which estimates structural fidelity on a 0-100 scale given the residue sequence. pLDDT is calculated for the ProteinMPNN-redesigned sequences.
    \item \textbf{pAE} (predicted aligned error), which measures the alignment error of residue pairs between the generated and refolded structures.
\end{itemize}
Additionally, the amino acid sequences in the motif template are fixed and will not be redesigned by ProteinMPNN unless specified as unknown (X) in the template. This ensures that motif information, such as enzymatic residues, is preserved during sequence redesign.
Following~\citet{rfdiffusion}, a scaffold is considered \emph{successful} if at least one refolded structure (out of 8 candidates) satisfies scRMSD $\le$ 2Å, motifRMSD $\le$ 1Å, pLDDT $\ge$ 70, and pAE $\le$ 5Å. Successful designs are clustered as described below, and each resulting cluster counts as one \emph{unique success}.

\textbf{Diversity.}
\emph{Diversity} is defined as the proportion of designable clusters with respect to the total number of designable samples.
All predicted structures are clustered with FoldSeek~\citep{van2024fast} at a specified TM-score threshold~\citep{zhang2004tmscore}, TM-score is a length-independent measure of global fold similarity. We calculate the diversity using the following FoldSeek~\cite{van2024fast} command:
\begin{verbatim}
foldseek easy-cluster <path_samples> <path_tmp>/res <path_tmp> \
  --alignment-type 1 --cov-mode 0 --min-seq-id 0 \
  --tmscore-threshold <threshold>
\end{verbatim}
where the threshold is 0.5 for unconditional generation and 0.6 for motif scaffolding.

\textbf{Novelty.}
We assess the structural novelty score via exhaustive search against a structure database with FoldSeek. \emph{Novelty} is defined as the maximum TM-score across all hits for each generated structure, averaged across all generations.
We calculate the novelty score using the following FoldSeek command:
\begin{verbatim}
foldseek easy-search <path_samples> <path_db> out.tsv <path_tmp> \
  --alignment-type 1 --exhaustive-search --tmscore-threshold 0.0 \
  --max-seqs 10000000000 --format-output query,target,alntmscore,lddt
\end{verbatim}
where \verb|path_db| is the path prefix to either the full PDB structure database or the filtered 713k AFDB subset, resulting in the two score variants in Table~\ref{tab:uncond}. The output TSV file contains results from the exhaustive search against the designated database, which is then processed to extract the highest TM-score.

\textbf{Secondary structure composition.}
Per-residue three-state secondary structure is assigned with DSSP~\citep{kabsch1983dssp} via the Biotite library~\citep{kunzmann2018biotite}.
We report the average per-protein helix ($\alpha$) and strand ($\beta$) fractions in the generated structures, and the coil fraction is given by $1-\alpha-\beta$.

For unconditional protein design, we calculate all metrics over 500 generations per model. For \ourmodel, the 512 lengths are determined by the learned rate. Following the published fixed-length benchmark, each baseline receives 100 requests at each length in $\{50,100,150,200,250\}$. 
We note that sampling baseline lengths from the empirical PDB marginal would instead constitute an oracle-conditioned diagnostic because the fixed-length models do not learn $p(L)$. We therefore use the standard grid for the primary end-to-end comparison and report \ourmodel's binned designability separately in Appendix~\ref{suppl:result}.

For motif scaffolding, we compute the metrics on 1000 generated structures for each motif target, using FoldSeek's clustering utility to identify the number of unique successes described above.
For motif tasks 5TRV, 6E6R, 6EXZ, and 7MRX, the original RFDiffusion benchmark provides contig templates across 3 different length ranges. Each contig template specifies the minimum and maximum lengths of each linker segment and of the entire backbone. For fixed-length generative models like Proteina, given a contig template, the model first randomly samples the lengths of each linker segment to satisfy the contig constraints, then fixes the positions of these linker residues and redesigns their structures. In this way, the length distribution in Figure~\ref{fig:motif_length} exhibits 3 modes corresponding to the long, medium, and short length ranges.
For the baselines, we use the results from the best contig configuration on tasks 5TRV, 6E6R, 6EXZ, and 7MRX. In contrast, \ourmodel does not rely on predefined contig templates or length-range constraints. Instead, we sample 1000 backbones conditioned solely on the motif segments, yielding a continuous length distribution without the 3 distinct modes.

\subsubsection{Baseline Information}

For protein structure design, we compare \ourmodel against the following baselines:

\textbf{Proteina}~\cite{geffner2025proteina} is a CA-only structure design model based on flow matching. For unconditional generation on PDB, a 60M model is retrained on the same PDB dataset using the same training hyperparameters as our \ourmodel for 300k iterations, with a total batch size of 128. For AFDB, the official 60M unconditional Proteina checkpoint is used. For motif scaffolding, we use the separate 60M Proteina motif architecture with the same AFDB training set and motif augmentation as our motif \ourmodel. Because the unconditional and motif architectures expose different inputs, their checkpoints are not interchangeable; each experiment therefore uses its corresponding Proteina base model.
For all Proteina models, we use exactly the same sampling hyperparameters as \ourmodel: 400 sampling steps, a noise scale of 0.35, and an exponential scheduler. No self-conditioning is used, as suggested in the default sampling config.

\textbf{FrameDiff}~\cite{yim2023se} is a Riemannian diffusion model that learns a rotation (orientation) and a translation (position) for each amino acid in a protein. \textbf{FrameFlow}~\cite{yim2023fast} further extends it by incorporating Riemannian flow matching, with improved training stability and performance reported in the original work. The results are directly taken from the Proteina paper, which runs inference using the provided checkpoint following the recommended hyperparameters.

\textbf{RFDiffusion}~\cite{rfdiffusion} is a diffusion-based protein backbone generative model that leverages the structure prediction power of RoseTTAFold~\cite{baek2021accurate}. RFDiffusion also supports motif scaffolding by incorporating motifs as conditions. The unconditional generation results for RFDiffusion are taken directly from the Proteina paper, whereas the motif scaffolding results are taken directly from the original paper.

\textbf{Genie2}~\cite{lin2024out} is a diffusion-based generative model with a specialized multi-motif framework that designs co-occurring motifs. The motif scaffolding results are taken directly from the Proteina paper, which follows the recommended sampling hyperparameters in the original paper.

\subsection{Protein Sequence Generation}\label{suppl:cfg}
\subsubsection{Data Specification}\label{suppl:seq_data}
We train our sequence-generation \ourmodel on the UniRef50 database~\cite{suzek_uniref_2007}, containing 41{,}887{,}589 representative protein sequences obtained by clustering UniProt~\cite{the_uniprot_consortium_uniprot_2023} sequences at 50\% sequence identity.
We follow the data processing in EvoDiff~\cite{alamdari_protein_2023}, which results in a dataset split into 41{,}546{,}293 training, 82{,}929 validation, and 48{,}941 test sequences. We further filter training sequences to retain only those with lengths between 10 and 1024 residues, yielding 40{,}471{,}430 training sequences with a mean length of 247.7 residues. Each sequence is tokenized at the single-residue level over a 21-character alphabet comprising the 20 standard amino acids plus an unknown token~X (\texttt{ACDEFGHIKLMNPQRSTVWYX}).
No subsampling of long sequences is applied; all sequences within the length range are used as-is. For sequence-based motif scaffolding, we finetune on a subset of UniRef50 filtered to sequences with AlphaFold2~\cite{jumper2021highly} pLDDT $\geq 95$ and lengths between 10 and 200, yielding 480,143 sequences.   

\subsubsection{Model Specification}\label{suppl:seq_model}
\paragraph{Architecture.}
We use the Diffusion Transformer (DiT)~\cite{peebles_scalable_2023} architecture for both the insertion and deletion models.
Each model consists of 14 transformer blocks with a hidden dimension of 1600, 25 attention heads, totaling 642M parameters. Dropout is set to 0.1. Length conditioning is performed via a continuous sinusoidal length embedding over the range $[10, 1024]$.


\paragraph{Training.}
Both models are trained with AdamW~\citep{kingma_adam_2017} ($\beta_1 = 0.9$, $\beta_2 = 0.95$, no weight decay) and a learning rate of $3 \times 10^{-4}$ with 2{,}000 warmup steps followed by a cosine annealing schedule ($\eta_{\min} = 10^{-5}$). Gradient clipping is applied at a norm of 5.0. Dynamic batching is used with a budget of 80{,}000 tokens per batch with approximate random-length sampling. 
We use gradient accumulation over 4 micro-batches, yielding an effective batch size of 320{,}000 tokens. An EMA of the model parameters is maintained with a decay of 0.9999 and is used for all evaluations.

To enable classifier-free guidance during sampling, the length condition $c$ is randomly dropped with probability 0.2 during training, so that a single model learns both the conditional edit rate $u_t^\theta(x | x_t, c)$ and the unconditional rate $u_t^\theta(x | x_t)$. At sampling time, these are combined as described in Eq.~\ref{eqn:cfg_rate} and \ref{eqn:cfg_q}. Training was carried out on four NVIDIA RTX PRO 6000 GPUs for 100k gradient updates, taking approximately three days.

\paragraph{Sampling.}
Generation proceeds by simulating the learned CTMC from the empty sequence ($t=0$) to the data distribution ($t=1$) over a uniform time grid with step size $\tau = 1/T$ (Euler solver).
Each step independently samples the number of insertions at each position $i$ from a Poisson distribution with intensity $\lambda_{t,i}^{\mathrm{ins}}(X_t)\tau$, and draws inserted token identities from $Q_{t,i}^{\mathrm{ins}}(\cdot | X_t)$.
This first-order approximation corresponds to the linearization $1 - e^{-\lambda \tau} \approx \lambda\tau$.

To mitigate discretization errors that accumulate over these Euler steps, we employ the corrector procedure of \citet{havasi2025edit}, which leverages both the forward rate $u_t$ (insertion) and reverse rate $\tilde{u}_t$ (deletion). The combined rate $(1+\alpha)u_t + \alpha\tilde{u}_t$ preserves the marginal distribution $p_t$ for any corrector strength $\alpha \geq 0$. Note that when $\alpha = 0$, this reduces to the standard forward-only Euler sampler.
This allows the introduction of additional self-correction without altering the target distribution.
Concretely, at each time step $t_k$ with step size $h$:
\begin{enumerate}[leftmargin=1.5em,itemsep=2pt]
  \item \textbf{Insertion sub-step.}
    For each position $i$ of the current sequence $X_{t_k}$, we sample insertions from a Poisson distribution with intensity $(1+\alpha)\lambda_{t_k,i}^{\mathrm{ins}}(X_{t_k})\tau$.     Inserted token identities are drawn from $Q_{t_k,i}^{\mathrm{ins}}(\cdot | X_{t_k})$. This produces an intermediate state $Y$ that overshoots by $\alpha \tau$ in time.
  \item \textbf{Deletion sub-step.}
    For each position $i$ of $Y$, we sample whether to delete that position with probability $\alpha\tilde{\lambda}_{t_k',i}^{\mathrm{del}}(Y)\tau$, where $t_k' = t_k + (1+\alpha)\tau$ and $\tilde{\lambda}^{\mathrm{del}}$ is the learned reverse (deletion) rate. Deleted positions are removed, yielding $X_{t_{k+1}}$. This pulls back by $\alpha\tau$, so the net time advancement is $\tau$.
\end{enumerate}

We further apply classifier-free guidance (CFG) independently to both the rate magnitude $\lambda$ and the token distribution $Q$ for the insertion model and only to the latter for the deletion model. Following the formulation in \citet{havasi2025edit}, the guided quantities are
\begin{align}
  \tilde{\lambda}_{t,i}(x_t, c) &= \lambda_{t,i}(x_t | c)^{1+w}\,\lambda_{t,i}(x_t)^{-w}, \label{eqn:cfg_rate}\\
  \tilde{Q}_{t,i}(a | x_t, c) &\propto Q_{t,i}(a | x_t, c)^{1+w}\,Q_{t,i}(a | x_t)^{-w}, 
  \label{eqn:cfg_q}
\end{align}
where $w \geq 0$ is the guidance weight, $c$ is the target length, and the conditional and unconditional outputs are produced by the same model (with and without the length condition, respectively).
Eq.~\ref{eqn:cfg_rate} corresponds to log-linear extrapolation in rate space, $\log \tilde{\lambda} = (1+w)\log\lambda(\cdot | c) - w\log\lambda(\cdot)$, amplifying the conditional rate relative to the unconditional baseline. Eq.~\ref{eqn:cfg_q} applies log-linear extrapolation in token-distribution space, analogous to the rate guidance.
Together, guidance modulates both \emph{where} edits occur (via $\tilde{\lambda}$) and \emph{what} tokens are inserted (via $\tilde{Q}$).
Figure~\ref{fig:length_control} demonstrates the effectiveness of this length control mechanism.

For the results in the main text, we generate 500 sequences using $T = 5{,}000$ steps with corrector strength $\alpha = 1$ and guidance weight $w = 0.75$ for both the forward and reverse models.
For the length distribution evaluation (Figure~\ref{fig:seq}(e)), we generate a larger set of 1{,}000 sequences using $T = 10{,}000$ steps with the forward-only Euler sampler ($\alpha = 0$), nucleus sampling with $p = 0.95$, and temperature annealing from $\tau = 2.0$ to $\tau = 0.1$. 
For sequence-based motif scaffolding, generation starts from the motif tokens and proceeds for 5{,}000 steps with corrector strength $\alpha = 1$, with no classifier-free guidance or length conditioning applied.
   
\subsubsection{Evaluation Pipeline}\label{suppl:seq_eval}
We evaluate \ourmodel samples against ground-truth sequences in UniRef50 and samples generated by the baseline models.
For all baselines, we follow the DPLM~\citep{wang_diffusion_2024} evaluation protocol to generate 500 sequences in total: 50 sequences at each of the 10 fixed lengths $L \in \{100, 200, \ldots, 1000\}$.
UniRef50 reference sequences are obtained from 500 randomly sampled training sequences. We fold generated sequences into 3D structures with ESMFold~\citep{lin2023evolutionary}. The predicted structures are then assessed with the metrics detailed below, following a similar protocol in structure design.

\textbf{Foldability} is measured by the sequence pLDDT score, as folded by ESMFold. We report the mean CA pLDDT across all residues of each structure, stratified by sequence length in bins of width 100 residues. \textbf{Novelty} is evaluated by performing an exhaustive structural search across the entire PDB database with FoldSeek. The maximum TM-score across all hits defines the pdb-TM, and structures with $\text{pdb-TM} < 0.5$ are classified as structurally novel. \textbf{Structural diversity} is reported as the ratio of distinct clusters to the total number of generated structures clustered by FoldSeek with a TM-score threshold of 0.5. \textbf{Secondary structure compositions} using DSSP are also reported to provide population-level structural similarity.

\subsubsection{Baseline Information}\label{suppl:seq_baselines}
We compare \ourmodel against four protein sequence generative models. For each baseline, we use the authors' official model checkpoints and follow the default sampling procedures. 

\textbf{EvoDiff}~\citep{alamdari_protein_2023} provides two discrete diffusion variants. EvoDiff-OADM uses a masking corruption process that replaces one residue with a mask token at each step until the sequence is fully masked. EvoDiff-D3PM uses a categorical transition-matrix corruption process that introduces substitution noise, so that the terminal distribution approaches a uniform amino-acid sequence. Sampling starts from a uniform amino-acid initialization and iteratively denoises. We report EvoDiff-OADM, the better-performing variant in the original evaluation. The model is trained on UniRef50 sequences 10-1024 residues long. Sequences longer than 1024 residues are subsampled into 1024-residue chunks. In our evaluation, we use the EvoDiff-OADM-640M checkpoint. 

\textbf{DPLM}~\citep{wang_diffusion_2024} is a discrete diffusion language model instantiated using an absorbing-mask corruption process. The corruption schedule replaces residues with a mask token and treats the mask as an absorbing noise state. In our evaluation, we use the DPLM-650M checkpoint trained on the same dataset as \citet{alamdari_protein_2023}. For sampling, we perform a 500-step discrete diffusion denoising process starting from a fully masked sequence, using temperature annealing from 2.0 to 0.1, with stochastic unmasking and top-$p=0.95$ filtering.

\textbf{ESM3}~\citep{hayes_simulating_2025} is a generative masked language model trained over multiple discrete token tracks that represent sequence, structure, and function. In our experiments, we evaluate only the sequence track, which is trained on a mixture of large-scale protein sequence databases including UniRef, MGnify, JGI, and OAS, together with sequences paired with structure datasets such as PDB, AlphaFoldDB, and ESMAtlas, and further augmented with inverse-folded sequences~\citep{hayes_simulating_2025}.
We generate sequences using 20-step iterative masked language model decoding with a cosine unmasking schedule and temperature annealing.

\textbf{SCISOR}~\citep{baron_diffusion_2025} is a discrete diffusion model specialized for deletions. Its forward noising process inserts random residues into natural protein sequences, and the learned reverse process removes residues to recover natural-like sequences. SCISOR is trained on the same dataset as \citet{alamdari_protein_2023}. For unconditional generation, we draw an excess length S from its forward insertion-noise schedule and independently sample L+S residues according to the amino acid frequencies in the training data. We then apply 10-step reverse deletion diffusion at temperature 1.0 to obtain the target length L.

\input{suppl/D_localpath}

\subsection{Peptide Co-Design}
\label{suppl:pep}

\subsubsection{Data Specification}
We use the dataset derived in PepFlow~\citep{li2024full}, which contains filtered PDB structure entries with peptide lengths ranging from 3 to 25. The full dataset contains 8,365 protein--peptide complex structures, which are clustered at 40\% peptide sequence identity, yielding 292 unique clusters. Ten clusters across different lengths are selected as the test set, comprising 158 structures. The binding pocket is given as additional conditioning information, following the practice in PepFlow.

Each data pair is a complex of a target receptor and its bound peptide. Following PepFlow, the binding pocket is provided as a fixed conditioning context, and the peptide is the target for generation. The peptide is modeled at the all-atom level: a backbone frame (rotation and translation), the residue type, and up to four sidechain torsion angles per residue. During training, uniform sampling over the training clusters is performed.

\subsubsection{Model Specification}
\paragraph{Architecture} We adapt PepFlow~\citep{li2024full}, an equivariant geometric GNN with heads to regress vector fields for all-atom peptide co-design. PepFlow represents all-atom structures with per-residue rotation and translation, amino acid logit, and up to four sidechain torsion angles. Built upon the IPA structure, PepFlow denoises these modalities together using (Riemannian) flow matching to co-design the peptide conditioned on the target protein structure. 
We enable variable-length generation with an order-preserving process: a rate head predicts how many residues to insert at each inter-residue gap, so the peptide length is generated rather than conditioned on the native length. Each inserted residue is predicted across all of PepFlow's modalities with additional heads when such an insertion happens.

\paragraph{Training} We optimize GPFlow with AdamW (learning rate $10^{-4}$, weight decay 0.01) at a total batch size of 48, keeping an EMA of the weights (decay 0.999) for evaluation. The model is trained across four NVIDIA RTX PRO 6000 GPUs for approximately two days. 
We adopt two techniques to improve the early rate prediction during generation. \textbf{Noisy context learning (NCL).} During training, the kept peptide context fed to the decoder is corrupted as $x_{\mathrm{ncl}} = w\,x_t + (1-w)\,\epsilon$, with $w\sim\mathcal{U}[0,1]$ and $\epsilon\sim\mathcal{N}(0,I)$, so the model learns to read structural guidance without relying on perfectly accurate context. \textbf{Scheduled sampling (SS).} With probability $0.5$, the self-conditioning channel is fed the decoder's own clean-data prediction $x_{\mathrm{pred}} = x_t + (1-t)\,v_\theta(x_t, t, z)$ in place of the ground truth, exposing the model to its own output.

\paragraph{Sampling} For each target, we generate 10 peptides over 250 sampling steps. The all-atom structure and sequence are produced with a low-temperature score-SDE over the continuous modalities (noise scale 0.05), while the peptide length is not supplied and emerges from the learned insertion rate.

\subsubsection{Evaluation Pipeline}\label{suppl:peptide_metric}
We use the metrics in PepFlow to evaluate the quality of generation across various modalities. For each baseline model, 10 samples are generated for each holdout target, and the metrics are averaged across all generations.

\paragraph{AAR} Amino Acid Recovery (AAR) quantifies the sequence identity between the generated and the native peptide. AAR is computed as the overlap ratio between the generated and the native sequence. A higher AAR indicates a closer match between the generated and native peptides in terms of amino acid composition, reflecting the accuracy of the generated sequence.

\paragraph{RMSD} To evaluate the generation quality of the peptides, the RMSD between the ground-truth and generated peptides is calculated from their CA coordinates using the Kabsch algorithm. A lower RMSD indicates a closer structural alignment between the generated and native peptides.

\paragraph{SSR} The Secondary Structure Ratio (SSR) assesses the similarity between the secondary structure of the generated peptide and the ground-truth peptide. This metric is computed by determining the ratio of identical entries in the secondary structure labels of the two peptides. The secondary structure labels are obtained using the DSSP software. A higher SSR indicates a closer match in the secondary structure between the generated and native peptides. 

\paragraph{BSR} Binding Site Recovery (BSR) gauges the interaction similarity between the generated peptide--target pair and the native pair. BSR assesses whether the generated peptide recognizes residues in the target protein similarly to the native peptide. Following PepFlow, a residue is considered to be in the binding site if its CB atom is within a 6 Å radius of any peptide residue. BSR is calculated as the overlap between the binding sites derived from the generated and native peptides. A higher BSR indicates a closer match between the generated and native peptide-protein interactions.

\paragraph{Affinity} The Affinity ratio represents the percentage of designed peptides exhibiting lower binding energy, indicating a higher binding affinity to the target protein compared to the native peptide. Following PepFlow, the binding energy is computed using the \textit{InterfaceAnalyzerMover} in PyRosetta~\citep{chaudhury2010pyrosetta}, after relaxing the complex and defining the interface between the peptide and target protein. A higher Affinity percentage indicates that the designed peptides exhibit enhanced binding affinity, suggesting potential improvements in their functional capabilities.

\paragraph{Stability} The Stability ratio is defined as the proportion of designed peptides that exhibit a lower energy score compared to the native complex. Following PepFlow, we utilize the \textit{FastRelax} method in PyRosetta, where each complex first undergoes relaxation, and the total score is evaluated using the REF2015 score function. The Stability metric is then calculated as the ratio of complexes with reduced energy scores, highlighting the designed peptides that contribute to the enhanced stability in the protein-peptide complex.

\paragraph{Designability} Similar to previous structure design tasks, Designability assesses whether a generated peptide structure corresponds to a sequence that can fold into a structure similar to itself. Specifically, we use ESMFold to predict the peptide structure and subsequently employ Rosetta \textit{FlexPep}~\citep{london2011rosetta,raveh2011rosetta} to dock the predicted structure into the target. The resulting docked structure is then compared with the native peptide structure, where a \emph{designable} peptide should exhibit less than 2 Å RMSD to the native structure. Designability is then reported as the ratio of designable peptides across all test targets.

\paragraph{Diversity} Similar to previous structure design tasks, Diversity is quantified by calculating all the pairwise TM-scores among the generated peptides for a given target using the original TM-align program~\citep{zhang2005tm}. The diversity metric is then defined as 1 minus the average TM-score. A higher diversity value indicates greater structural variation among the generated peptides, showcasing the extent of structural exploration in the design process.

\input{tabs/peptide_breakdown}
\paragraph{Adaptations to Variable-Length Design} Among the metrics above, AAR, RMSD, and SSR are computed per residue and therefore require a correspondence between the generated peptide of length $n$ and the native peptide of length $m$. When $n=m$, this is the identity correspondence $i\mapsto i$. Since \ourmodel is a variable-length model, the generated peptide length $n$ differs from $m$ on $67.9\%$ of samples, and no canonical identity correspondence exists.
We therefore align the two peptides with an order-preserving map. When $n=m$, we use the identity correspondence; otherwise, we compute a Needleman-Wunsch alignment with zero gap cost, which yields a monotone matching of $\min(n,m)$ residue pairs. We average each metric over the matched residues only and exclude the $|n-m|$ unmatched residues, so the three metrics do not penalize the length difference, which is instead reported as a separate length statistic. The AAR, RMSD, and SSR computed under this scheme are reported in Table~\ref{tab:composite-breakdown}.

\subsubsection{Baseline Information}
In addition to the base PepFlow model, for which we use the official PepFlow checkpoint for evaluation with 200 flow sampling steps, we also compare \ourmodel instantiation against two additional peptide co-design models or pipelines. We follow PepFlow's configurations for all three baselines for fair comparison.

\paragraph{RFDiffusion} The same RFDiffusion~\citep{rfdiffusion} model used in motif scaffolding is also used here for conditional peptide design. The official RFDiffusion implementation is used with 200 diffusion steps. ProteinMPNN~\citep{dauparas2022robust} is subsequently applied to redesign the amino acid types for evaluation.

\paragraph{ProteinGenerator} ProteinGenerator~\citep{lisanza2025multistate} is built on RFDiffusion and additionally supports co-designing protein backbone structures and sequences. The official inference scripts are used with 200 diffusion steps. For a fair comparison, no additional hotspot or DSSP information is fed into the model.

%% file: tabs/config.tex
\begin{table}[ht]
\centering
\caption{Per-task modeling choices for \ourmodel. All tasks use the order-preserving variant.}\label{tab:config}
\resizebox{\linewidth}{!}{
\small
\begin{tabular}{@{}lccccc@{}}
\toprule
Task & Length-dependent modality & $\mathcal{L}_\text{FM}$ & $\mathcal{L}_\text{rec}$ & Scheduler $\kappa_t$ & Backbone \\ \midrule
Unconditional structure & CA coordinates ($\mathbb{R}^3$) & OT-FM (Eq.~\ref{eqn:loss_ot_fm}) & mean MSE (Eq.~\ref{eqn:loss_rec_mean}) & $\min(1,t/0.6)$ & Proteina \\
Structure motif scaffolding & CA coordinates ($\mathbb{R}^3$) & OT-FM (Eq.~\ref{eqn:loss_ot_fm}) & mean MSE (Eq.~\ref{eqn:loss_rec_mean}) & $\min(1,t/0.6)$ & Proteina \\
Unconditional sequence & residue type (discrete) & none ($\mathcal{L}_\text{FM}{=}0$) & CE (Eq.~\ref{eqn:loss_dfm}) & $t^3$ (localized) & DiT \\
Sequence motif scaffolding & residue type (discrete) & none ($\mathcal{L}_\text{FM}{=}0$) & CE (Eq.~\ref{eqn:loss_dfm}) & $t^3$ (localized) & DiT \\
Peptide co-design & $\mathrm{SO}(3)$, $\mathbb{R}^3$, $\mathbb{T}$, type (simplex) & RFM \& CE & geodesic mean MSE & linear & PepFlow \\ \bottomrule
\end{tabular}
}
\end{table}

%% file: suppl/D_localpath.tex
\subsubsection{Localized Probability Paths}\label{suppl:localpath}
In Algorithm~\ref{alg:train}, step~4, each target position $i$ is independently included in $z_t$ with probability $\kappa_t$.
This corresponds to the standard factorized token-wise mixture path~\citep{gat2024discrete}, in which the inclusion of each position is an independent Bernoulli trial.
While simple, this factorized construction produces partial sequences $\hat{S}_t$ consisting of non-neighboring tokens interspersed with gaps.
For long sequences, the local context around each insertion point consists largely of gap positions, limiting the signal for predicting token identities. Adapting the construction introduced in~\citet{havasi2025edit}, we replace the factorized probability path with a localized construction that introduces spatial correlation: when a target position becomes present, its neighbors are encouraged to also become present, yielding local neighborhoods in $\hat{S}_t$.

\paragraph{Construction.}
Let $X_1 = y$ denote the target sequence length.
We introduce an auxiliary boolean indicator $\mathbf{m}_t \in \{\texttt{true}, \texttt{false}\}^{y}$, where $\mathbf{m}_t^j = \texttt{true}$ means target position $j$ is present in the partial sequence $\hat{S}_t$ at time $t$.
Rather than sampling each $\mathbf{m}_t^j$ independently, we construct $\mathbf{m}_t$ via an auxiliary space of Boolean variables $\mathbf{M} \in \{\texttt{true}, \texttt{false}\}^{y \times y}$, consisting of $y$ independent CTMC processes, one per row, all initialized at $\mathbf{M}_0 = \texttt{false}$.
The dynamics of row $i$ are:
\begin{equation}
\label{eq:local_ctmc}
u_t(\mathbf{M}^{i,j} | \mathbf{M}_t^i)
= \Bigl(\lambda_t^{\mathrm{indep}}\,\delta_{ij}
+ \mathds{1}_{[\mathbf{M}_t^{i,j-1} \vee \mathbf{M}_t^{i,j+1}]}\,\lambda^{\mathrm{prop}}\Bigr)
\bigl(\mathds{1}_{[\mathbf{M}^{i,j}]} - \delta_{\mathbf{M}^{i,j}}(\mathbf{M}_t^{i,j})\bigr).
\end{equation}
The diagonal entry $\mathbf{M}^{i,i}$ switches to \texttt{true} at the time-dependent rate $\lambda_t^{\mathrm{indep}} = \dot{\kappa}_t/(1-\kappa_t)$, independently of all other positions.
Once $\mathbf{M}^{i,i}$ is active, neighboring off-diagonal entries $\mathbf{M}^{i,j}$ ($j \neq i$) switch to \texttt{true} at a constant rate $\lambda^{\mathrm{prop}}$ whenever an adjacent entry ($\mathbf{M}_t^{i,j-1}$ or $\mathbf{M}_t^{i,j+1}$) is already \texttt{true}, propagating the \texttt{true} state outward from position $i$.
We use a cubic schedule $\kappa_t = t^3$ and set $\lambda^{\mathrm{prop}} = 3.0$.
The per-position indicator is recovered by taking the column-wise disjunction:
\begin{equation}
\label{eq:local_m_map}
\mathbf{m}_t^j = \mathbf{M}_t^{1,j} \vee \mathbf{M}_t^{2,j} \vee \cdots \vee \mathbf{M}_t^{y,j}.
\end{equation}
When $\lambda^{\mathrm{prop}} = 0$, each row activates only its diagonal entry, recovering the factorized probability path in which each $\mathbf{m}_t^j$ is independent.

\paragraph{Effective rate.}
When computing the training loss, the rate at which position $j$ transitions from absent to present (i.e., $\mathbf{m}_t^j$ switches to \texttt{true}) is no longer the uniform $\dot{\kappa}_t / (1 - \kappa_t)$ but an effective rate that depends on the local propagation state:
\begin{equation}
\label{eq:local_eff_rate}
\lambda_{j,t}^{\mathrm{eff}} = \lambda_t^{\mathrm{indep}}
+ \sum_{l=1}^{y} \mathds{1}\{\mathbf{M}_t^{l,j-1} \vee \mathbf{M}_t^{l,j+1}\}\,\lambda^{\mathrm{prop}}.
\end{equation}
Positions adjacent to already-present neighbors have a higher effective rate, while isolated positions receive only the independent rate.

\paragraph{Training loss.}
The localized path modifies both the Poisson NLL $\mathcal{L}_{\mathrm{PP}}$ (Eq.~\ref{eqn:adaptednll}) and the reconstruction loss $\mathcal{L}_{\mathrm{rec}}$ (Eq.~\ref{eqn:loss_rec_order}) by replacing the uniform weight $\dot{\kappa}_t/(1-\kappa_t)$ with the position-dependent effective rate.
The modified Poisson NLL is:
\begin{equation}
\label{eq:local_pp_loss}
\mathcal{L}_{\mathrm{GP}}^{\mathrm{loc}}
= \mathbb{E}_{t, Y_1}\!\left[
  \sum_{i=0}^{X_t} \lambda_t^i
  - \sum_{i=0}^{X_t} \biggl(\sum_{k \in \Omega_i} \lambda_{k,t}^{\mathrm{eff}}\biggr) \log \lambda_t^i
\right],
\end{equation} and the modified reconstruction loss is:
\begin{equation}
\label{eq:local_loss}
\mathcal{L}_{\mathrm{rec}}^{\mathrm{loc}}
= \mathbb{E}_{Y_1, Y_t}\!\left[
  -\sum_{i=0}^{X_t}\sum_{k \in \Omega_i}
  \lambda_{k,t}^{\mathrm{eff}}\,
  \log \rho_\theta^i(s_k | Y_t)
\right],
\end{equation}
where the sum runs over the $X_1 - X_t$ not-yet-present positions $k$ (grouped into bins $\Omega_i$), and $\rho_\theta^i$ is the sampling distribution at insertion position $i$. The effective rate $\lambda_{k,t}^{\mathrm{eff}}$ upweights reconstruction loss at positions adjacent to already-present residues, encouraging the model to extend locally consistent subsequences rather than predict isolated positions.

%% file: tabs/peptide_breakdown.tex
\begin{table}[ht]
\centering
\caption{Per-subset breakdown of \ourmodel's composite AAR/RMSD/SSR.}\label{tab:composite-breakdown}
\begin{tabular}{@{}lccc@{}}
\toprule
 & Length-matched & Length-unmatched & Composite \\ \midrule
AAR $\uparrow$ (\%)    & 57.67 & 51.83 & \text{53.71} \\
RMSD $\downarrow$ (\AA) & 0.94  & 1.60 & \text{1.39} \\
SSR $\uparrow$ (\%)    & 95.93  & 88.79 & \text{91.08} \\ \bottomrule
\end{tabular}
\end{table}

%% file: suppl/D_ablation.tex
\section{Ablation Studies}\label{suppl:ablation}
In this section, we ablate the key design choices of \ourmodel on the unconditional structure generation task using the PDB dataset, reporting the same metrics as in Table~\ref{tab:uncond}.

\subsection{Insertion Scheduler and Sampling Step Ablation}
Table~\ref{tab:ablation} presents an ablation study of the length-insertion scheduler $\kappa_t$ and the number of sampling steps. Both the linear scheduler and the early-insertion scheduler ($\kappa_t:=\min(t/0.3,1)$, used in our main results) yield comparably high designability of $96.1$, whereas the late-insertion scheduler ($\kappa_t:=t^2$) noticeably degrades designability to $91.6$. This suggests that growing the structure earlier in the generative trajectory leaves more steps for the continuous vector field to refine the inserted coordinates. For the number of sampling steps, increasing from $200$ to $400$ substantially improves designability ($89.3\to96.1$), while $800$ steps provide no further gain ($95.7$); we therefore adopt $400$ steps as a balance between generation quality and sampling cost.

\input{tabs/ablation}

\subsection{Designability-Diversity Trade-Off}\label{suppl:tradeoff}

\input{tabs/tradeoff}

Table~\ref{tab:tradeoff} reports the effect of the sampling noise scale $\gamma$ in the SDE sampler (Eq.~\ref{eqn:sample_sde}). A smaller $\gamma$ mimics low-temperature sampling, increasing designability at the cost of diversity (e.g., $\gamma=0.3$ attains $97.7$ designability but only $0.27$ diversity), whereas a larger $\gamma=0.5$ reverses this trade-off ($82.6$ designability, $0.45$ diversity). We adopt $\gamma=0.35$ as the default, which balances high designability with reasonable diversity. This trade-off also explains the generally lower diversity of the Proteina-family models reported in Table~\ref{tab:uncond}.

\subsection{\texorpdfstring{$\tau$}{τ}-Leaping Ablation}
\input{tabs/tau_leap}

The $\tau$-leaping technique (Appendix~\ref{suppl:tau_leap}) is proposed to mitigate the large number of sampling steps required to achieve decent generation quality compared to standard flow-based models. For example, EditFlow requires 5k-10k sampling steps. For our structure design ablations in Table~\ref{tab:tau_leap}, 400 steps suffice for designability, with only a marginal improvement, likely due to the shorter length. However, switching to only 100 steps makes a crucial difference in recovering the original length distribution compared to the significantly shorter generations caused by rate overflow. The 93.8\% designability is biased because the length distribution is significantly shorter.

\subsection{Sequence Generation Technique Ablation}

\input{tabs/seq_ablate}

Table~\ref{tab:seq_ablation} demonstrates the impact of the sampling techniques of using the corrector and localized path in sequence generation. The corrector, which allows deletion of low-confidence amino acids, plays a crucial role in improving pLDDT, whereas the localized path has a consistent but more moderate impact. In the factorized path results, the similar length shows that the pLDDT improvement is not due to a length-distribution discrepancy, but that the corrector is learning to remove inconsistent amino acids. The corrector has a more noticeable impact on the generation length for the localized path, reducing the average length by 34.

%% file: tabs/ablation.tex
\begin{table}[htbp]
\centering
\caption{Ablation studies of different \ourmodel variants on the PDB dataset.}\label{tab:ablation}
\begin{tabular}{@{}lccccc@{}}
\toprule
\multirow{2}{*}{\begin{tabular}[c]{@{}l@{}}\ourmodel\\ Variant\end{tabular}} & \multirow{2}{*}{\begin{tabular}[c]{@{}c@{}}Design-\\ ability $\uparrow$\end{tabular}} & \multirow{2}{*}{Diversity $\uparrow$} & \multicolumn{2}{c}{Novelty $\downarrow$} & \multirow{2}{*}{\begin{tabular}[c]{@{}c@{}}Struct\%\\ $\alpha$/$\beta$\end{tabular}} \\ \cmidrule(lr){4-5}
 &  &  & PDB & AFDB &  \\ \midrule
400 steps, 0.3 early insertion & 96.1 & 0.26 & 0.76 & 0.80 & 61.2/13.8 \\
400 steps, late insertion (square) & 91.6 & 0.26 & 0.75 & 0.79 & 59.2/14.2 \\ \midrule
200 steps & 89.3 & 0.33 & 0.74 & 0.78 & 58.6/13.0 \\
800 steps & 95.7 & 0.22 & 0.76 & 0.80 & 60.1/14.3 \\ \bottomrule
\end{tabular}
\end{table}

%% file: tabs/tradeoff.tex
\begin{table}[htbp]
\centering
\caption{Ablation studies of different sampling-noise scales $\gamma$ for \ourmodel on the PDB dataset.}\label{tab:tradeoff}
\begin{tabular}{@{}lccccc@{}}
\toprule
\multirow{2}{*}{$\gamma$} & \multirow{2}{*}{\begin{tabular}[c]{@{}c@{}}Design-\\ ability $\uparrow$\end{tabular}} & \multirow{2}{*}{Diversity $\uparrow$} & \multicolumn{2}{c}{Novelty $\downarrow$} & \multirow{2}{*}{\begin{tabular}[c]{@{}c@{}}Struct\%\\ $\alpha$/$\beta$\end{tabular}} \\ \cmidrule(lr){4-5}
 &  &  & PDB & AFDB &  \\ \midrule
0.3 & 97.7 & 0.27 & 0.80 & 0.84 & 55.2/11.2 \\
0.35 (default) & 96.1 & 0.26 & 0.76 & 0.80 & 61.2/13.8 \\
0.5 & 82.6 & 0.45 & 0.73 & 0.78 & 56.7/12.9 \\ \bottomrule
\end{tabular}
\end{table}

%% file: tabs/tau_leap.tex
\begin{table}[htbp]
\centering
\caption{Ablation studies of $\tau$-leaping for different sampling steps on unconditional structure design.}\label{tab:tau_leap}
\begin{tabular}{@{}lcccc@{}}
\toprule
\multirow{2}{*}{\#Step} & \multicolumn{2}{c}{400} & \multicolumn{2}{c}{100} \\ \cmidrule(lr){2-3} \cmidrule(l){4-5} 
 & w/ $\tau$ & w/o $\tau$ & w/ $\tau$ & w/o $\tau$ \\ \midrule
Length & 153.7 ± 48.6 & 153.4 ± 43.5 & 162.1 ± 44.0 & 84.2 ± 18.3 \\
Designability$\uparrow$ & \textbf{96.1\%} & 94.0\% & 83.6\% & 93.8\% \\ \bottomrule
\end{tabular}
\end{table}

%% file: tabs/seq_ablate.tex
\begin{table}[htbp]
\centering
\caption{Ablation studies of the corrector and localized path used in sequence modeling.}\label{tab:seq_ablation}
\begin{tabular}{@{}lcccc@{}}
\toprule
\multirow{2}{*}{Techniques} & \multicolumn{2}{c}{Localized path} & \multicolumn{2}{c}{Factorized path} \\ \cmidrule(lr){2-3} \cmidrule(l){4-5} 
 & w/ corrector & w/o corrector & w/ corrector & w/o corrector \\ \midrule
Length & 479 & 513 & 519 & 525 \\
pLDDT $\uparrow$ & \textbf{68.7} & 53.2 & 64.6 & 52.6 \\ \bottomrule
\end{tabular}
\end{table}

%% file: suppl/E_result.tex
\section{Additional Results}\label{suppl:result}
In this section, we provide additional experimental results and visualizations of the \ourmodel generations to further validate its performance and generation quality.

\subsection{Unconditional Protein Generation}

\begin{figure}[ht]
    \centering
    \includegraphics[width=.9\linewidth]{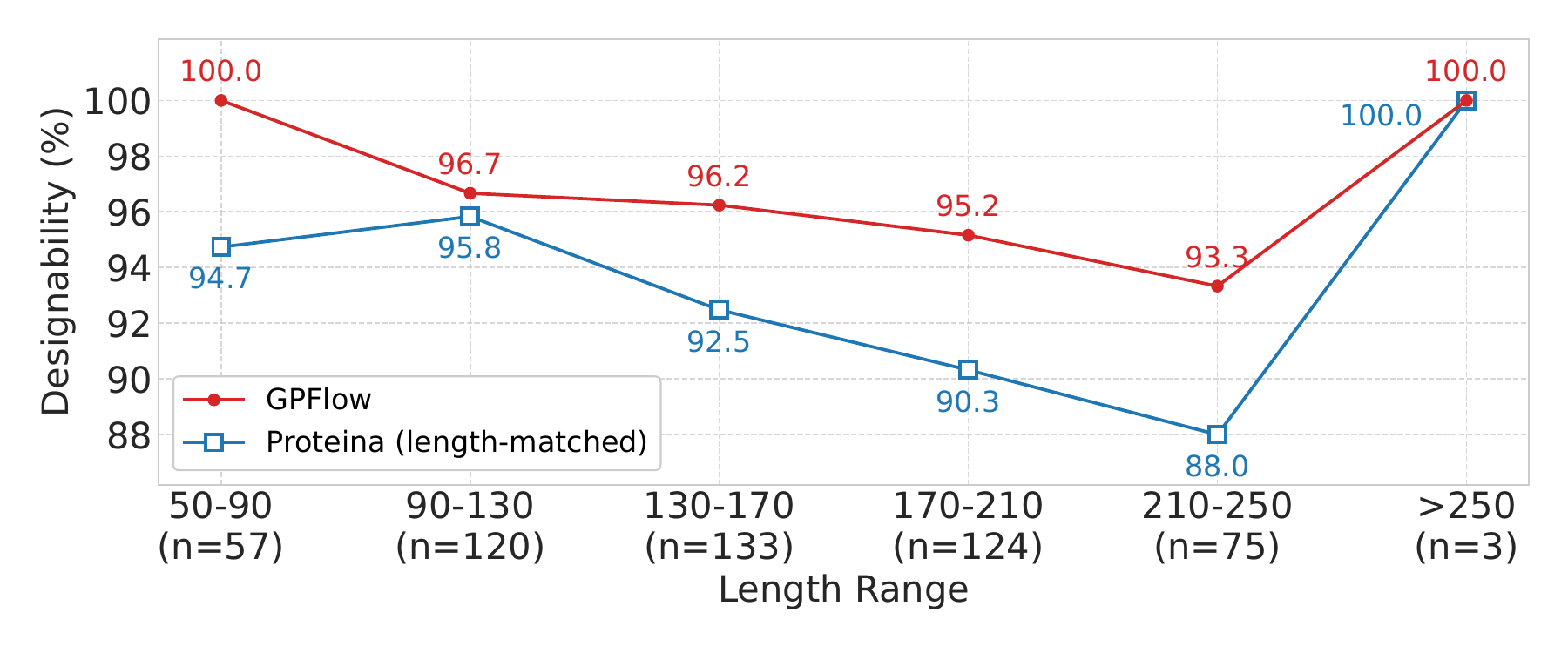}
    \vspace{-0.5em}
    \caption{Designability per length range for generated structures for \ourmodel and Proteina. Proteina generations are length-matched to \ourmodel's generated length distribution.}
    \label{fig:length_design}
\end{figure}

For unconditional protein structure design, Figure~\ref{fig:length_design} reports \ourmodel's designability within each generated length bin. Additionally, we explicitly compare Proteina under a length-matched protocol: we specify each length so that Proteina generations' length distribution matches \ourmodel's. Across all non-trivial matched length bins, \ourmodel achieves better designability than the Proteina baseline, indicating that the improvement in the aggregate designability is not driven solely by the difference in the length range. Therefore, we treat this as strong evidence that the superior designability of \ourmodel is not due to length mismatch but rather to its capture of the joint distribution over length and continuous coordinates.



\subsection{Sequence-Based Length Control}
In protein design settings where the desired length is known \textit{a priori}, it would be useful to generate samples of a specified length. To this end, we evaluate the effectiveness of using classifier-free guidance (see Figure~\ref{fig:length_control}) for controlling protein length. Specifically, we generate 100 sequences at each target length from 100 to 1000 in increments of 100, using a guidance weight of $w=0.75$ and 5{,}000 Euler steps. As shown in Figure~\ref{fig:length_control}, the generated lengths closely match the specified targets across the full range, demonstrating that \ourmodel is capable of accurate and reliable protein length control. 
\begin{figure}[ht]
    \centering
    \includegraphics[width=\linewidth]{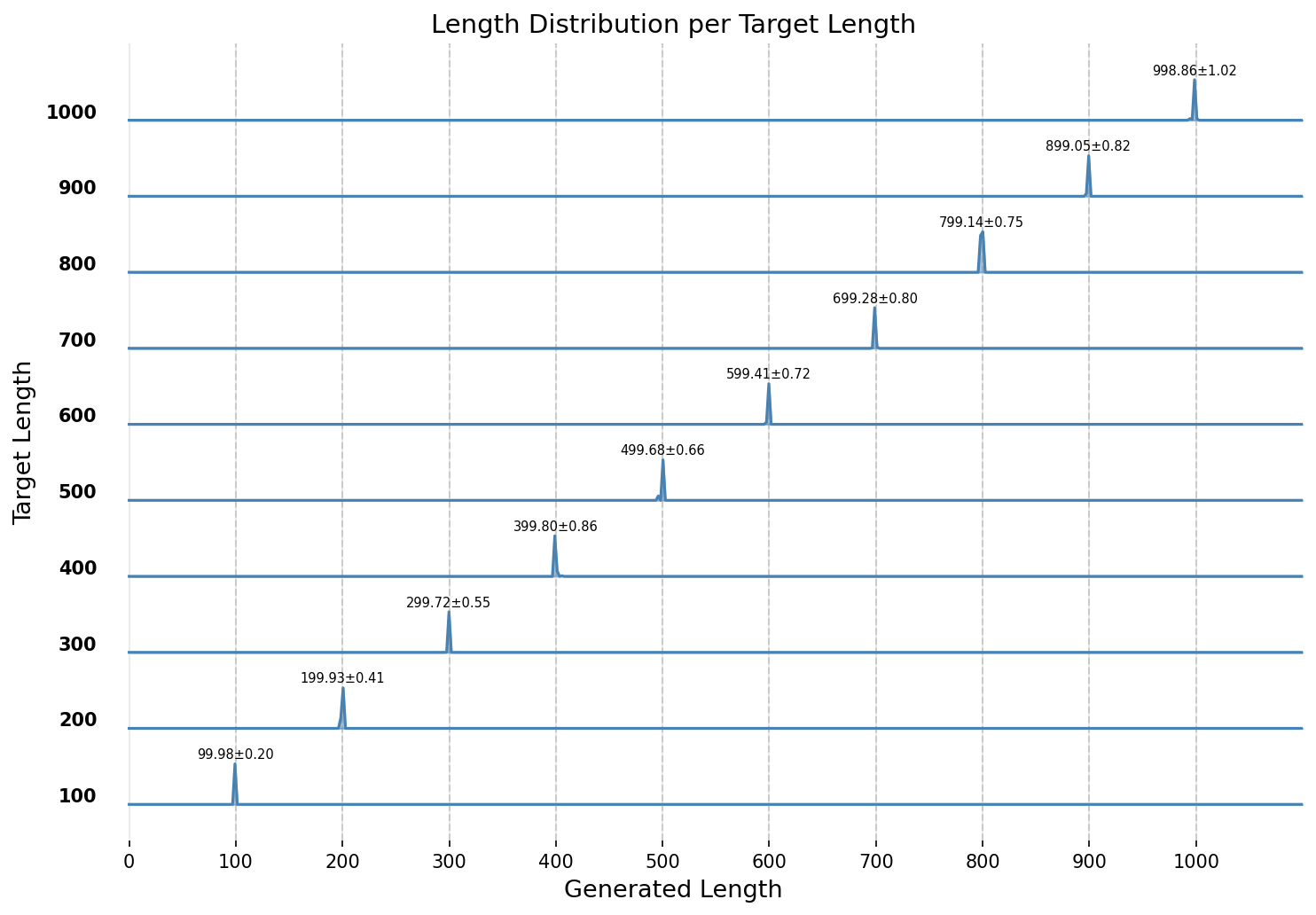}
    \caption{Evaluation of protein length control via classifier-free guidance. Each ridge shows the distribution of generated sequence lengths for a given target length, with the mean $\pm$ standard deviation annotated.}
    \label{fig:length_control}
\end{figure}

\subsection{Structure-Based Motif Scaffolding}\label{suppl:motif_length}

In Figure~\ref{fig:motif_length}, we showcase the length distribution across all 16 motif tasks generated by \ourmodel and Proteina. Compared to Proteina, our method yields a broader and smoother length distribution, reflecting its ability to adaptively generate backbones of appropriate lengths conditioned on the motif target, rather than relying on predefined contig templates.
Remarkably, \ourmodel can generate successful scaffolds longer than 200 residues, whereas Proteina generations are strictly limited by the contig templates, with most generations clustered around 75 residues.

We additionally demonstrate the per-task length distributions for representative tasks 1YCR, 3IXT, 5TPN, and 6E6R. For each task, the raw success lengths (red) and the average length within each unique success (blue) are plotted on the same scale in Figure~\ref{fig:motif_length_demo}. 
Because the sampling hyperparameters are identical across tasks, these results suggest that \ourmodel adapts its generated lengths to the motif input, spanning a wider range for 3IXT while concentrating on a shorter-length peak for 6E6R. The longer unique successes observed for 3IXT and 5TPN further indicate that the model can produce successful and diverse scaffolds across a broad length range.

\begin{figure}[ht]
    \centering
    \includegraphics[width=.8\linewidth]{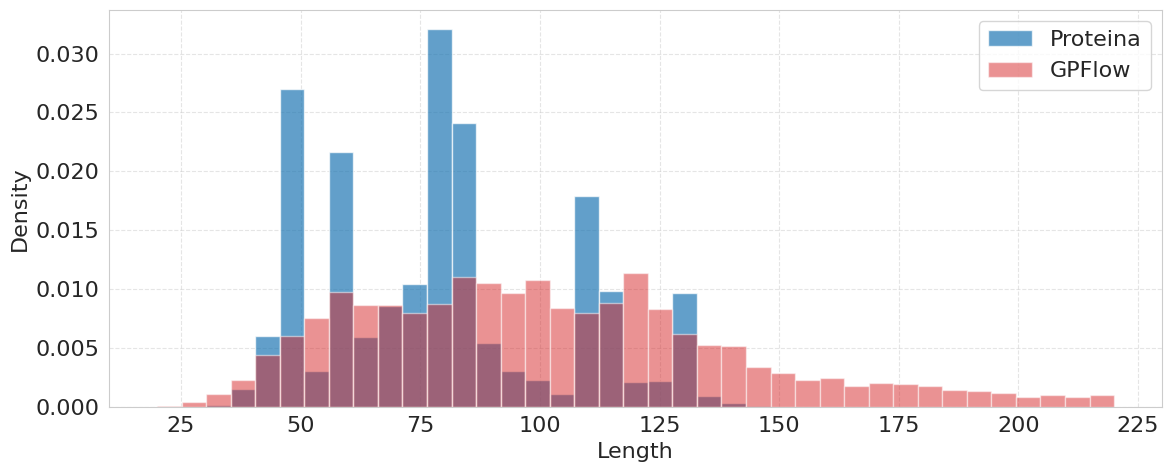}
    \vspace{-.5em}
    \caption{Aggregate protein length distribution across all 16 motif scaffolding tasks for \ourmodel and Proteina.}
    \label{fig:motif_length}
\end{figure}

\begin{figure}[ht]
    \centering
    \includegraphics[width=\linewidth]{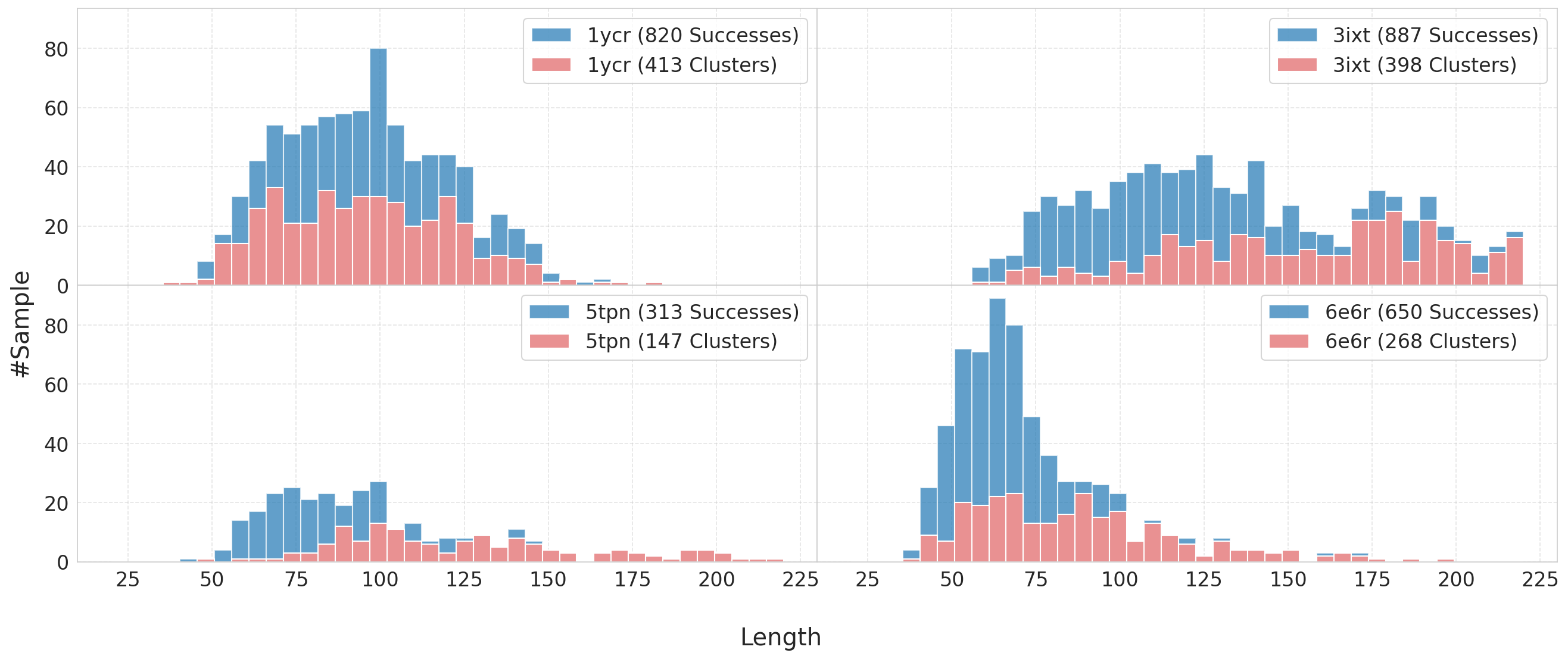}
    \vspace{-.5em}
    \caption{Protein length distributions for representative tasks for \ourmodel. Raw success lengths and the average length within each unique success are plotted on the same scale.}
    \label{fig:motif_length_demo}
\end{figure}



\subsection{Sequence-Based Motif Scaffolding}\label{suppl:seq_motif}
We adapt our 642M sequence-generation \ourmodel with the localized probability path described in Appendix~\ref{suppl:localpath} to perform sequence-based motif scaffolding. During inference, generation starts from the motif tokens at $t=0$, and the model grows a scaffold through progressive insertions with motif positions frozen. The learned deletion model serves as a corrector to refine the scaffold. The scaffold length emerges naturally from the learned rate function.

\input{tabs/motif_seq}

We compare against two sequence-based baselines: EvoDiff-OADM~\citep{alamdari_protein_2023} and DPLM~\citep{wang_diffusion_2024}, using their official checkpoints and inference protocols. A design is successful if CA motif RMSD $<$ 1Å and CA pLDDT $>$ 70 as evaluated by ESMFold~\citep{lin2023evolutionary}. We generate 100 candidates per task across the 17 motif problems adapted from the RFdiffusion benchmark~\citep{rfdiffusion} to the sequence setting by~\citet{alamdari_protein_2023}.

As shown in Table~\ref{tab:motif_scaffold_esmf}, \ourmodel passes 10 out of 17 tasks, compared to 5/17 for EvoDiff and 6/17 for DPLM. Notably, \ourmodel uniquely solves four tasks (5WN9, 2KL8, 4ZYP, 6VW1) where both baselines achieve zero success. On tasks where multiple methods succeed, \ourmodel shows substantial gains: 87\% on 3IXT versus 2\% (EvoDiff) and 23\% (DPLM); 59\% on 7MRX versus 0\% and 22\%. 



\subsection{Peptide Co-Design}\label{suppl:pep_length}
We compare the lengths of \ourmodel-generated samples with those of native peptides. In Figure~\ref{fig:peptide_length} (left), the generated and native length distributions agree closely. We further compare the generated length with the native length for each target (Figure~\ref{fig:peptide_length}, right). GPFlow recovers the lengths of short peptides nearly exactly; however, its conditional mean increasingly falls below the identity line as the native length increases, under-generating long peptides by approximately $1.6$ residues for native lengths above $12$. A possible contributing factor is the limited conditioning signal: GPFlow observes only the binding pocket and is not given the target peptide length. Thus, for longer peptides, residues extending beyond the pocket interface may be less directly constrained, potentially making their lengths harder to recover accurately.
\begin{figure}[ht]
    \centering
    \includegraphics[width=\linewidth]{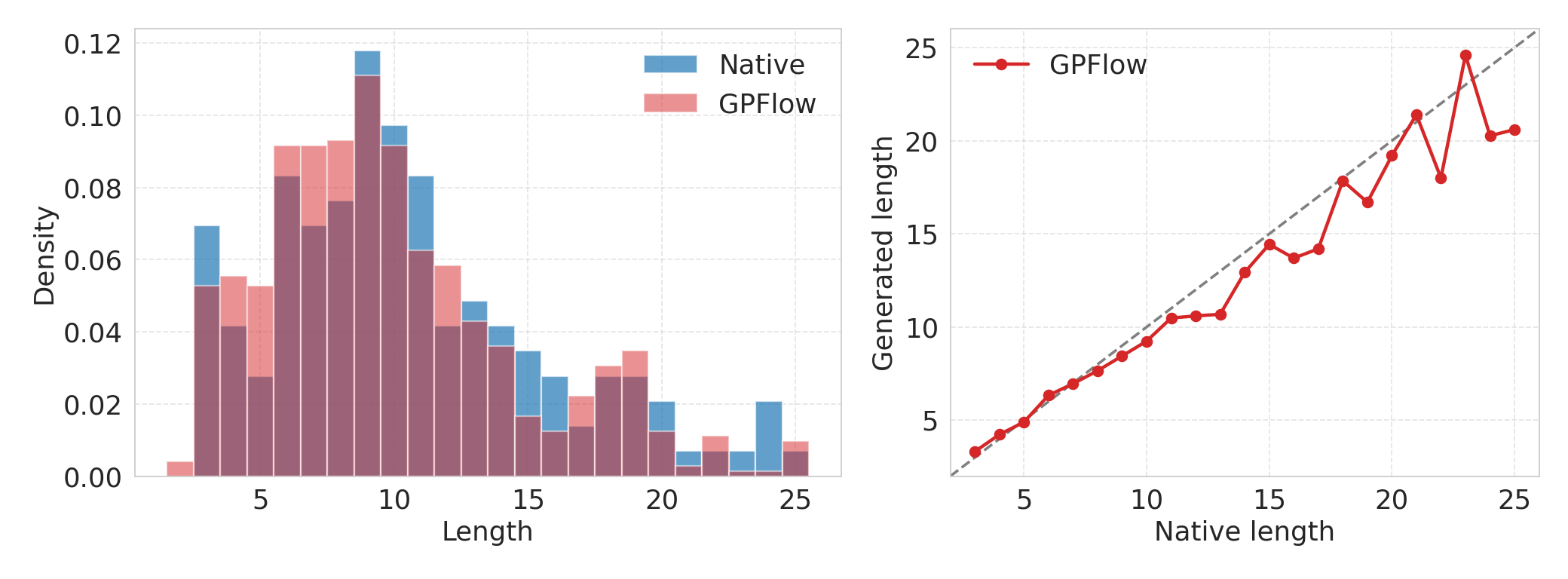}
    \vspace{-.5em}
    \caption{Peptide length. \textbf{Left}: length distributions of GPFlow-generated and native peptides. \textbf{Right}: mean generated length at each native length; the dashed line marks exact recovery.}
    \label{fig:peptide_length}
\end{figure}

%% file: tabs/motif_seq.tex
\begin{table}[ht]
\centering
\caption{Sequence-based zero-shot motif scaffolding success rates (100 generations per task). Superscripts denote the number of motif segments. Best results are shown in \textbf{bold}.}
\label{tab:motif_scaffold_esmf}
\small
\rowcolors{2}{gray!15}{white}
\begin{tabular}{l|ccc} \toprule
    PDB & \ourmodel (Ours) & DPLM [\citenum{wang_diffusion_2024}] & EvoDiff [\citenum{alamdari_protein_2023}] \\
  \midrule
    1BCF$^4$ & \textbf{0.31} & 0.00 & \textbf{0.31} \\
    1PRW$^2$ & 0.05 & \textbf{0.86} & 0.31 \\
    1QJG$^3$ & 0.00 & 0.00 & 0.00 \\
    1YCR$^1$ & 0.24 & \textbf{0.42} & 0.01 \\
    2KL8$^1$ & \textbf{0.01} & 0.00 & 0.00 \\
    3IXT$^1$ & \textbf{0.87} & 0.23 & 0.02 \\
    4JHW$^2$ & 0.00 & 0.00 & 0.00 \\
    4ZYP$^1$ & \textbf{0.01} & 0.00 & 0.00 \\
    5IUS$^2$ & 0.00 & 0.00 & 0.00 \\
    5TPN$^1$ & 0.00 & 0.00 & 0.00 \\
    5TRV$^1$ & 0.00 & 0.00 & 0.00 \\
    5WN9$^1$ & \textbf{0.13} & 0.00 & 0.00 \\
    5YUI$^3$ & 0.00 & 0.00 & 0.00 \\
    6E6R$^1$ & 0.12 & \textbf{0.90} & 0.05 \\
    6EXZ$^1$ & 0.00 & \textbf{0.02} & 0.00 \\
    6VW1$^2$ & \textbf{0.05} & 0.00 & 0.00 \\
    7MRX$^1$ & \textbf{0.59} & 0.22 & 0.00 \\ \midrule
    Pass rate & \textbf{10/17} & 6/17 & 5/17 \\
    \bottomrule
\end{tabular}
\end{table}

%% file: secs/6_impact.tex
\section{Broader Impacts}\label{suppl:impact}
This work introduces Generalized Poisson Flow (\ourmodel), a framework for variable-length generation evaluated on protein design. Its connection between generalized Poisson processes and probability flows may also be relevant to other domains with variable-dimensional data, although such applications are not evaluated here.
In computational biology, \ourmodel addresses the dependence of structural feasibility on sequence length in \emph{de novo} protein design. Learning a distribution over backbone lengths may reduce manual length searches in motif-scaffolding applications relevant to vaccine and enzyme design.
As with any advanced protein design technology, there is a potential for misuse, such as the design of harmful biological agents or toxins. The generalizability of our framework implies that it could potentially be applied to biological systems with unknown or hazardous properties. We are dedicated to ensuring the responsible use of our model for societal benefit.